\documentclass{article}

     \usepackage[final,nonatbib]{neurips_2020}

\usepackage[utf8]{inputenc} 
\usepackage[T1]{fontenc}    
\usepackage{hyperref}       
\usepackage{url}            
\usepackage{booktabs}       
\usepackage{amsfonts}       
\usepackage{nicefrac}       
\usepackage{microtype}      

\usepackage{wrapfig}

\usepackage{commath}
\usepackage{amsthm}
\usepackage{amsmath}
\usepackage{amssymb}
\usepackage{bbm}
\usepackage[utf8]{inputenc}
\usepackage[english]{babel}
 \usepackage{subcaption}

\usepackage{graphicx}

\usepackage[toc,page]{appendix}

\usepackage{algorithm}
\usepackage{algorithmic}

\usepackage{cleveref}
\crefformat{footnote}{#2\footnotemark[#1]#3}

\usepackage{verbatim}

\allowdisplaybreaks

\newtheorem{theorem}{Theorem}

\newtheorem{lemma}{Lemma}

\newtheorem{definition}{Definition}
\title{Federated Bayesian Optimization via\\
Thompson Sampling}

\author{%
  Zhongxiang Dai$^{\dagger}$, Bryan Kian Hsiang Low$^{\dagger}$, Patrick Jaillet$^{\S}$\\
  Dept. of Computer Science, National University of Singapore, Republic of Singapore$^{\dagger}$\\
  Dept. of Electrical Engineering and Computer Science, MIT, USA$^{\S}$\\
  \texttt{\{daizhongxiang,lowkh\}@comp.nus.edu.sg$^{\dagger}$,jaillet@mit.edu$^{\S}$}\\
}

\begin{document}

\maketitle

\begin{abstract}
\emph{Bayesian optimization} (BO) is a prominent approach to optimizing expensive-to-evaluate black-box functions. The massive computational capability of edge devices such as mobile phones, coupled with privacy concerns, has led to a surging interest in \emph{federated learning} (FL) which focuses on collaborative training of \emph{deep neural networks} (DNNs) via \emph{first-order optimization} techniques.
However, some common machine learning tasks such as hyperparameter tuning of DNNs lack access to gradients and thus require \emph{zeroth-order/black-box optimization}.
This hints at the possibility of extending BO to the FL setting (FBO) for agents to collaborate in these black-box optimization tasks.
This paper presents \emph{federated Thompson sampling} (FTS) which overcomes a number of key challenges of FBO and FL in a principled way: 
We (a) use \emph{random Fourier features} to approximate the Gaussian process surrogate model used in BO, which naturally produces the parameters to be exchanged between agents,
(b) design FTS based on \emph{Thompson sampling}, which significantly reduces the number of parameters to be exchanged,
and (c) provide a theoretical convergence guarantee that is robust against heterogeneous agents, which is a major challenge in FL and FBO.
We empirically demonstrate the effectiveness of FTS in terms of communication efficiency, computational efficiency, and practical performance.
\end{abstract}

\section{Introduction}
\label{sec:introduction}
%
\emph{Bayesian optimization} (BO) has recently become a prominent approach to optimizing expensive-to-evaluate black-box functions with no access to gradients, 
such as in hyperparameter tuning of \emph{deep neural networks} (DNNs)~\cite{shahriari2016taking}.
A rapidly growing computational capability of edge devices such as mobile phones, as well as an increasing concern over data privacy, has 
given rise to the widely celebrated paradigm of \emph{federated learning} (FL)~\cite{mcmahan2016communication} which is also known as \emph{federated optimization}~\cite{li2019convergence}.
In FL, individual agents, without transmitting their raw data, attempt to leverage the contributions from the other agents
to more effectively optimize the parameters of their \emph{machine learning} (ML) model (e.g., DNNs) through \emph{first-order optimization} techniques 
(e.g., stochastic gradient descent)~\cite{kairouz2019advances,li2019federated}.
However, some common ML tasks such as hyperparameter tuning of DNNs lack access to gradients and thus require \emph{zeroth-order/black-box optimization},
and a recent survey~\cite{kairouz2019advances} has pointed out that hyperparameter optimization of DNNs in the FL setting is one of the promising research directions for FL.
This opportunity, combined with the proven capability of BO to efficiently optimize expensive-to-evaluate black-box functions~\cite{shahriari2016taking}, naturally suggests the 
possibility of extending BO to the FL setting, which we refer to as \emph{federated BO} (FBO).

The setting of our FBO is similar to that of FL, except that FBO uses zeroth-order optimization, in contrast to first-order optimization adopted by FL.
In FBO, every agent uses BO to optimize a black-box function (e.g., hyperparameter optimization of a DNN) and 
attempts to improve the efficiency of its BO task by incorporating the information from other agents.
The information exchange between agents has to take place without directly transmitting the raw data of their BO tasks (i.e., history of input-output pairs).
A motivating example is when a number of mobile phone users collaborate in optimizing the hyperparameters of their separate DNNs used for next-word prediction in a smart keyboard application,
without sharing the raw data of their own hyperparameter optimization tasks. 
This application cannot be handled by FL due to the lack of gradient information and thus calls for FBO.
Note that the generality of BO as a black-box optimization algorithm makes the applicability of FBO extend beyond hyperparameter tuning of DNNs on edge devices.
For example, hospitals can be agents in an FL system~\cite{kairouz2019advances}: When a hospital uses BO to select the patients to perform a medical test~\cite{yu2015predicting}, 
FBO can be employed to help the hospital accelerate its BO task using the information from other hospitals without requiring their raw data.
Unfortunately, despite its promising applications, FBO faces a number of major challenges,
some of which are only present in FBO, while others plague the FL setting in general.

The first challenge, which arises only in FBO yet not FL, results from the requirement for retaining (hence not transmitting) the raw data.
In standard FL, the transmitted information consists of the parameters of a DNN~\cite{mcmahan2016communication}, which reduces the risk of privacy violation compared to passing the raw data.
In BO, the information about a BO task is contained in the \emph{surrogate model} which is used to model the objective function and hence guide the query selection (Section~\ref{sec:background}).
However, unlike a DNN, a \emph{Gaussian process} (GP) model~\cite{rasmussen2004gaussian}, which is the most commonly used surrogate model in BO,
is \emph{nonparametric}.
Therefore, a BO task has no parameters (except for the raw data of BO) that can represent the GP surrogate and thus be exchanged between agents, 
while the raw data of BO should be retained and never transmitted~\cite{kusner2015differentially}.
To overcome this challenge, we exploit \emph{random Fourier features} (RFF)~\cite{rahimi2008random} to 
approximate a GP using a Bayesian linear regression model.
This allows us to naturally derive parameters that contain the information about the approximate GP surrogate and 
thus can be communicated between agents without exchanging the raw data (Section~\ref{sec:background}).
In fact, with RFF approximation, the parameters to be exchanged in FBO are equivalent to those of a linear model in standard FL (Section~\ref{subsec:compare_fts_with_rgpe_taf}).

FBO also needs to handle some common challenges faced by FL in general: communication efficiency and heterogeneity of agents.
Firstly, communication efficiency is an important factor in the FL setting since a large number of communicated parameters places a demanding requirement on the communication bandwidth~\cite{kairouz2019advances} 
and is also more vulnerable to potential malicious privacy attacks~\cite{chang2019cronus}.
To this end, we use \emph{Thompson sampling} (TS)~\cite{thompson1933likelihood}, which has been recognized as a highly effective practical method~\cite{chapelle2011empirical}, to develop our FBO algorithm.
The use of TS reduces the required number of parameters to be communicated while maintaining competitive performances (Section~\ref{subsec:compare_fts_with_rgpe_taf}).
Secondly, the heterogeneity of agents is an important practical consideration in FL since different agents can have highly disparate properties~\cite{li2019federated}. 
In FBO, heterogeneous agents represent those agents whose objective functions are significantly different from that of the \emph{target agent} (i.e., the agent performing BO).
For example, the optimal hyperparameters of the DNN for next-word prediction may vary significantly across agents as a result of the distinct typing habits of different mobile phone users.
To address this challenge, we derive a theoretical convergence guarantee for our algorithm which is \emph{robust against heterogeneous agents}.
In particular, our algorithm achieves \emph{no regret} asymptotically even when some or all other agents have highly different objective functions from the target agent.

This paper introduces the first algorithm for the FBO setting called \emph{federated Thompson sampling} (FTS) which is both theoretically principled and practically effective.
We provide a theoretical convergence guarantee for FTS that is robust against heterogeneous agents (Section~\ref{sec:theoretical_results}).
We demonstrate the empirical effectiveness of FTS in terms of communication efficiency, 
computational efficiency, and practical performance using a landmine detection experiment and 
two activity recognition experiments using Google glasses and mobile phone sensors (Section~\ref{sec:experiments}).
%
\section{Background}
\label{sec:background}
\textbf{Bayesian Optimization (BO) and Gaussian Process (GP).}
BO attempts to find a global maximizer of a black-box \emph{objective function} $f$ defined on a domain $\mathcal{X} \subset \mathbb{R}^D$, i.e., find $x^*\triangleq{\arg\max}_{x\in\mathcal{X}}f(x)$ through sequential queries.
That is, in iteration $t=1,\ldots,T$, BO queries an input $x_t$ to observe a noisy output $y(x_t)\triangleq f(x_t)+\epsilon$ where $\epsilon$ is an additive Gaussian noise with variance $\sigma^2$: $\epsilon\sim \mathcal{N}(0,\sigma^2)$.
The aim of a BO algorithm is to minimize \emph{regret}.
Specifically, if the \emph{cumulative regret} $R_T\triangleq\sum_{t=1,\ldots,T}[f(x^*) - f(x_t)]$ grows sublinearly, then the BO algorithm is said to achieve \emph{no regret} since 
it implies that the \emph{simple regret} $S_T\triangleq\min_{t=1,\ldots,T} [f(x^*)-f(x_t)] \leq R_T/T$ goes to $0$ asymptotically.
In order to sequentially select the queries to minimize regret, BO usually uses a GP as a surrogate to model the objective function $f$.
A GP is a stochastic process in which any finite subset of random variables follows a multivariate Gaussian distribution~\cite{rasmussen2004gaussian}.
A GP, which is represented as $\mathcal{GP}(\mu(\cdot), k(\cdot,\cdot))$, is fully characterized by its mean function $\mu(x)$ and covariance/kernel function $k(x,x')$ for all $x,x'\in \mathcal{X}$.
We assume w.l.o.g. that $\mu(x)\triangleq 0$ and $k(x,x') \leq 1$ for all $x,x'\in\mathcal{X}$.
We consider the widely used \emph{squared exponential} (SE) kernel here.
Conditioned on a set of $t$ observations $D_t \triangleq \{(x_1,y(x_1)),\ldots,(x_t,y(x_t))\}$, 
the posterior mean and covariance of GP can be expressed as
\begin{equation}
    \mu_t(x) \triangleq k_t(x)^\top(K_t+\sigma^2I)^{-1}y_{t}\ ,\quad \sigma_t^2(x,x') \triangleq k(x,x')-k_t(x)^\top(K_t+\sigma^2I)^{-1}k_t(x')
\label{gp_posterior}
\end{equation}
where $K_t\triangleq [k(x_{t'}, x_{t''})]_{t',t''=1,\ldots,t}$, $k_t(x)\triangleq [k(x, x_{t'})]^{\top}_{t'=1,\ldots,t}$, and $y_t\triangleq [y(x_1),\ldots,y(x_t)]^{\top}$.

Unfortunately, GP suffers from poor scalability (i.e., by incurring $\mathcal{O}(t^3)$ time), thus calling for the need of approximation methods.
Bochner's theorem states that any continuous stationary kernel $k$ (e.g., SE kernel) can be expressed as the Fourier integral of a spectral density $p(s)$~\cite{rasmussen2004gaussian}.
As a result, random samples can be drawn from $p(s)$ to construct the $M$-dimensional ($M\geq1$) \emph{random features} $\phi(x)$ for all $x\in\mathcal{X}$ (Appendix~\ref{app:construct_random_features}) 
whose inner product can be used to approximate the kernel values: $k(x,x') \approx \phi(x)^{\top}\phi(x')$ for all $x,x'\in \mathcal{X}$~\cite{rahimi2008random}.
The approximation quality of such a \emph{random Fourier features} (RFF) approximation method is theoretically guaranteed with high probability: $\sup_{x,x'\in \mathcal{X}}|k(x, x') - \phi(x)^{\top}\phi(x')| \leq \varepsilon$ 
where $\varepsilon\triangleq\mathcal{O}(M^{-1/2})$~\cite{rahimi2008random}.
Therefore, increasing the number $M$ of random features improves the approximation quality (i.e., smaller $\varepsilon$).

A GP with RFF approximation can be interpreted as a Bayesian linear regression model with $\phi(x)$ as the features: $\hat{f}(x)\triangleq\phi(x)^{\top}\omega$.
With the prior of $\mathbb{P}(\omega) \triangleq \mathcal{N}(0, I)$ and given the set of observations $D_t$, the posterior belief of $\omega$ can be derived as
\begin{equation}
\mathbb{P}(\omega|\Phi(X_t), y_t) = \mathcal{N}(\nu_t, \sigma^2\Sigma_t^{-1}) 
\label{eq:posterior_dist_omega}
\end{equation}
where $\Phi(X_t)=[\phi(x_1),\ldots,\phi(x_t)]^{\top}$ is a $t\times M$-dimensional matrix and
\begin{equation}
\Sigma_t \triangleq \Phi(X_t)^{\top}\Phi(X_t) + \sigma^{2}I\ , \qquad 
\nu_t \triangleq \Sigma_t^{-1}\Phi(X_t)^{\top}y_t 
\label{eq:sigma_nu}
\end{equation}
which contain $M^2$ and $M$ parameters, respectively.
As a result, we can sample a function $\tilde{f}$ from the GP posterior/predictive belief with RFF approximation by firstly sampling $\tilde{\omega}$ from the posterior belief of $\omega$~\eqref{eq:posterior_dist_omega} 
and then setting $\tilde{f}(x)=\phi(x)^{\top}\tilde{\omega}$ for all $x \in \mathcal{X}$.
Moreover, $\Sigma_t$ and $\nu_t$~\eqref{eq:sigma_nu} fully define the GP posterior/predictive belief  with RFF approximation at any input $x$, 
which is a Gaussian with the mean $\hat{\mu}_t(x) \triangleq \phi(x)^{\top}\nu_t$ and variance 
$\hat{\sigma}_t^2(x) \triangleq \sigma^2 \phi(x)^{\top}\Sigma_t^{-1}\phi(x)$ (Appendix~\ref{app:gp_posterior_rff}).
\textbf{Problem Setting of Federated Bayesian Optimization.}
Assume that there are $N+1$ agents in the system: $\mathcal{A}$ and $\mathcal{A}_1,\ldots,\mathcal{A}_N$. 
For ease of exposition, we focus on the perspective of $\mathcal{A}$ as the \emph{target agent},
i.e., $\mathcal{A}$ attempts to use the information from agents $\mathcal{A}_1,\ldots,\mathcal{A}_N$ to accelerate its BO task.
We denote $\mathcal{A}$'s objective function as $f$ 
and a sampled function from $\mathcal{A}$'s GP posterior belief~\eqref{gp_posterior} at iteration $t$ as $f_t$.
We represent $\mathcal{A}_n$'s objective function as $g_n$ and a sampled function from $\mathcal{A}_n$'s GP posterior belief with RFF approximation as $\hat{g}_{n}$.
We assume that all agents share the same set of random features $\phi(x)$ for all $x\in\mathcal{X}$, 
which is easily achievable since it is equivalent to sharing the first layer of a neural network in FL (Appendix~\ref{app:construct_random_features}).
For theoretical analysis, we assume that all objective functions are defined on the same domain $\mathcal{X} \subset \mathbb{R}^{D}$ which is assumed to be discrete for simplicity but our analysis can be easily extended to compact domain through discretization~\cite{chowdhury2017kernelized}.
A smoothness assumption on these functions is required for theoretical analysis; so, we assume that they 
have bounded norm induced by the \emph{reproducing kernel Hilbert space} (RKHS) associated with the kernel $k$: $\norm{f}_{k} \leq B$ and $\norm{g_n}_{k} \leq B$ for $n=1,\ldots,N$.
This further suggests that the absolute function values are upper-bounded: $|f(x)| \leq B$ and $|g_n(x)|\leq B$ for all $x\in\mathcal{X}$.
We denote the maximum difference between $f$ and $g_n$ as $d_n\triangleq\max_{x\in \mathcal{X}}|f(x) - g_n(x)|$ which characterizes the similarity between $f$ and $g_n$.
A smaller $d_n$ implies that $f$ and $g_n$ are more similar and heterogeneous agents are those with large $d_n$'s.
Let $t_n$ denote the number of BO iterations that $\mathcal{A}_n$ has completed (i.e., number of observations of $\mathcal{A}_n$) when it passes information to $\mathcal{A}$;
$t_n$'s are constants unless otherwise specified.
\section{Federated Bayesian Optimization (FBO)}
\subsection{Federated Thompson Sampling (FTS)}
\label{subsec:fts_algo}
Before agent $\mathcal{A}$ starts to run a new BO task, it can request for information from the other agents $\mathcal{A}_1,\ldots,\mathcal{A}_N$.
Next, every agent $\mathcal{A}_n$ for $n=1,\ldots,N$ uses its own history of observations, as well as the shared random features (Section~\ref{sec:background}), 
to calculate the posterior belief $\mathcal{N}(\nu_{n}, \sigma^2\Sigma_{n}^{-1})$~\eqref{eq:posterior_dist_omega} where $\nu_{n}$ and $\Sigma_{n}$ represent $\mathcal{A}_n$'s parameters of the RFF approximation~\eqref{eq:sigma_nu}. Then, $\mathcal{A}_n$ draws a sample from the posterior belief: $\omega_n \sim \mathcal{N}(\nu_{n}, \sigma^2\Sigma_{n}^{-1})$ and passes the $M$-dimensional vector $\omega_n$ to the target agent $\mathcal{A}$
(possibly via a central server).
After receiving the messages from other agents, $\mathcal{A}$ uses them to start the FTS algorithm (Algorithm~\ref{alg:F_TS}).
To begin with, $\mathcal{A}$ needs to define (a) a monotonically increasing sequence $[p_t]_{t\in\mathbb{Z}^+}$ s.t.~$p_t\in(0, 1]$ for all $t\in\mathbb{Z}^+$ and $p_t\rightarrow 1$ as $t\rightarrow +\infty$,
and (b) a discrete distribution $P_{N}$ over the agents $\mathcal{A}_1,\ldots,\mathcal{A}_N$ s.t.~$P_{N}[n]\in[0,1]$ for $n=1,\ldots,N$ and $\sum^N_{n=1}P_{N}[n]=1$.
In iteration $t \geq 1$ of FTS, with probability $p_t$  (line $4$ of Algorithm~\ref{alg:F_TS}), $\mathcal{A}$ samples a function $f_t$ using its current GP posterior belief~\eqref{gp_posterior} and chooses $x_t=\arg\max_{x\in\mathcal{X}}f_t(x)$. 
With probability $1-p_t$ (line $6$), $\mathcal{A}$ firstly samples an agent $\mathcal{A}_n$ from $P_{N}$  and then chooses $x_t=\arg\max_{x\in\mathcal{X}}\hat{g}_{n}(x)$ where $\hat{g}_{n}(x)=\phi(x)^{\top}\omega_{n}$ corresponds to a sampled function from $\mathcal{A}_n$'s GP posterior belief with RFF approximation.
Next, $x_t$ is queried to observe $y(x_t)$ and FTS proceeds to the next iteration $t+1$.
%
\begin{algorithm}[H]
\begin{algorithmic}[1]
	\FOR{$t=1,2,\ldots, T$}
		\STATE Sample $r$ from the uniform distribution over $[0,1]$: $r\sim U(0,1)$
		\IF{$r \leq p_t$}
	        \STATE Sample $f_t \sim \mathcal{GP}(\mu_{t-1}(\cdot), \beta_t^2\sigma^2_{t-1}(\cdot,\cdot))$\footnotemark 
	        and choose $x_t=\arg\max_{x\in\mathcal{X}}f_t(x)$
		\ELSE
            \STATE Sample agent $\mathcal{A}_n$ from the distribution $P_N$ and choose $x_t=\arg\max_{x\in\mathcal{X}} \phi(x)^{\top}\omega_{n}$
        \ENDIF
        \STATE Query $x_t$ to observe $y(x_t)$ and update GP posterior belief~\eqref{gp_posterior} with $(x_t, y(x_t))$
	\ENDFOR
\end{algorithmic}
\caption{Federated Thompson Sampling (FTS)}
\label{alg:F_TS}
\end{algorithm}
\vspace{-2mm}
\footnotetext{We will define $\beta_t$ in Theorem~\ref{theorem:fts} (Section~\ref{sec:theoretical_results}).}
%
Interestingly, FTS (Algorithm~\ref{alg:F_TS}) can be interpreted as a variant of \emph{TS with a mixture of GPs}. That is, in each iteration $t$, we firstly sample a GP: The GP of $\mathcal{A}$ is sampled with probability $p_t$ while the GP of $\mathcal{A}_n$ is sampled with probability $(1-p_t)P_N[n]$ for  $n=1,\ldots,N$. Next, we draw a function from the sampled GP whose maximizer is selected to be queried. As a result, $x_t$ follows the same distribution as the maximizer of the mixture of GPs and the mixture model gradually converges to the GP of $\mathcal{A}$ as $p_t \rightarrow 1$. 
The sequence $[p_t]_{t\in\mathbb{Z}^+}$ controls the degree of which information from the other agents is exploited, such that decreasing the value of this sequence encourages the utilization of such information. 
The distribution $P_N$ decides the preferences for different agents. 
A natural choice for $P_N$ is the uniform distribution $P_N[n]=1/N$ for $n=1,\ldots,N$ indicating equal preferences for all agents,
which is a common choice when we have no knowledge regarding which agents are more similar to the target agent.
In FTS, \emph{stragglers}\footnote{Stragglers refer to those agents whose information is not received by the target agent~\cite{li2019convergence}.} 
can be naturally dealt with by simply assigning $0$ to the corresponding agent $\mathcal{A}_n$ in the distribution $P_N$ such that 
$\mathcal{A}_n$ is never sampled (line $6$ of Algorithm~\ref{alg:F_TS}).
Therefore, FTS is robust against communication failure which is a common issue in FL~\cite{li2019convergence}.

Since only one message $\omega_n$ is received from each agent \emph{before} the beginning of FTS, once an agent $\mathcal{A}_n$ is sampled and its message $\omega_n$ is used (line $6$ of Algorithm~\ref{alg:F_TS}), 
we remove it from $P_N$ by setting the corresponding element to $0$ and then re-normalize $P_N$.
However, FTS can be easily generalized to allow $\mathcal{A}$ to receive information from each agent after every iteration (or every few iterations) 
such that every agent can be sampled multiple times.
This more general setting requires more rounds of communication.
In practice, FTS is expected to perform similarly in both settings when (a) the number $N$ of agents is large (i.e., a common assumption in FL),  
and (b) $P_N$ gives similar or equal preferences to all agents 
such that the probability of an agent being sampled more than once is small.
Furthermore, this setting can be further generalized to encompass the scenario where multiple (even all) agents are concurrently performing optimization tasks using FTS.
In this case, the information received from $\mathcal{A}_n$ can be updated as $\mathcal{A}_n$ collects more observations, 
i.e., $t_n$ may increase as updated information is received from $\mathcal{A}_n$.
\subsection{Comparison with Other BO Algorithms Modified for the FBO Setting}
\label{subsec:compare_fts_with_rgpe_taf}
Although FTS is the first algorithm for the FBO setting, some algorithms for \emph{meta-learning} in BO, 
such as \emph{ranking-weighted GP ensemble} (RGPE)~\cite{feurer2018scalable} and \emph{transfer acquisition function} (TAF)~\cite{wistuba2018scalable},
can be adapted to the FBO setting through a heuristic combination with RFF approximation.
Meta-learning aims to use the information from previous tasks to accelerate the current task.
Specifically, both RGPE and TAF use a separate GP surrogate to model the objective function of every agent (i.e., previous task) and use these GP surrogates to accelerate the current BO task.
To modify both algorithms to suit the FBO setting, every agent $\mathcal{A}_n$ firstly applies RFF approximation to its own GP surrogate and passes the resulting parameters $\nu_n$ and $\Sigma_n^{-1}$ (Section~\ref{sec:background}) to the target agent $\mathcal{A}$.
Next, after receiving $\nu_n$ and $\Sigma_n^{-1}$ from the other agents, $\mathcal{A}$ can use them to calculate the GP surrogate (with RFF approximation) of each agent (Section~\ref{sec:background}),
which can then be plugged into the original RGPE/TAF
algorithm.\footnote{Refer to~\cite{feurer2018scalable} and~\cite{wistuba2018scalable} for more details about RGPE and TAF, respectively.}
However, unlike FTS, RGPE and TAF do not have theoretical convergence guarantee and thus lack an assurance to guarantee consistent performances in the presence of heterogeneous agents.
Moreover, as we will analyze below and show in the experiments (Section~\ref{sec:experiments}),
FTS outperforms both RGPE and TAF in a number of major aspects including communication efficiency, computational efficiency, and practical performance.


Firstly, regarding communication efficiency, both RGPE and TAF require $\nu_{n}$ and $\Sigma_{n}^{-1}$ (i.e., $M + M^2$ parameters) from each agent 
since both the posterior mean and variance of every agent are needed.
Moreover, TAF additionally requires the incumbent (currently found maximum observation value) of every agent, which can further increase the risk of privacy leak.
In a given experiment and for a fixed $M$, our FTS algorithm is superior in terms of communication efficiency since it only requires an $M$-dimensional vector $\omega_n$ from each agent, 
which is equivalent to standard FL using a linear model with $M$ parameters.
Secondly, FTS is also advantageous in terms of computational efficiency: When $x_t$ is selected using $\omega_n$ from an agent,
FTS only needs to solve the optimization problem of $x_t=\arg\max_{x\in\mathcal{X}} \phi(x)^{\top}\omega_{n}$ (line $6$ of Algorithm~\ref{alg:F_TS}), which incurs minimal computational cost;\footnote{We use the DIRECT method for this optimization problem, which, for example, takes on average $0.76$ seconds per iteration in the landmine detection experiment (Section~\ref{subsec:exp_real_world}).} 
when $x_t$ is selected by maximizing a sampled function from $\mathcal{A}$'s GP posterior belief (line $4$ of Algorithm~\ref{alg:F_TS}), 
this maximization step can also utilize the RFF approximation, which is computationally cheap.
In contrast, for both RGPE and TAF, every evaluation of the acquisition function (i.e., to be maximized to select $x_t$) at an input $x\in \mathcal{X}$ 
requires calculating the posterior mean and variance using the GP surrogate of \emph{every} agent at $x$.
Therefore, their required computation in every iteration grows linearly in the number $N$ of agents and  can thus become prohibitively costly when $N$ is large.
We have also empirically verified this in our experiments (see Fig.~\ref{fig:real_world_exp_time_plot}d in Section~\ref{subsec:exp_real_world}).
%
\section{Theoretical Results}
\label{sec:theoretical_results}
In our theoretical analysis, since we allow the presence of heterogeneous agents (i.e., other agents with  significantly different objective functions from the target agent),
we do not aim to show that FTS achieves a faster convergence than standard TS and instead prove a convergence guarantee that is robust against heterogeneous agents.
This is consistent with most works proving the convergence of FL algorithms~\cite{li2018federated,li2019convergence}  
and makes the theoretical results more applicable in general since the presence of heterogeneous agents is a major and inevitable challenge of FL and FBO.
Note that we analyze FTS in the more general setting where communication is allowed before every iteration instead of only before the first iteration.
However, as discussed in Section~\ref{subsec:fts_algo}, FTS behaves similarly in both settings in the common scenario when $N$ is large and $P_N$ assigns similar probabilities to all agents.
Theorem~\ref{theorem:fts} below is our main theoretical result (see Appendix~\ref{app:proof_theorem_fts} for its proof):
\begin{theorem}
\label{theorem:fts}
Let $\gamma_t$ be the maximum information gain on $f$ from any set of $t$ observations.
Let $\delta\in(0,1)$, $\beta_{t} \triangleq B+\sigma\sqrt{2(\gamma_{t-1} + 1 + \log(4/\delta)}$, 
and $c_t \triangleq \beta_t (1 + \sqrt{2\log(|\mathcal{X}|t^2)})$ for all $t\in\mathbb{Z}^+$.
Choose $[p_t]_{t\in\mathbb{Z}^+}$ as a monotonically increasing sequence satisfying $p_t\in(0, 1]$ for all  $t\in\mathbb{Z}^+$, $p_t\rightarrow 1$ as $t\rightarrow +\infty$, and $(1-p_t)c_t \leq (1-p_1)c_1$ for all  $t\in\mathbb{Z}^+\setminus\{1\}$.
With probability of at least $1 - \delta$, the cumulative regret incurred by FTS is\footnote{The $\tilde{\mathcal{O}}$ notation ignores all logarithmic factors.}
\[
 R_T=\tilde{\mathcal{O}}\left(\left(B+1/p_1\right)\gamma_T\sqrt{T}+{\sum}^T_{t=1}\psi_t\right)
\]
where $\psi_t \triangleq 2(1-p_t)\sum^N_{n=1}P_N[n] \Delta_{n,t}$ and 
$\Delta_{n, t} \triangleq \tilde{\mathcal{O}}(M^{-1/2} B t_n^2 + B + \sqrt{\gamma_{t_n}} + \sqrt{M} + d_n + \sqrt{\gamma_t})$.
\end{theorem}
Since $\gamma_T=\mathcal{O}((\log T)^{D+1})$ for the SE kernel~\cite{srinivas2009gaussian}, the first term in the upper bound is sublinear in $T$. 
Moreover, since the sequence of $[p_t]_{t\in\mathbb{Z}^+}$ is chosen to be monotonically increasing and goes to $1$ when $t\rightarrow \infty$, $1-p_t$ goes to $0$ asymptotically.
Therefore, the second term in the upper bound also grows sublinearly.\footnote{Note that for the SE kernel, $\sqrt{\gamma_t}$ in $\Delta_{n,t}$ is logarithmic in $t$: $\sqrt{\gamma_t}=\mathcal{O}((\log t)^{(D+1)/2})$.}
For example, if $[p_t]_{t\in\mathbb{Z}^+}$ is chosen such that $1 - p_t=\mathcal{O}(1/\sqrt{t})$, $\sum^T_{t=1}\psi_t=\tilde{\mathcal{O}}(\sqrt{T})$.
As a result, FTS achieves \emph{no regret} asymptotically regardless of the difference between the target agent and the other agents,
which is a highly desirable property for FBO in which the heterogeneity among agents is a prominent challenge.
Such a robust regret upper bound is achieved because we upper-bound the worst-case error for \emph{any} set of agents (i.e., any set of values of $d_n$ and $t_n$ for $n=1,\ldots,N$) in our proof.
The robust nature of the regret upper bound can be reflected in its dependence on the sequence $[p_t]_{t\in\mathbb{Z}^+}$ as well as on $d_n$ and $t_n$.
When the value of the sequence $[p_t]_{t\in\mathbb{Z}^+}$ is small, i.e., when the information from the other agents is exploited more (Section~\ref{subsec:fts_algo}),
the worst-case error due to more utilization of these information is also increased.
This is corroborated by Theorem~\ref{theorem:fts} since smaller values of $p_t$ increase the regret upper bound through the terms $1/p_1$ and $(1-p_t)$ in $\psi_t$.
Theorem~\ref{theorem:fts} also shows that the regret bound becomes worse with larger values of $d_n$ and $t_n$ because 
a larger $d_n$ increases the difference between the objective functions of $\mathcal{A}_n$ and $\mathcal{A}$, and more observations from $\mathcal{A}_n$ (i.e., larger $t_n$)
also loosens the upper bound since for a fixed $d_n$, a larger number of observations increases the worst-case error by 
accumulating the individual errors.\footnote{In the most general setting where $\mathcal{A}_n$ may collect more observations between different rounds of communication 
such that $t_n$ may increase (Section~\ref{subsec:fts_algo}), $1-p_t$ can decay faster to preserve the no-regret convergence.}

The dependence of the regret upper bound (through $\Delta_{n,t}$) on the number $M$ of random features is particularly interesting due to the interaction between two opposing factors.
Firstly, the $M^{-1/2} B t_n^2$ term arises since better approximation quality of the agent's GP surrogates (i.e., larger $M$) improves the performance.
However, the $\sqrt{M}$ term suggests the presence of another factor with an opposite effect.
This results from the need to upper-bound the distance between the $M$-dimensional Gaussian random variable $\omega_{n}$ and its mean $\nu_n$~\eqref{eq:posterior_dist_omega},
which grows at a rate of $\mathcal{O}(\sqrt{M})$
(Lemma~\ref{lemma:bound_posterior_sample_around_mean} in Appendix~\ref{app:proof_theorem_fts}).
Taking the derivative of both terms w.r.t.~$M$ reveals that the regret bound is guaranteed to become tighter with an increasing $M$
(i.e., the effect of the $M^{-1/2} B t_n^2$ term dominates more) when $t_n$ is sufficiently large, i.e., when $t_n=\Omega(\sqrt{M/B})$.
An intuitive explanation of this finding, which is verified in our experiments (Section~\ref{subsec:exp_synth_func}), 
is that the positive effect (i.e., a tighter regret bound) of better RFF approximation from a larger $M$ is amplified when more observations are available (i.e., $t_n$ is large).
In contrast, when $t_n$ is small, minimal information is offered by agent $\mathcal{A}_n$ and increasing the quality of RFF approximation thus only leads to marginal or negligible improvement in the performance.
The practical implication of this insight is that when the other agents only have a small number of observations, 
it is not recommended to use a large number of random features since it requires a larger communication bandwidth (Section~\ref{subsec:compare_fts_with_rgpe_taf}) yet is unlikely to improve the performance.
\section{Experiments and Discussion}
\label{sec:experiments}
We firstly use synthetic functions to investigate the behavior of FTS.
Next, using $3$ real-world experiments, 
we demonstrate the effectiveness of FTS in terms of communication efficiency, computational efficiency,  and practical performance.
Since it has been repeatedly observed that the theoretical choice of $\beta_t$ that is used to establish the confidence interval is overly conservative~\cite{bogunovic2018adversarially,srinivas2009gaussian},
we set it to a constant: $\beta_t=1.0$.
As a result, $c_t$ (Theorem~\ref{theorem:fts}) grows slowly (i.e., logarithmically) and we thus do not explicitly check the validity of the condition $(1-p_t)c_t \leq (1-p_1)c_1$ for all $t\in\mathbb{Z}^+\setminus\{1\}$.
All error bars represent standard errors.
For simplicity, we focus on the simple setting here where communication happens only before the beginning of FTS (Section~\ref{subsec:fts_algo}). In Appendix~\ref{app:results_increasing_t_n}, we also evaluate the performance of FTS in the most general setting where the other agents are also performing optimization tasks such that they may collect more observations between different rounds of communication (i.e., increasing $t_n$). The results (Fig.~\ref{fig:increasing_tn} in Appendix~\ref{app:results_increasing_t_n}) show consistent performances of FTS in both settings.
More experimental details and results are deferred to Appendix~\ref{app:experiments} due to space constraint.
%
\subsection{Optimization of Synthetic Functions}
\label{subsec:exp_synth_func}
In synthetic experiments, the objective functions are sampled from a GP (i.e., defined on a $1$-D discrete domain within $[0,1]$) using the SE kernel 
and scaled into the range $[0, 1]$.
We fix the total number of agents as $N=50$ and vary $d_n$, $t_n$, and $M$ to investigate their effects on the performance.
We use the same $d_n$ and $t_n$ for all agents for simplicity.
We choose $P_N$ to be uniform: $P_N[n]=1/N$ for $n=1,\ldots,N$ and choose the sequence $[p_t]_{t\in\mathbb{Z}^+}$ as $p_t=1-1/\sqrt{t}$ for all $t\in\mathbb{Z}^+\setminus\{1\}$
and $p_1=p_2$. 
Figs.~\ref{fig:synth_func}a and b show that when $d_n=0.02$ is small, FTS is able to perform better than TS.
Intuitively, the performance advantage of FTS results from its ability to exploit the additional information from the other agents to reduce the need for exploration.
These results also reveal that when $t_n$ of every agent is small (Fig.~\ref{fig:synth_func}a), the effect of $M$ is negligible;
on the other hand, when $t_n$ is large (Fig.~\ref{fig:synth_func}b), increasing $M$ leads to evident improvement in the performance.
This corroborates our theoretical analysis (Section~\ref{sec:theoretical_results}) stating that when $t_n$ is large, 
increasing the value of $M$ is more likely to tighten the regret bound and hence improve the performance.
Moreover, comparing the green curves in Figs.~\ref{fig:synth_func}a and b shows that when the other agents' objective functions are similar to the target agent's objective function (i.e., $d_n=0.02$ is small)
and the RFF approximation is accurate (i.e., $M=100$ is large), increasing the number of observations from the other agents ($t_n=100$ vs.~$t_n=40$) improves the performance.
Lastly, Fig.~\ref{fig:synth_func}c verifies FTS's theoretically guaranteed robustness against heterogeneous agents (Section~\ref{sec:theoretical_results}) 
since it shows that even when all other agents are heterogeneous (i.e., every $d_n=1.2$ is large), the performances of FTS are still comparable to that of standard TS.
Note that Fig.~\ref{fig:synth_func}c demonstrates a potential limitation of our method, i.e., in this scenario of heterogeneous agents, FTS may converge slightly slower than TS if $p_t$ does not grow sufficiently fast.
However, the figure also shows that making $p_t$ grow faster (i.e., making the effect of the other agents decay faster) allows FTS to match the performance of TS in this adverse scenario (red curve).
\begin{figure}
	\centering
	\begin{tabular}{ccc}
		\hspace{-3mm} \includegraphics[width=0.32\linewidth]{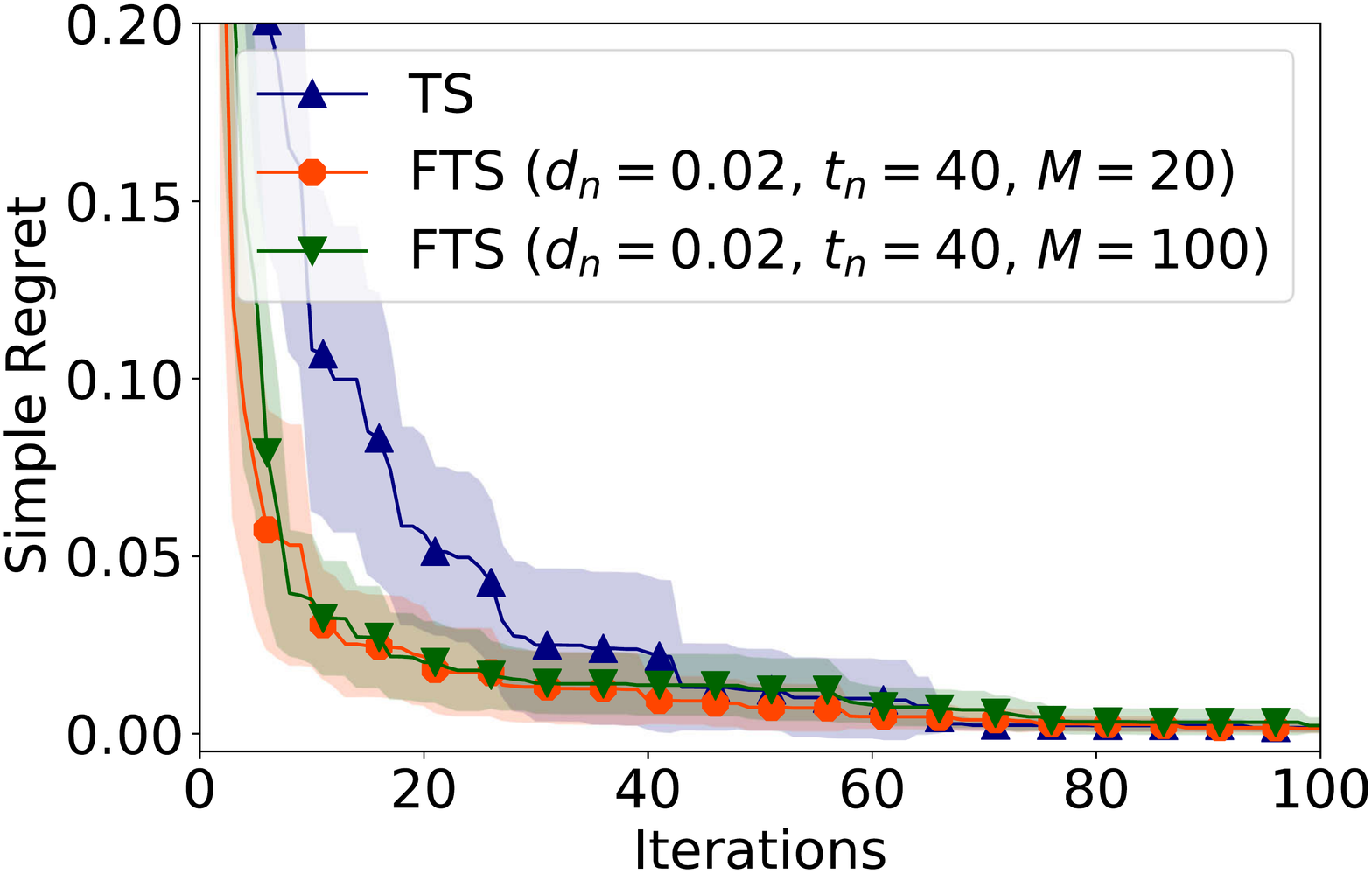} & \hspace{-4mm}
		\includegraphics[width=0.32\linewidth]{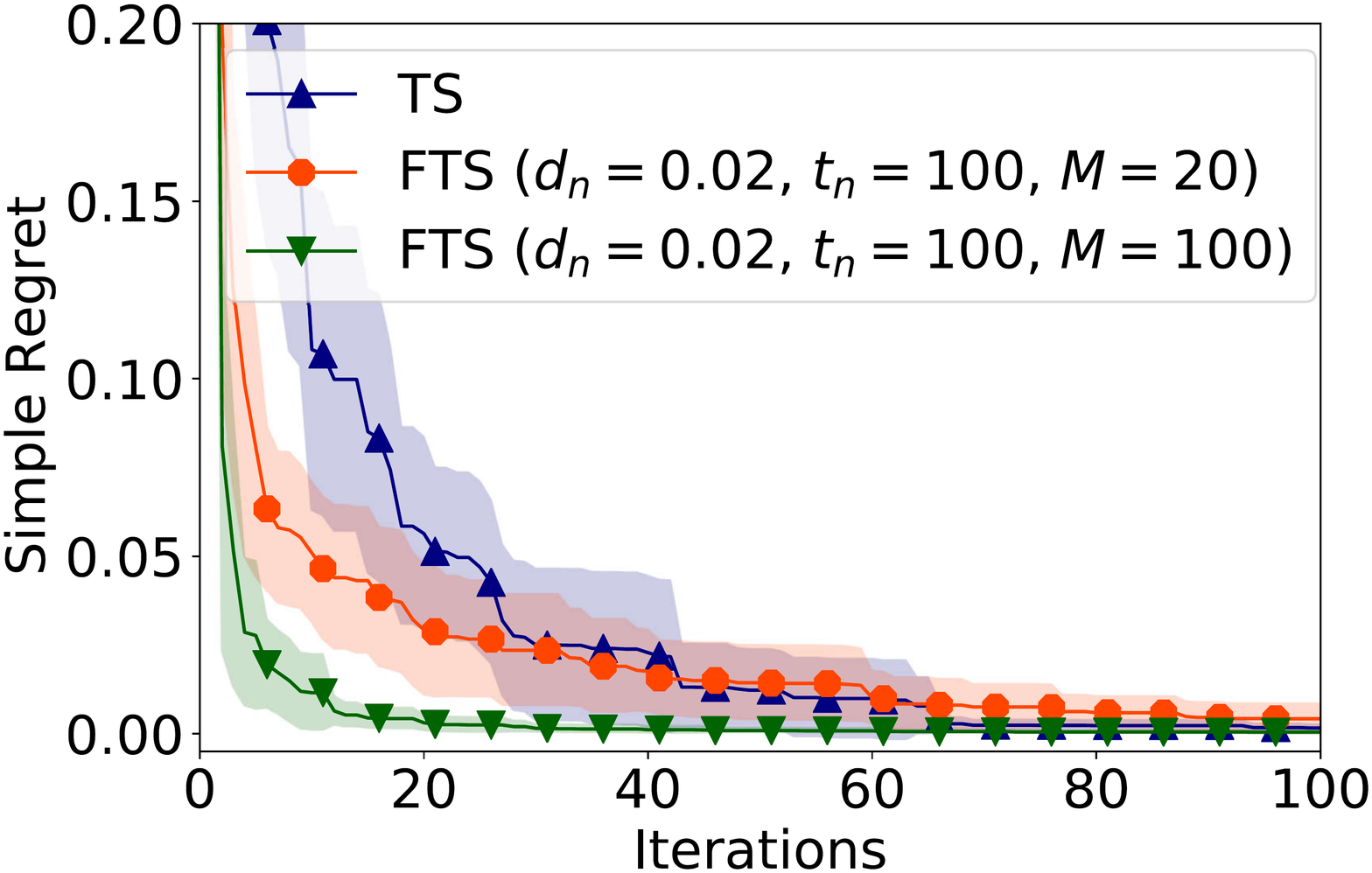} & \hspace{-4mm} 
		\includegraphics[width=0.32\linewidth]{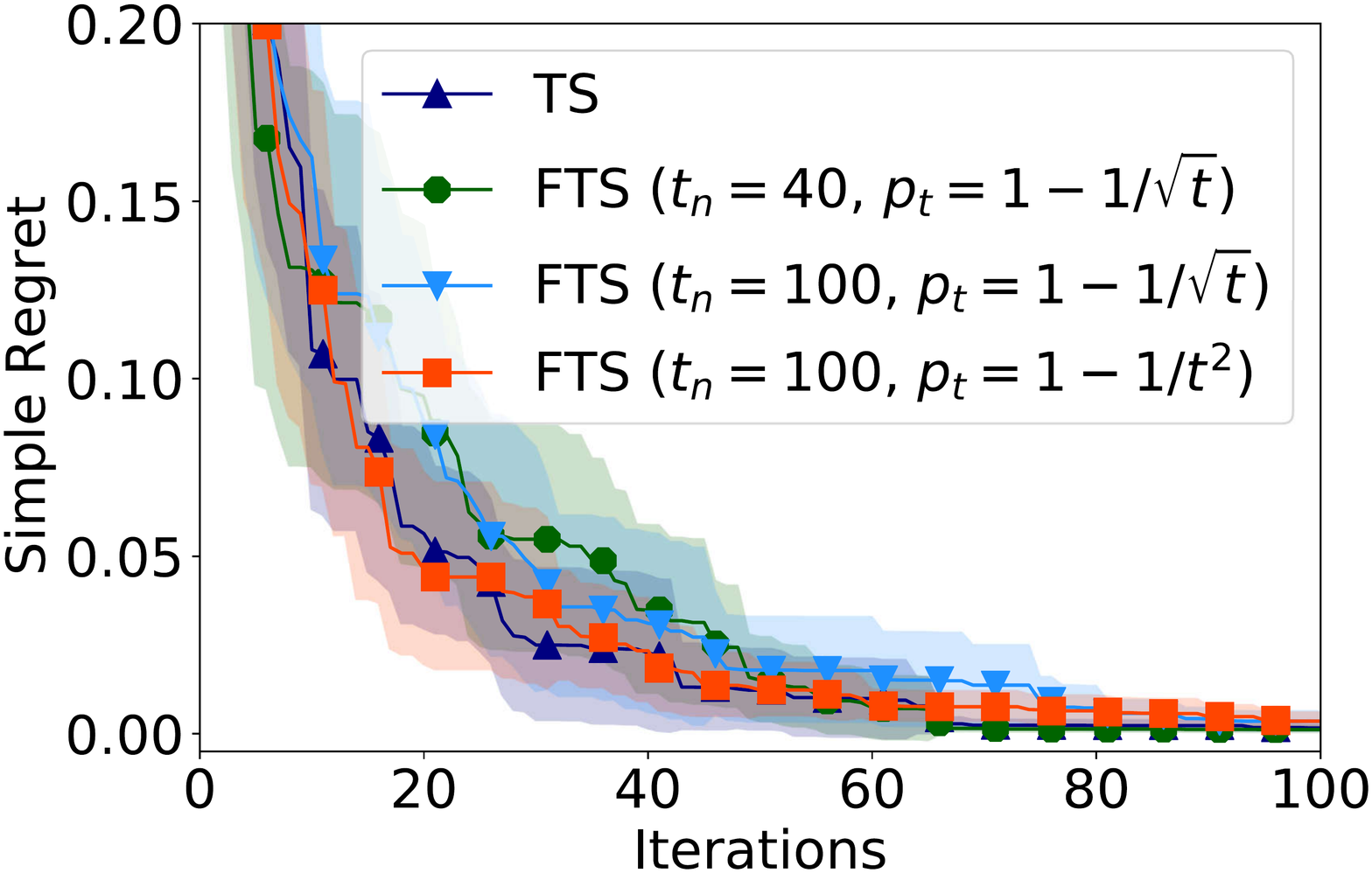}\\
		{(a)} & {(b)} & {(c)}
	\end{tabular}\vspace{-1.5mm}
	\caption{Simple regret in synthetic experiments showing the effect of $M$ when $t_n$ is (a) small and (b) large, and (c) the performance when $d_n=1.2$ is large. Each curve is averaged over $5$ randomly sampled functions from a GP and $5$ random initializations of $1$ input for each function.}
	\label{fig:synth_func}\vspace{-1mm}
\end{figure}
%
%
\subsection{Real-world Experiments}
\label{subsec:exp_real_world}
For real-world experiments, we use $3$ datasets generated in federated settings that naturally contain heterogeneous agents~\cite{smith2017federated}.
Firstly, we use a landmine detection dataset in which the landmine fields are located at two different terrains~\cite{xue2007multi}.
Next, we use two activity recognition datasets collected using Google glasses~\cite{rahman2015unintrusive} and mobile phone sensors~\cite{anguita2013public}, both of which contain heterogeneous agents 
since cross-subject heterogeneity has been a major challenge for human activity recognition~\cite{rahman2015unintrusive}.
We compare our FTS algorithm with standard TS (i.e., no communication with other agents), RGPE, and TAF. Note that RGPE and TAF are meta-learning algorithms for BO and are hence not specifically designed for the FBO setting (Section~\ref{subsec:compare_fts_with_rgpe_taf}).

\textbf{Landmine Detection.} This dataset includes $29$ landmine fields. 
For each field, every entry in the dataset consists of $9$ features and a binary label indicating whether the corresponding location contains landmines.
The task of every field is to tune $2$ hyperparameters of an SVM classifier (i.e., RBF kernel parameter in $[0.01,10]$ and L$2$ regularization parameter in $[10^{-4},10]$) 
that is used to predict whether a location contains landmines.
We fix one of the landmine fields as the target agent and the remaining $N=28$ fields as the other agents, each of whom has completed a BO task of $t_n=50$ iterations.

\textbf{Activity Recognition Using Google Glasses.} This dataset contains sensor measurements from Google glasses worn by $38$ participants. 
Every agent attempts to use $57$ features, which we have extracted from the corresponding participant's measurements, to predict whether the participant is eating or performing other activities.
Every agent uses \emph{logistic regression} (LR) for activity prediction and needs to tune $3$ hyperparameters of LR: batch size ($[20, 60]$), L$2$ regularization parameter ($[10^{-6}, 1]$), and learning rate ($[0.01, 0.1]$).
We fix one of the participants as the target agent and all other $N=37$ participants as the other agents, each of whom possesses $t_n=50$ BO observations.

\textbf{Activity Recognition Using Mobile Phone Sensors.} This dataset consists of mobile phone sensor measurements from $30$ subjects performing $6$ activities.
Each agent attempts to tune the hyperparameters of a subject's activity prediction model whose input includes $561$ features and output is one of the $6$ activity classes.
The activity prediction model and tuned hyperparameters, as well as their ranges, are the same as that in the Google glasses experiment.
We again fix one of the subjects as the target agent and all other $N=29$ subjects as the other agents with $t_n=50$ observations each.
\begin{figure}
	\centering
	\begin{tabular}{ccc}
		\hspace{-3mm} \includegraphics[width=0.32\linewidth]{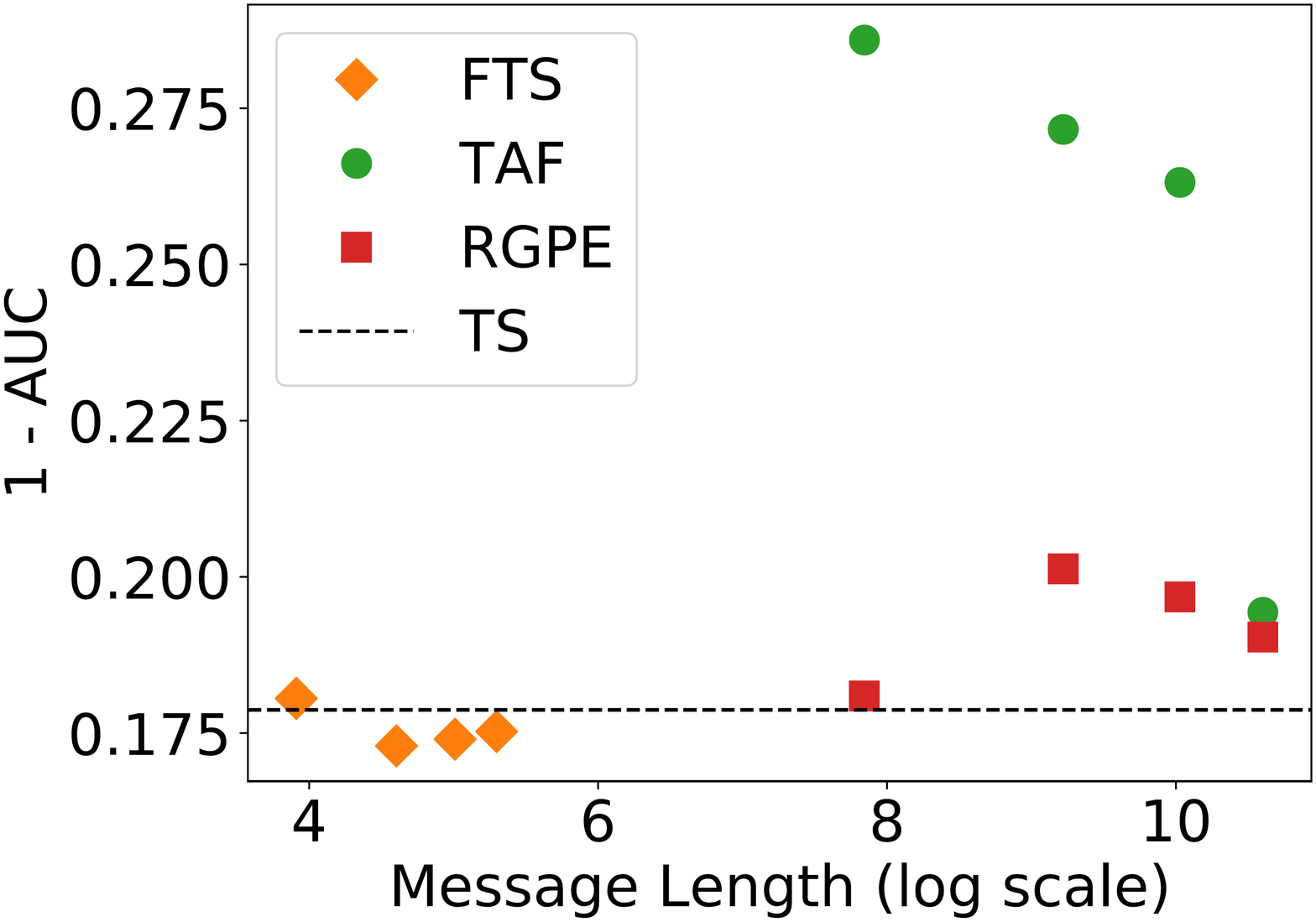} & \hspace{-4mm}
		\includegraphics[width=0.32\linewidth]{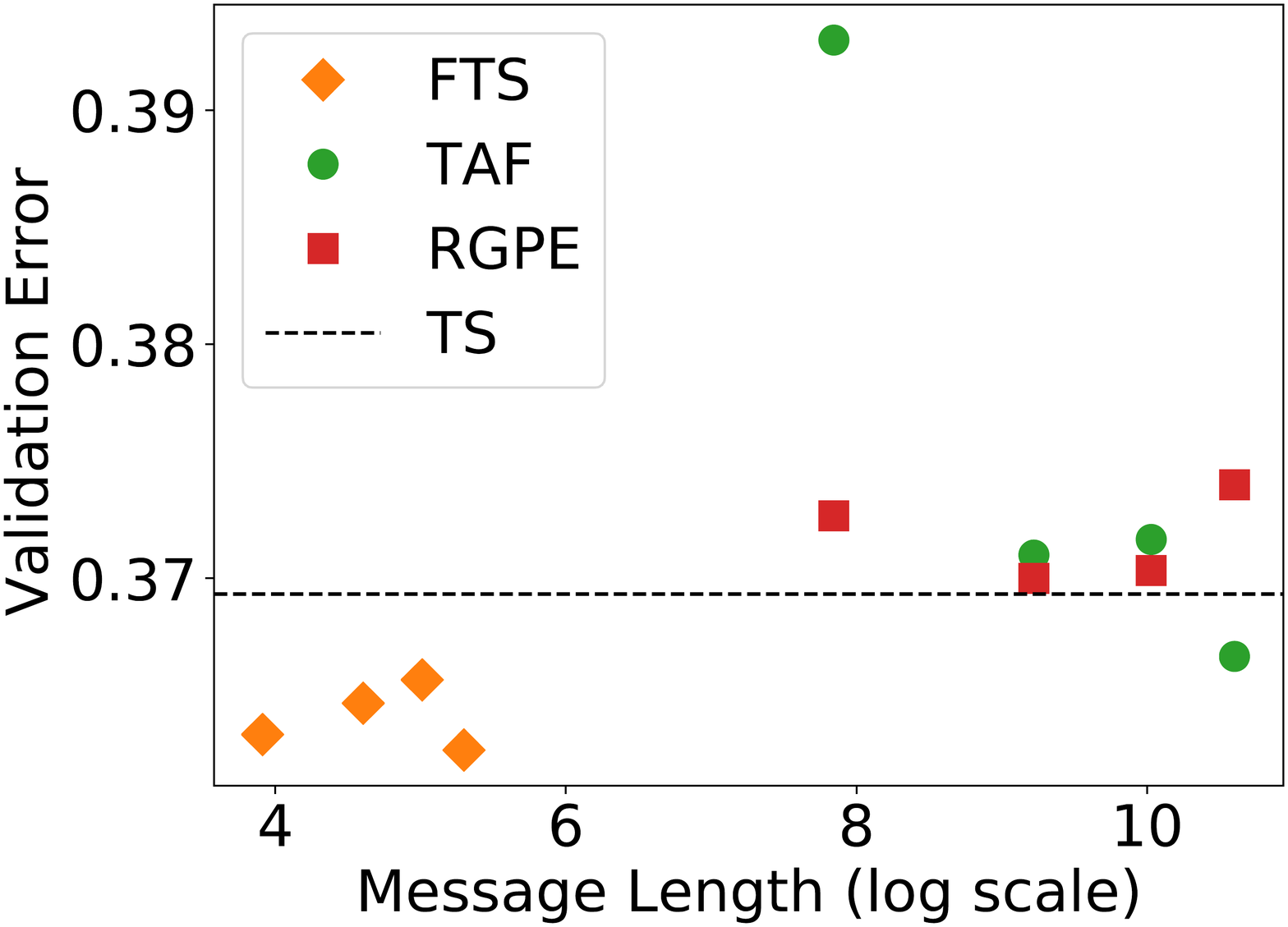} & \hspace{-4mm} 
		\includegraphics[width=0.32\linewidth]{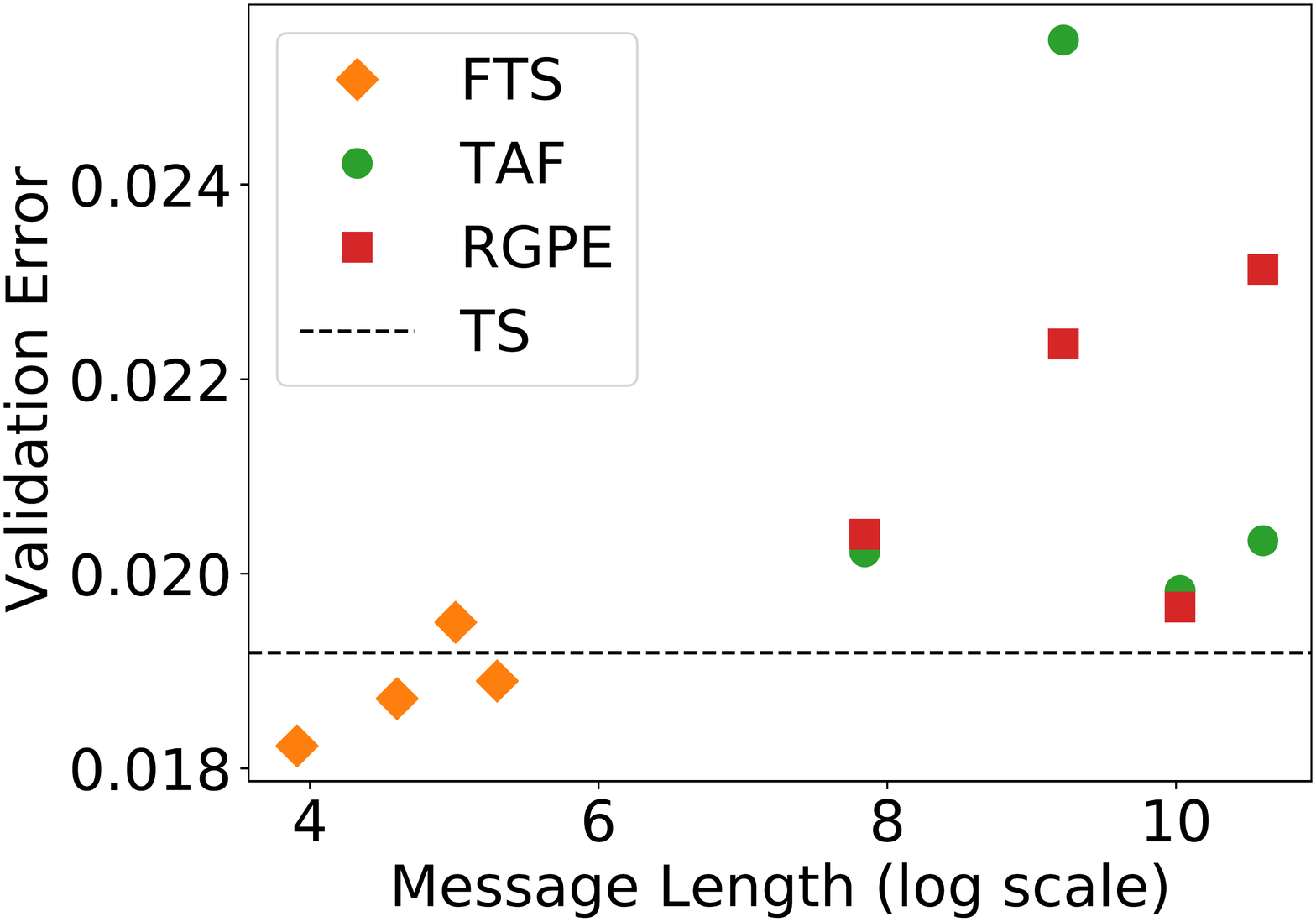}\\
		{(a)} & {(b)} & {(c)}
	\end{tabular}\vspace{-1.5mm}
	\caption{Best performance after $50$ iterations (vertical) vs.~the length of the message (i.e., the number of parameters) communicated from each agent (horizontal) for the (a) landmine detection, 
	(b) Google glasses, and (c) mobile phone sensors experiments.
	The more to the \emph{bottom left}, the better the performance and the less the required communication.
    The results for every method correspond to $M=50,100,150,200$, respectively.
    Every result is averaged over $6$ different target agents and each target agent is averaged over $5$ different initializations of $3$ randomly selected inputs.}
	\label{fig:real_world_exp}\vspace{-2mm}
\end{figure}
\begin{figure}
	\centering
	\begin{tabular}{cccc}
		\hspace{-4mm} \includegraphics[width=0.25\linewidth]{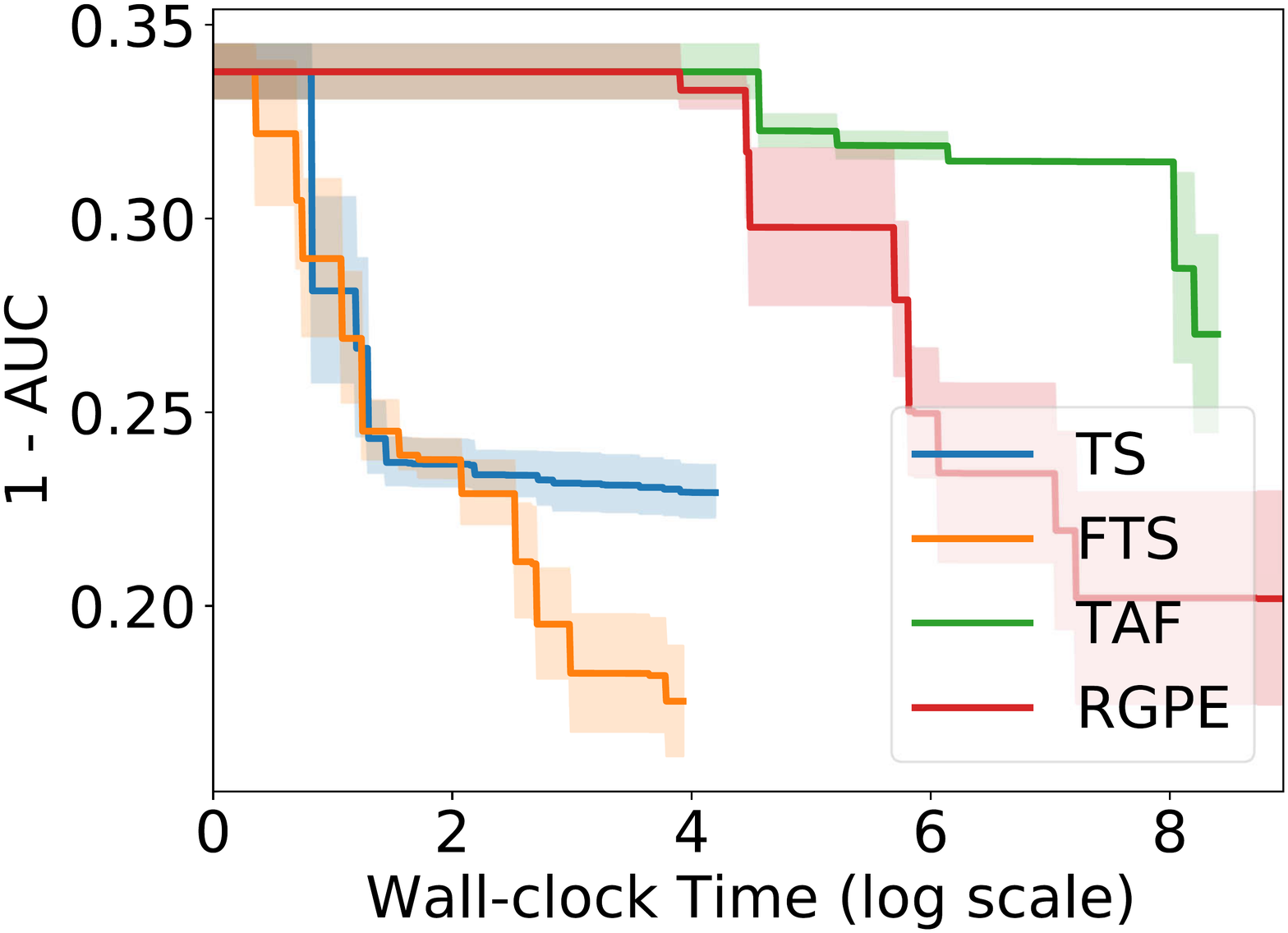} & \hspace{-4mm}
		\includegraphics[width=0.25\linewidth]{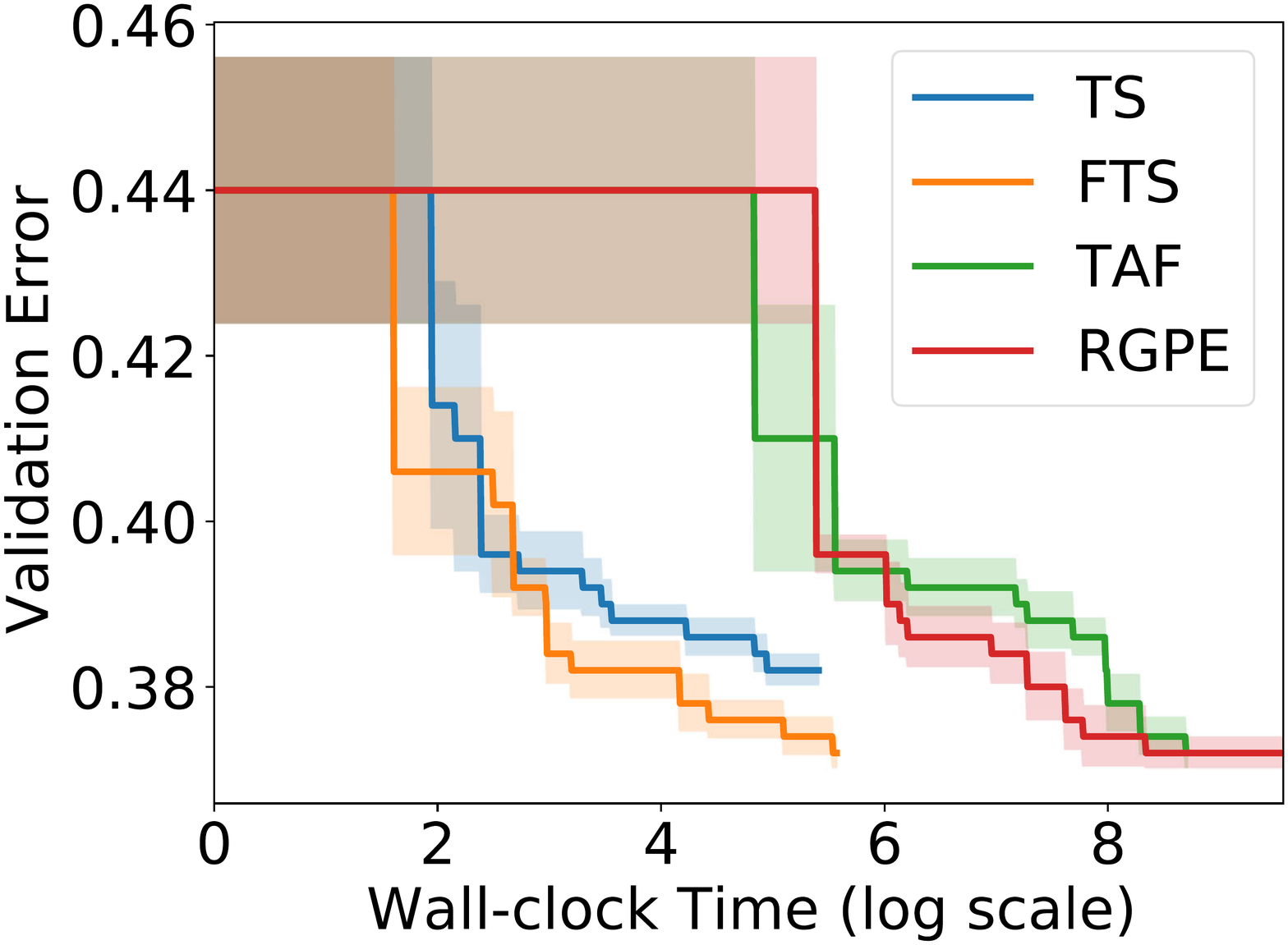} & \hspace{-4mm} 
		\includegraphics[width=0.25\linewidth]{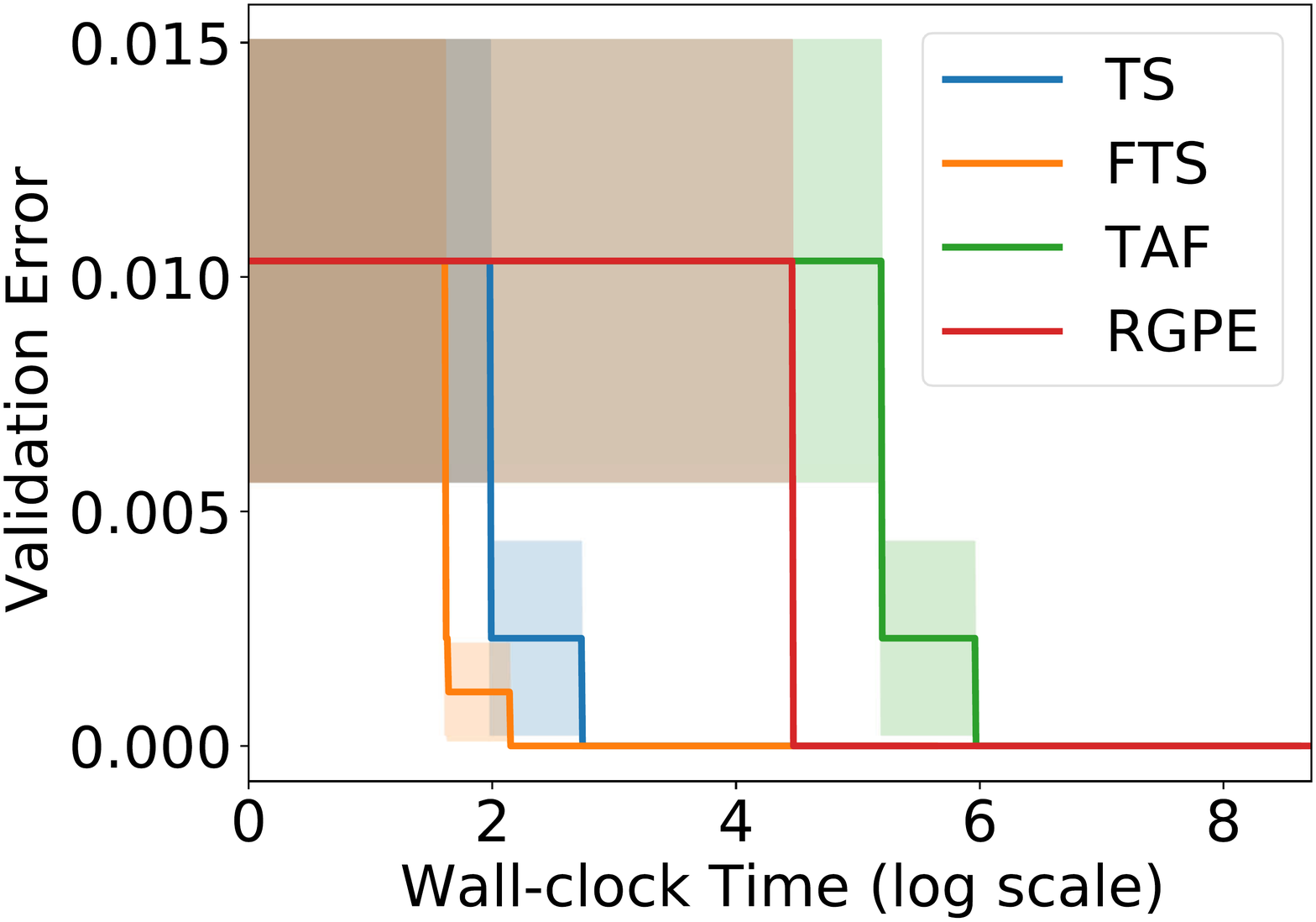} & \hspace{-4mm} 
		\includegraphics[width=0.25\linewidth]{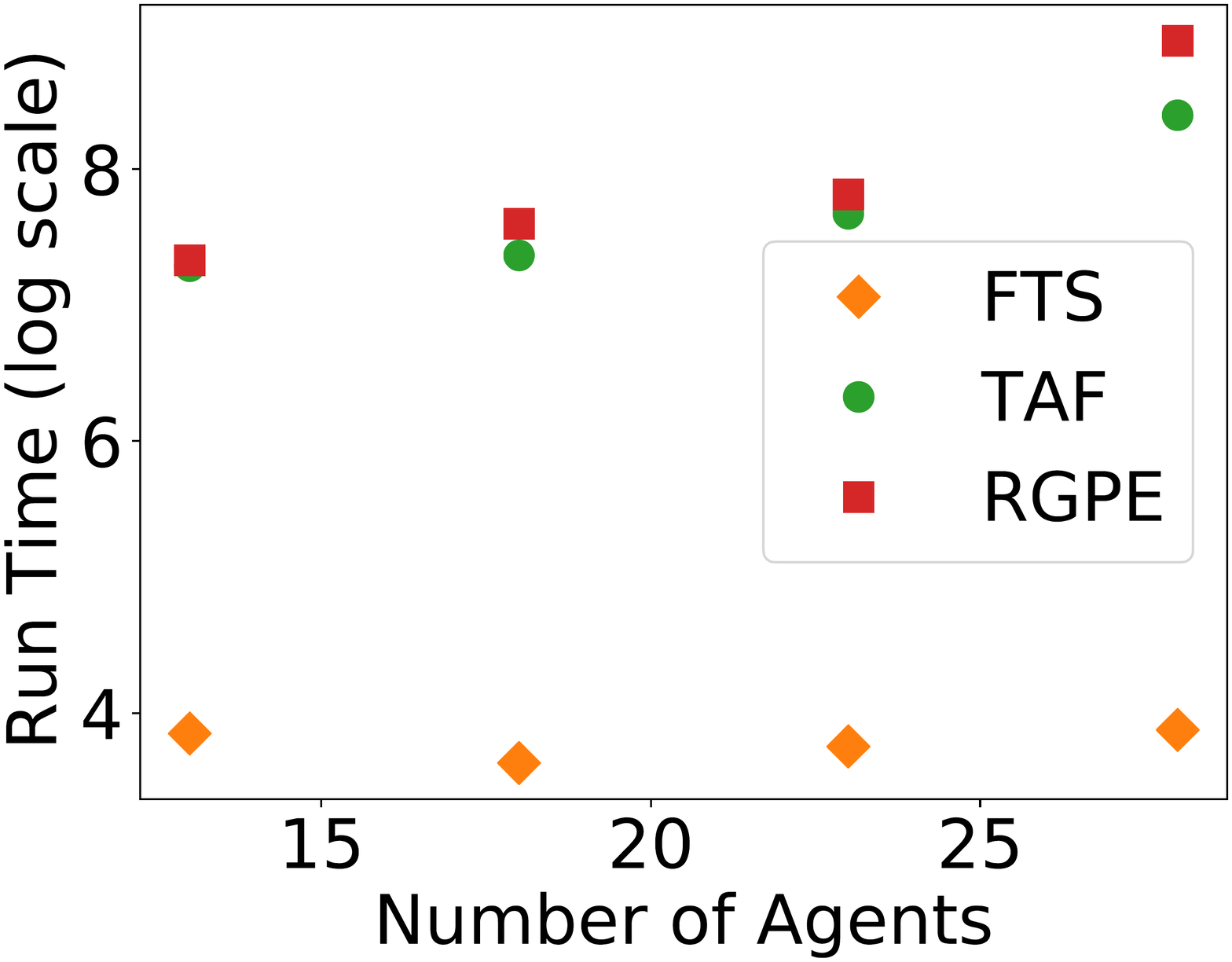}\\
		{(a)} & {(b)} & {(c)} & {(d)}
	\end{tabular}\vspace{-1mm}
	\caption{Best performance observed vs.~run time (seconds) for the (a) landmine detection, (b) Google glasses, and (c) mobile phone sensors experiments,
	in which \emph{FTS converges faster than other methods}.
    These results correspond to the first (of the $6$) target agent used in each experiment in Fig.~\ref{fig:real_world_exp} with $M=100$ and are averaged over $5$ random initializations of $3$ inputs.\cref{footnote_single_agent}
    Every method is run for $50$ iterations.
    (d) Total runtime vs.~the number of agents for the landmine detection experiment.}
	\label{fig:real_world_exp_time_plot}\vspace{-4mm}
\end{figure}

For all experiments, we set $P_N$ to be uniform: $P_N[n]=1/N, \forall n=1,\ldots,N$, and $p_t=1-1/t^2$ for all $t\in\mathbb{Z}^+\setminus\{1\}$ and $p_1=p_2$.
We use validation error as the performance metric for the two activity recognition experiments, and use \emph{area under the receiver operating characteristic curve} (AUC)
to measure the performance of the landmine detection experiment since this dataset is extremely imbalanced 
(i.e., only $6.2\%$ of all locations contain landmines).
We repeat every experiment $6$ times with each time treating one of the first $6$ agents as the target agent.
Fig.~\ref{fig:real_world_exp} shows the (averaged) best performance after $50$ iterations of different methods (vertical axis) as well as their required number of parameters to be passed from each agent (horizontal axis).
FTS outperforms both RGPE and TAF in terms of both the \emph{performance metric} and the \emph{communication efficiency}.
Note that this comparison is unfair for FTS since FTS is much more computationally efficient than RGPE and TAF (Section~\ref{subsec:compare_fts_with_rgpe_taf}) 
such that it completes $50$ iterations in significantly shorter time (Fig.~\ref{fig:real_world_exp_time_plot}).
Fig.~\ref{fig:real_world_exp_time_plot} plots the best performance achieved vs.~the run time of different algorithms 
with the first agent treated as the target agent; 
refer to Appendix~\ref{app:more_experimental_results} for the results of the other agents.\footnote{\label{footnote_single_agent}We cannot average the results across different agents since their output scales vary significantly.}
Fig.~\ref{fig:real_world_exp_time_plot} shows that FTS achieves the fastest convergence among all methods and showcases the advantage of FTS over RGPE and TAF in terms of \emph{computational efficiency} (Section~\ref{subsec:compare_fts_with_rgpe_taf}).
Overall, the consistent performance advantage of FTS across all real-world experiments is an indication of its practical robustness,
which may be largely attributed to its robust theoretical convergence guarantee ensuring its consistent performance even in the presence of heterogeneous agents (Section~\ref{sec:theoretical_results}).
Furthermore, we also use the landmine detection experiment to illustrate the scalability of our method w.r.t. the number $N$ of agents. The results (Fig.~\ref{fig:real_world_exp_time_plot}d) show that increasing $N$ has minimal effect on the runtime of FTS yet leads to growing computational cost for RGPE and TAF. This verifies the relevant discussion at the end of Section~\ref{subsec:compare_fts_with_rgpe_taf}.
%
\section{Related Works}
Since its recent introduction in~\cite{mcmahan2016communication}, FL has gained tremendous attention mainly due to its prominent practical relevance in the collaborative training of ML models such as DNNs~\cite{mcmahan2016communication} or decision tree-based models~\cite{li2020practical,li2020privacy}.
Meanwhile, efforts have also been made to derive theoretical convergence guarantees for FL algorithms~\cite{li2018federated,li2019convergence}. 
Refer to recent surveys~\cite{kairouz2019advances,li2019federated2,li2019federated} for more comprehensive reviews of FL.
TS~\cite{thompson1933likelihood} has been known as a highly effective practical technique for multi-armed bandit problems~\cite{chapelle2011empirical,russo2017tutorial}.
The Bayesian regret~\cite{russo2014learning} and frequentist regret~\cite{chowdhury2017kernelized} of TS in BO have both been analyzed 
and TS has been shown to perform effectively in BO problems such as high-dimensional BO~\cite{mutny2018efficient}.
The theoretical analysis in this work has adopted techniques used in the works of~\cite{chowdhury2017kernelized,mutny2018efficient}.
Our algorithm is also related to multi-fidelity BO~\cite{dai2019bayesian,kandasamy2016gaussian,poloczek2017multi,wu2020practical,zhang2020bayesian,zhang2017information} which has the option to query low-fidelity functions.
This is analogous to our algorithm allowing the target agent to use the information from the other agents for query selection and the similarity between an agent and the target agent can be interpreted as a measure of fidelity.
Moreover, our algorithm also bears similarity to parallel/distributed BO 
algorithms~\cite{contal2013parallel,daxberger2017distributed,desautels2014parallelizing}, 
especially those based on 
TS~\cite{hernandez2017parallel,kandasamy2018parallelised}. 
However, there are fundamental differences: For example, they usually optimize a single objective function whereas we need to consider possibly heterogeneous objective functions from different agents.
On the other hand, BO involving multiple agents with possibly different objective functions has been studied from the perspective of game theory by the works of~\cite{dai2020r2,sessa2019no}.
As discussed in Section~\ref{subsec:compare_fts_with_rgpe_taf}, some works on meta-learning for BO~\cite{feurer2018scalable,wistuba2018scalable},
which study how information from other related BO tasks is used to accelerate the current BO task, can be adapted to the FBO setting.
However, these works do not provide theoretical convergence guarantee nor 
tackle the issues of avoiding the transmission of raw data and achieving efficient communication.
Moreover, their adapted variants for FBO have been shown to be outperformed by our FTS algorithm 
in various major aspects including communication efficiency, computational efficiency, and practical performance (Section~\ref{subsec:exp_real_world}).
%
\section{Conclusion and Future Works}
This paper introduces the first algorithm for the FBO setting called FTS which addresses some key challenges in FBO in a principled manner.
We theoretically show its convergence guarantee which is robust against heterogeneous agents,
and empirically demonstrate its communication efficiency, computational efficiency, and practical effectiveness using three real-world experiments.
As a future work, we plan to explore techniques to automatically optimize the distribution $P_N$ used by FTS to sample agents 
by learning the similarity between each agent and the target agent (i.e., the fidelity of each agent).
Other than the RFF approximation used in this work, other approximation techniques for GP (such as those based on inducing points~\cite{Chen13,LowTASE15,LowRSS13,LowUAI12,MinhAAAI17,HoangICML19,NghiaICML16,HoangICML16,NghiaAAAI19,LowECML14a,low15,Ruofei18,teng20,LowAAAI14,Haibin19,HaibinAPP}) may also be used to derive the parameters to be exchanged between agents, which is worth exploring in future works.
Moreover, in our experiments, the hyperparameters of the target agent's GP is learned by maximizing the marginal likelihood; it would be interesting to explore whether the GP hyperparameters can also be shared among the agents, which can potentially facilitate better collaboration.
Furthermore, our current algorithm is only able to avoid sharing of raw data 
and may hence be susceptible to advanced privacy attacks.
So, it would be interesting to incorporate the state-of-the-art privacy-preserving techniques into our algorithm such as differential privacy~\cite{dwork2006calibrating} which has been applied to BO by the works of~\cite{kharkovskii2020private,kusner2015differentially}.
We will consider incentivizing collaboration in FBO~\cite{Rachael20} and generalizing FBO to nonmyopic BO~\cite{dmitrii20a,ling16} and high-dimensional BO~\cite{NghiaAAAI18} settings.

\section*{Broader Impact}
Since the setting of our FBO is similar to that of FL, our work inherits a number of potential broader impacts of FL. We will analyze here the potential impacts of our work in the scenario where the individual agents are edge devices such as mobile phones since it is a major area of application for FBO and FL.

Specifically, since our algorithm can be used to improve the efficiency of black-box optimization tasks performed by mobile phones, it has the potential of dramatically improving
the efficacy and function of various applications for the user (e.g., the smart keyboard example that we mentioned in Section~\ref{sec:introduction}), which will enhance their user experience and productivity.
On the other hand, some negative impacts of FL also need to be considered when promoting the widespread application of our work.
For example, 
although our work is able to prevent exchanging the raw data in the same way as standard FL,
advanced privacy attack methods (e.g., inference attack) may still incur risks of privacy leak for FL and thus FBO.
Preventing this risk through principled privacy protection techniques (e.g., differential privacy) is important for the widespread adoption of FL and FBO algorithms 
and hence represents an interesting and promising direction for future research.

\begin{ack}
This research/project is supported by A*STAR under its RIE$2020$ Advanced Manufacturing and Engineering (AME) Industry Alignment Fund -- Pre Positioning (IAF-PP) (Award A$19$E$4$a$0101$). 
\end{ack}

\bibliography{fbo}
\bibliographystyle{abbrv}

\newpage
\appendix

\section{Construction of Random Fourier Features}
\label{app:construct_random_features}
As mentioned in Section~\ref{sec:background}, in this work, we focus on the widely used \emph{squared exponential} (SE) kernel: $k(x,x')=\sigma^2_0\exp (-\norm{x-x'}_2^2/(2l^2))$ 
in which $l$ is the length scale and $\sigma_0^2$ is the signal variance.
$\sigma_0^2=1$ is usually the default value, which we use in all experiments.
We construct the random features following the work of~\cite{rahimi2008random}.
Specifically, for the SE kernel with length scale $l$, the spectral density follows a $D$-dimensional Gaussian distribution: $p(s)=\mathcal{N}(0, \frac{1}{l^2}I_{D\times D})$.
To begin with, we draw $M$ independent samples of $\{s_i\}_{i=1,\ldots,M}$ from $p(s)$ (every $s_i$ is a $D$-dimensional vector), 
and $M$ independent samples of $\{b_i\}_{i=1,\ldots,M}$ from the uniform distribution over $[0, 2\pi]$ (every $b_i$ is a scalar).
Next, for an input $x$, the corresponding $M$-dimensional random features (basis functions) can be constructed as $\phi(x)^{\top}=[\sqrt{2/M}\cos(s_i^{\top}x+b_i)]_{i=1,\ldots,M}$.
Each set of random features $\phi(x)$ is then normalized such that $\norm{\phi(x)}_2^2=\sigma^2_0$.
As a result, sharing the random features $\phi(x),\forall x\in\mathcal{X}$ among all agents (Section~\ref{sec:background}) can be achieved by simply sharing the parameters $\{s_i\}_{i=1,\ldots,M}$ and $\{b_i\}_{i=1,\ldots,M}$.
This is easily achievable since it is equivalent to sharing the parameters of the first layer of a neural network model with $M$ units in the hidden layer,
in which $\{s_i\}_{i=1,\ldots,M}$ are the weights (which form a $D\times M$-dimensional weight matrix) and $\{b_i\}_{i=1,\ldots,M}$ are the biases.

\section{GP Posterior/Predictive Belief with Random Fourier Features Approximation}
\label{app:gp_posterior_rff}
Here we derive the expressions of the posterior/predictive mean and variance of a GP with random Fourier features (RFF) approximation (Section~\ref{sec:background}).
Recall that we have defined $\Phi(X_t)=[\phi(x_1),\ldots,\phi(x_t)]^{\top}$ which is a $t\times M$-dimensional matrix.

With the RFF approximation, the kernel function is approximated by $k(x, x') \approx \phi(x)^{\top}\phi(x')$.
Define $\hat{K}_t=[\phi(x_{t'})^{\top}\phi(x_{t''})]_{t',t''=1,\ldots,t}=\Phi(X_t)\Phi(X_t)^{\top}$ and $\hat{k}_t(x)=[\phi(x)^{\top}\phi(x_{t'})]^{\top}_{t'=1,\ldots,t}=\Phi(X_t)\phi(x)$,
which are analogous to $K_t$ and $k_t(x)$ in~\eqref{gp_posterior} with the kernel values $k(x,x')$ replaced by the approximate kernel values $\phi(x)^{\top}\phi(x')$.
With these definitions, we have that
\begin{equation}
\begin{split}
\Phi(X_t)^{\top}\left[\hat{K}_{t} + \sigma^2I\right]&=\Phi(X_t)^{\top}\left[\Phi(X_t)\Phi(X_t)^{\top} + \sigma^2I\right]\\
&=\Phi(X_t)^{\top}\Phi(X_t)\Phi(X_t)^{\top} + \sigma^2\Phi(X_t)^{\top}\\
&=\left[\Phi(X_t)^{\top}\Phi(X_t) + \sigma^2I\right]\Phi(X_t)^{\top}\\
&=\Sigma_t\Phi(X_t)^{\top}.
\end{split}
\end{equation}
Multiplying both sides by $\Sigma^{-1}_t$ from the left and $(\hat{K}_t+\sigma^2I)^{-1}$ from the right, we get
\begin{equation}
\begin{split}
\Sigma^{-1}_t\Phi(X_t)^{\top}=\Phi(X_t)^{\top}(\hat{K}_t+\sigma^2I)^{-1}.
\end{split}
\end{equation}
Then multiplying both sides by $\phi(x)^{\top}$ from the left and $y_t$ from the right, we get
\begin{equation}
\begin{split}
\hat{\mu}_t(x)=\phi(x)^{\top}\nu_t=\phi(x)^{\top}\Sigma^{-1}_t\Phi(X_t)^{\top}y_t&=\phi(x)^{\top}\Phi(X_t)^{\top}(\hat{K}_t+\sigma^2I)^{-1}y_t\\
&=\hat{k}_t(x)^{\top}(\hat{K}_t+\sigma^2I)^{-1}y_t,
\end{split}
\label{eq:post_mean_derivation}
\end{equation}
which proves that the expression of the approximate posterior mean with RFF approximation: $\hat{\mu}_t(x)=\phi(x)^{\top}\nu_t$ 
matches the expression of the posterior mean of standard GP without RFF approximation, 
except that the kernel values $k(x,x')$ are replaced by the approximate kernel values $\phi(x)^{\top}\phi(x')$.

Next, we derive the expression of the approximate posterior variance. Making use of the matrix inversion lemma, we get
\begin{equation}
\begin{split}
\hat{\sigma}^{2}_{t}(x)&=\sigma^2\phi(x)^{\top}\Sigma_t^{-1}\phi(x)=\sigma^2\phi(x)^{\top}(\Phi(X_t)^{\top}\Phi(X_t) + \sigma^2I)^{-1}\phi(x)\\
&=\sigma^2\phi(x)^{\top}\left[\frac{1}{\sigma^2} I - \frac{1}{\sigma^2} \Phi(X_t)^{\top}\left( I + \Phi(X_t) \frac{1}{\sigma^2} \Phi(X_t)^{\top} \right)^{-1} \Phi(X_t)\frac{1}{\sigma^2}\right] \phi(x)\\
&=\phi(x)^{\top}\phi(x) - \phi(x)^{\top}\Phi(X_t)^{\top}\left(\sigma^2I + \Phi(X_t)^{\top}\Phi(X_t) \right)^{-1}\Phi(X_t)\phi(x)\\
&=\hat{k}(x, x) - \hat{k}_t(x)^{\top}\left( \hat{K}_t + \sigma^2I \right)^{-1} \hat{k}_t(x),
\end{split}
\label{eq:post_var_derivation}
\end{equation}
which gives the expression of the approximate posterior variance: $\hat{\sigma}^{2}_{t}(x)=\sigma^2\phi(x)^{\top}\Sigma_t^{-1}\phi(x)$.
To conclude, Equations~\eqref{eq:post_mean_derivation} and~\eqref{eq:post_var_derivation} prove that 
the expressions of the GP posterior mean and variance with RFF approximation given in Section~\ref{sec:background} (in the paragraph after Equation~\eqref{eq:sigma_nu}) match the corresponding expressions of standard GP posterior
mean and variance without RFF approximation~\eqref{gp_posterior}, except that the original kernel values (i.e., $k(x,x')$) are replaced by the corresponding approximate kernel values (i.e., $\phi(x)^{\top}\phi(x')$).

\section{Proof of Theorem~\ref{theorem:fts}}
\label{app:proof_theorem_fts}
As mentioned in Section~\ref{sec:theoretical_results}, we analyze our FTS algorithm in the more general setting 
in which a message can be received from each agent $\mathcal{A}_n$ before every iteration $t$, instead of only before the first iteration.
Therefore, throughout our theoretical analysis, we use $\omega_{n,t}$, instead of $\omega_n$, to denote the message received from agent $\mathcal{A}_n$ before iteration $t$.
Similarly, we use $\hat{g}_{n,t}$, instead of $\hat{g}_{n}$, to denote the corresponding sampled function from agent $\mathcal{A}_n$ with RFF approximation in iteration $t$, obtained using $\omega_{n,t}$:
$\hat{g}_{n,t}(x)=\phi(x)^{\top}\omega_{n,t},\forall x\in\mathcal{X}$.
Note that our theoretical analysis and results also hold in the most general setting where every agent $\mathcal{A}_n$ may collect more observations between different rounds of communication,
in which the only difference is that every $t_n,\forall n=1,\ldots,N$ may increase over different iterations.

Define $\mathcal{F}_{t}$ as the filtration containing agent $\mathcal{A}$'s history of selected inputs and observed outputs up to iteration $t$.
Let $\delta\in(0,1)$, 
we have defined in Theorem~\ref{theorem:fts} that $\beta_{t} = B+\sigma\sqrt{2(\gamma_{t-1} + 1 + \log(4/\delta)}$
and $c_t = \beta_t (1 + \sqrt{2\log(|\mathcal{X}|t^2)})$ for all $t\in\mathbb{Z}^+$.
Clearly, both $\beta_t$ and $c_t$ are increasing in $t$.
Denote by $A_t$ the event that agent $\mathcal{A}$ chooses $x_t$ by maximizing a sampled function from its own GP posterior belief (i.e., $x_t=\arg\max_{x\in \mathcal{X}}f_t(x)$, as in line $4$ of Algorithm~\ref{alg:F_TS}), 
which happens with probability $p_t$; 
denote by $B_t$ the event that $\mathcal{A}$ chooses $x_t$ by maximizing the sampled function from any other agent $\mathcal{A}_1,\ldots,\mathcal{A}_N$ (line $6$ of Algorithm~\ref{alg:F_TS}), 
which happens with probability $(1-p_t)$; 
denote by $B_{t,n}$ the event that $\mathcal{A}$ chooses $x_t$ by maximizing the sampled function of agent $\mathcal{A}_n$ using RFF approximation 
(i.e., $x_t=\arg\max_{x\in\mathcal{X}}\hat{g}_{n,t}(x)$), which happens with probability $(1-p_t) \times P_N[n]$.

To begin with, we define two high-probability events through the following lemmas.
\begin{lemma}
\label{lemma:uniform_bound}
Let $\delta \in (0, 1)$. Define $E^f(t)$ as the event that $|\mu_{t-1}(x) - f(x)| \leq \beta_t \sigma_{t-1}(x)$ for all $x\in \mathcal{X}$.
We have that $\mathbb{P}\left[E^f(t)\right] \geq 1 - \delta / 4$ for all $t\geq 1$.
\end{lemma}
Lemma~\ref{lemma:uniform_bound} quantifies the concentration of the function $f$ around its posterior mean and its proof follows directly from Theorem 2 of the work of~\cite{chowdhury2017kernelized} by using an error probability of $\delta/4$.
\begin{lemma}
\label{lemma:uniform_bound_t}
Define $E^{f_t}(t)$ as the event that $|f_t(x) - \mu_{t-1}(x)| \leq \beta_t \sqrt{2\log(|\mathcal{X}|t^2)} \sigma_{t-1}(x)$.
We have that $\mathbb{P}\left[E^{f_t}(t) | \mathcal{F}_{t-1}\right] \geq 1 - 1 / t^2$ for any possible filtration $\mathcal{F}_{t-1}$.
\end{lemma}
Lemma~\ref{lemma:uniform_bound_t} illustrates how concentrated a sampled function $f_t$ from $f$ is around its posterior mean and is a simpler version of Lemma 5 of the work of~\cite{chowdhury2017kernelized}. Specifically, 
we have assumed a discrete domain, whereas the work of~\cite{chowdhury2017kernelized} deals with a compact domain.
Note that both events $E^{f}(t)$ and $E^{f_t}(t)$ are $\mathcal{F}_{t-1}$-measurable.

Next, we define a set of inputs at every iteration $t$ called \emph{saturated points}, which represents the set of ``bad'' inputs at iteration $t$.
These inputs are ``bad'' in the sense that the function values at these inputs have relatively large difference from the value of the global maximum of $f$.
In the subsequent proof, we will lower-bound the probability that the selected input $x_t$ is \emph{unsaturated}, which will be a critical step in the proof.
\begin{definition}
\label{def:saturated_set}
Define the set of saturated points at iteration $t$ as
\[
S_t = \{ x \in \mathcal{X} : \Delta(x) > c_t \sigma_{t-1}(x) \}
\]
in which $\Delta(x) = f(x^*) - f(x)$ and $x^* = \arg\max_{x\in \mathcal{X}}f(x)$.
\end{definition}
Note that from this definition, $x^*$ is always unsaturated since $\Delta(x)=f(x^*) - f(x^*) = 0 < c_t \sigma_{t-1}(x^*)$ for all $t\geq 1$. Also note that $S_t$ is $\mathcal{F}_{t-1}$-measurable.

The next lemma bounds the deviation of the sampled function $\hat{g}_{n,t}(x)$ from agent $\mathcal{A}_n$'s GP posterior belief with RFF approximation around its posterior mean $\hat{\mu}_{n,t}(x)$, 
whose proof is based on that of Lemma 11 of~\cite{mutny2018efficient}.
\begin{lemma}
\label{lemma:bound_posterior_sample_around_mean}
Given $\delta \in (0, 1)$. We have that for all agents $\mathcal{A}_n,\forall n=1,\ldots,N$, all $x\in \mathcal{X}$ and all $t\geq 1$, with probability of at least $1 - \delta/4$
\[
|\hat{\mu}_{n,t}(x) - \hat{g}_{n,t}(x)| \leq \sqrt{2\log\frac{2\pi^2t^2N}{3\delta} + M}.
\]
\end{lemma}
\begin{proof}
Recall from Section~\ref{sec:background} that the sampled function $\hat{g}_{n,t}$ is obtained by firstly sampling $\omega_{n,t} \sim \mathcal{N}(\nu_{n,t}, \sigma^2\Sigma_{n,t}^{-1})$, 
and then setting $\hat{g}_{n,t}(x)=\phi(x)^{\top}\omega_{n,t}, \forall x \in \mathcal{X}$.
Moreover, we have shown in Section~\ref{sec:background} that $\hat{\mu}_{n,t}(x) = \phi(x)^{\top}\nu_{n,t}$.
Denote $\omega_{n,t} = \nu_{n,t} + \sigma \Sigma_{n,t}^{-1/2} z$, in which $z\sim\mathcal{N}(0,I)$ is the $M\times 1$-dimensional standard Gaussian distribution.
We have that
\begin{equation}
\begin{split}
|\hat{\mu}_{n,t}(x) - \hat{g}_{n,t}(x)|^2 &= |\phi(x)^{\top}\nu_{n,t} - \phi(x)^{\top}(\nu_{n,t} + \sigma \Sigma_{n,t}^{-1/2} z)|^2\\
&=|\sigma\phi(x)^{\top} \Sigma_{n,t}^{-1/2} z|^2\\
&\leq \sigma^2 \norm{\phi(x)^{\top} \Sigma_{n,t}^{-1/2}}_2^2 \norm{z}_2^2\\
&=\sigma^2\phi(x)^{\top} \Sigma_{n,t}^{-1}\phi(x) \norm{z}_2^2\\
&=\hat{\sigma}^2_{n,t}(x)\norm{z}_2^2 \leq \norm{z}_2^2,
\end{split}
\end{equation}
in which we have made use of
the assumption w.l.o.g.~that the posterior variance is upper-bounded by $1$ in the last inequality.
Next, making use of the concentration of chi-squared distribution: $\mathbb{P}(\norm{z}_2^2 \geq M + 2\lambda) \leq \exp(-\lambda)$~\cite{mutny2018efficient}, 
we have that with probability of at least $1 - \frac{3\delta}{2\pi^2t^2N}$,
\begin{equation}
\norm{z}_2^2 \leq M + 2\log\frac{2\pi^2t^2N}{3\delta}.
\end{equation}
Taking a union bound over all agents $\mathcal{A}_1,\ldots,\mathcal{A}_N$ and over all $t\geq 1$ completes the proof.
\end{proof}

The following lemma uniformly upper-bounds the difference between agent $\mathcal{A}_n$'s objective function $g_{n}$ and sampled function $\hat{g}_{n,t}$ from its GP posterior belief with RFF approximation.
\begin{lemma}
\label{lemma:bound_gt_gn}
Given any $\delta \in (0, 1)$. 
For agent $\mathcal{A}_n$'s sampled function $\hat{g}_{n,t}$ from its GP posterior belief with RFF approximation, we have that for all agents $\mathcal{A}_n,\forall n=1,\ldots,N$, all $x\in \mathcal{X}$ and all $t\geq 1$, with probability of at least $1 - \delta/2$,
\[
|\hat{g}_{n,t}(x) - g_{n}(x)|\leq \tilde{\Delta}_{n,t},
\]
where $\beta'_{t} = B+\sigma\sqrt{2(\gamma_{t-1} + 1 + \log(8N/\delta)}$, and
\[
\tilde{\Delta}_{n,t} \triangleq \varepsilon\frac{(t_n+1)^2}{\sigma^2}\left(B +  \sqrt{2\log\left(\frac{4\pi^2t^2N}{3\delta}\right)}  \right) + \beta'_{t_n+1} + \sqrt{2\log\frac{2\pi^2t^2N}{3\delta} + M}.
\]
\end{lemma}
\begin{proof}
Recall that $\varepsilon$ is the accuracy of the RFF approximation, $t_n$ is the number of iterations that agent $\mathcal{A}_n$ has completed in its own BO task when it passes information to $\mathcal{A}$,
$M$ is the number of random features used in the RFF approximation.
Denote by $\hat{\mu}_{n,t}(x)$ and $\mu_{n,t}(x)$ ($\hat{\sigma}_{n,t}(x)$ and $\sigma_{n,t}(x)$) the posterior mean (standard deviation) at $x$ of agent $\mathcal{A}_n$'s GP after running its own BO task for $t_n$ iterations with and without the RFF approximation respectively. 

We have that for all $x\in \mathcal{X}$, all agents $\mathcal{A}_n$ and all $t\geq 1$, with probability of at least $1 - {\delta}/{8}$,
\begin{equation}
    |\mu_{n,t}(x) - \hat{\mu}_{n,t}(x)| \leq \varepsilon\frac{(t_n+1)^2}{\sigma^2}\left(B + \sqrt{2\log\left(\frac{4\pi^2t^2N}{3\delta}\right)}\right),
\end{equation}
which can be proved by following the proof of Theorem 5 in the work of~\cite{mutny2018efficient} by substituting the error probability of 
$\frac{3\delta}{4\pi^2t^2N}$,
and taking a union bound over all agents and all $t\geq 1$.
Next, making use of Lemma~\ref{lemma:uniform_bound} (replacing $f$ by $g_n$, and $\delta/4$ by $\delta/(8N)$), we get
\begin{equation}
|\mu_{n,t}(x) - g_n(x)| \leq \beta'_{t_n+1}\sigma_{n,t}(x) \leq \beta'_{t_n+1},
\end{equation}
which holds for all $x\in \mathcal{X}$, agents $\mathcal{A}_n$ and $t_n\geq 1$, with probability of at least $1 - {\delta}/{8}$.
The last inequality follows from our assumption w.l.o.g. that the posterior variance is upper-bounded by $1$. 


Combining the two equations above and making use of Lemma~\ref{lemma:bound_posterior_sample_around_mean} completes the proof.
\end{proof}

The next lemma shows a uniform upper bound on the difference between the sampled function $f_t$ of agent $\mathcal{A}$ and that of agent $\mathcal{A}_n$ with RFF approximation ($\hat{g}_{n,t}$).
\begin{lemma}
\label{lemma:bound_gt_ft}
At iteration $t$, conditioned on the events $E^{f}(t)$ and $E^{f_t}(t)$, we have that for all agents $\mathcal{A}_n,\forall n=1,\ldots,N$ and for all $x\in \mathcal{X}$ with probability $\geq 1 - \delta/2$
\[
|\hat{g}_{n,t}(x) - f_t(x)|\leq \Delta_{n,t},
\]
in which 
\[
\Delta_{n, t} \triangleq \varepsilon\frac{(t_n+1)^2}{\sigma^2}\left(B +  \sqrt{2\log\left(\frac{4\pi^2t^2N}{3\delta}\right)}  \right) + \beta'_{t_n+1} + \sqrt{2\log\frac{2\pi^2t^2N}{3\delta} + M} + d_n + c_t.
\]
\end{lemma}
\begin{proof}
Firstly, note that since we condition on both events $E^{f}(t)$ and $E^{f_t}(t)$, we have that for all $x\in \mathcal{X}$ and all $t\geq 1$
\begin{equation}
\begin{split}
|f(x) - f_t(x)| &\leq |f(x) - \mu_{t-1}(x)| + |\mu_{t-1}(x) - f_t(x)|\\
&= \beta_t\sigma_{t-1}(x) + \beta_t \sqrt{2\log(|\mathcal{X}|t^2)} \sigma_{t-1}(x)=c_t\sigma_{t-1}(x)
\end{split}
\label{eq:combine_two_events}
\end{equation}
Next,
\begin{equation}
\begin{split}
|\hat{g}_{n,t}(x) - f_t(x)| &\leq
|\hat{g}_{n,t}(x) - g_n(x)| + |g_n(x) - f(x)| + |f(x) - f_t(x)|\\
&\leq \tilde{\Delta}_{n,t} + d_n + c_t\sigma_{t-1}(x)\\
&\leq \tilde{\Delta}_{n,t} + d_n + c_t,
\end{split}
\end{equation}
in which we have made use of Lemma~\ref{lemma:bound_gt_gn}, 
the definition of $d_n$: $d_n=\max_{x\in \mathcal{X}}|f(x) - g_n(x)|$ (Section~\ref{sec:background}, last paragraph),
Equation~\eqref{eq:combine_two_events}, and the assumption that the posterior variance is upper-bounded by $1$.
Plugging in the expression of $\tilde{\Delta}_{n,t}$ from Lemma~\ref{lemma:bound_gt_gn} completes the proof.
\end{proof}


\begin{lemma}
\label{lemma:uniform_lower_bound}
For any filtration $\mathcal{F}_{t-1}$, conditioned on the events $E^f(t)$ and $A_t$, we have that for every $x\in \mathcal{X}$,
\begin{equation}
\mathbb{P}\left(f_t(x) > f(x) | \mathcal{F}_{t-1},E^f(t), A_t\right) \geq p,
\label{eq:tmp_1}
\end{equation}
in which $p=\frac{1}{4e\sqrt{\pi}}$. 
\end{lemma}
As shown in the proof of Lemma 8 of~\cite{chowdhury2017kernelized}, the proof of Lemma~\ref{lemma:uniform_lower_bound} makes use of the fact that $f_t(x)\sim \mathcal{N}(\mu_{t-1}(x), \beta_t^2\sigma_{t-1}^2(x))$ since we are conditioning on the event $A_t$, the confidence bound given in Lemma~\ref{lemma:uniform_bound} which holds since we are conditioning on the event $E^{f}(t)$, and the Gaussian anti-concentration lemma. That is, for a Gaussian random variable $X\sim \mathcal{N}(\mu, \sigma^2)$, for any $\beta > 0$, we have that 
\[
\mathbb{P}\left(\frac{X-\mu}{\sigma} > \beta\right) \geq \frac{\exp(-\beta^2)}{4\sqrt{\pi}\beta}.
\]

The next lemma shows that in each iteration $t$, the probability that an unsaturated input is selected can be lower-bounded.
\begin{lemma}
\label{lemma:prob_unsaturated}
For any filtration $\mathcal{F}_{t-1}$, conditioned on the event $E^f(t)$, we have that with probability $\geq 1 - \delta/2$,
\[
\mathbb{P}\left(x_t \in \mathcal{X}\setminus S_t | \mathcal{F}_{t-1} \right) \geq P_t,
\]
in which 
\[
P_t \triangleq p_t (p - 1/t^2).
\]
\end{lemma}
\begin{proof}
Note that all probabilities in this proof are conditioned on the event $E^{f}(t)$ and thus this conditioning is omitted for simplicity.
At iteration $t$, the probability that the selected input $x_t$ is unsaturated can be lower-bounded by:
\begin{equation}
\begin{split}
\mathbb{P}&\left(x_t \in \mathcal{X}\setminus S_t | \mathcal{F}_{t-1} \right) \geq \mathbb{P}\left(x_t \in \mathcal{X}\setminus S_t | \mathcal{F}_{t-1},A_t \right)\mathbb{P}(A_t)=\mathbb{P}\left(x_t \in \mathcal{X}\setminus S_t | \mathcal{F}_{t-1},A_t \right)p_t
\end{split}
\label{eq:unsaturated_prob_over_all_eq}
\end{equation}
Next, we attempt to lower-bound $\mathbb{P}\left(x_t \in \mathcal{X}\setminus S_t | \mathcal{F}_{t-1},A_t \right)$.


Firstly, recall that conditioned on the event $A_t$, $x_t$ is selected by maximizing $f_t$, which is sampled from the GP posterior belief of function $f$. This gives rise to:
\begin{equation}
\mathbb{P}\left(x_t \in \mathcal{X}\setminus S_t | \mathcal{F}_{t-1},A_t \right) \geq \mathbb{P}\left( f_t(x^*) > f_t(x),\forall x \in S_t | \mathcal{F}_{t-1},A_t \right).
\label{eq:lower_bound_prob_unsaturated}
\end{equation}
This inequality can be obtained by observing that the event on the right hand side is a subset of the event on the left hand side. Specifically, recall from Definition~\ref{def:saturated_set} that $x^*$ is always unsaturated. Therefore, if $f_t(x^*) > f_t(x),\forall x \in S_t$, as a result of the way in which $x_t$ is selected (i.e., $x_t=\arg\max_{x\in\mathcal{X}}f_t(x)$), this guarantees that an unsaturated input will be selected as $x_t$ since at least one unsaturated input ($x^*$) has a larger value of $f_t$ than all saturated inputs.

Next, we assume that both events $E^f(t)$ and $E^{f_t}(t)$ are true, which allows us to derive an upper bound on $f_t(x)$ for all $x\in S_t$:
\begin{equation}
\begin{split}
    f_t(x) \stackrel{(a)}{\leq} f(x) + c_t\sigma_{t-1}(x) \stackrel{(b)}{\leq} f(x) + \Delta(x)=f(x) + f(x^*) - f(x) = f(x^*),
\end{split}
\label{eq:bound_ftx_ftstar}
\end{equation}
in which $(a)$ follows from~\eqref{eq:combine_two_events} since here we also assume both events $E^f(t)$ and $E^{f_t}(t)$ are true, 
and $(b)$ results from the definition of saturated set (Definition~\ref{def:saturated_set}). Therefore,~\eqref{eq:bound_ftx_ftstar} implies that 
\begin{equation}
    \mathbb{P}\left( f_t(x^*) > f_t(x),\forall x \in S_t | \mathcal{F}_{t-1},A_t, E^{f_t}(t) \right) \geq \mathbb{P}\left( f_t(x^*) > f(x^*) | \mathcal{F}_{t-1},A_t, E^{f_t}(t) \right).
\end{equation}
Next, we can show that 
\begin{equation}
\begin{split}
    \mathbb{P}\left(x_t \in \mathcal{X}\setminus S_t | \mathcal{F}_{t-1},A_t \right) &\geq 
    \mathbb{P}\left( f_t(x^*) > f_t(x),\forall x \in S_t | \mathcal{F}_{t-1},A_t \right)\\
    &\stackrel{(a)}{\geq} \mathbb{P}\left( f_t(x^*) > f(x^*) | \mathcal{F}_{t-1},A_t \right) - \mathbb{P}\left(\overline{E^{f_t}(t)} | \mathcal{F}_{t-1}\right)\\
    &\stackrel{(b)}{\geq} p - 1 / t^2,
\end{split}
\label{eq:unsaturated_prob_plugin_1}
\end{equation}
in which $(a)$ follows from some simple probabilistic manipulations and the fact that the event $E^{f_t}(t)$ is $\mathcal{F}_{t-1}$-measurable and thus independent of the event $A_t$, $(b)$ results from Lemma~\ref{lemma:uniform_lower_bound} and the fact that the event $E^{f_t}(t)$ holds with probability of at least $1 - 1 / t^2$.
Combining this inequality with~\eqref{eq:unsaturated_prob_over_all_eq} completes the proof.

\end{proof}

The next lemma presents an upper bound on the expected instantaneous regret of the FTS algorithm.
\begin{lemma}
\label{lemma:upper_bound_expected_regret}
For any filtration $\mathcal{F}_{t-1}$, conditioned on the event $E^{f}(t)$, we have that with probability of $\geq 1 - \delta/2$
\[
\mathbb{E}[r_t |\mathcal{F}_{t-1}] \leq c_t\left(1 + \frac{10}{p p_1}\right)\mathbb{E}\left[\sigma_{t-1}(x_t)| \mathcal{F}_{t-1} \right] + \psi_t + \frac{2B}{t^2},
\]
in which $r_t$ is the instantaneous regret: $r_t=f(x^*) - f(x_t)$, and $\psi_t \triangleq 2(1-p_t)\sum^N_{n=1}P_N[n] \Delta_{n,t}$.
\end{lemma}
\begin{proof}
To begin with, we define $\overline{x}_t$ as the unsaturated input at iteration $t$ with the smallest (posterior) standard deviation:
\begin{equation}
\overline{x}_t = {\arg\min}_{x\in\mathcal{X}\setminus S_t}\sigma_{t-1}(x).
\end{equation}
Following this definition, for any $\mathcal{F}_{t-1}$ such that $E^{f}(t)$ is true, we have that 
\begin{equation}
\begin{split}
\mathbb{E}\left[\sigma_{t-1}(x_t) | \mathcal{F}_{t-1}\right] &\geq \mathbb{E}\left[\sigma_{t-1}(x_t) | \mathcal{F}_{t-1}, x_t \in \mathcal{X} \setminus S_t\right]\mathbb{P}\left(x_t \in \mathcal{X} \setminus S_t | \mathcal{F}_{t-1}\right)\\
&\geq \sigma_{t-1}(\overline{x}_t)P_t,
\end{split}
\end{equation}
in which the last inequality follows from the definition of $\overline{x}_t$ and Lemma~\ref{lemma:prob_unsaturated}.

Now we condition on both events $E^{f}(t)$ and $E^{f_t}(t)$, and analyze the instantaneous regret as:
\begin{equation}
\begin{split}
r_t=\Delta(x_t)&=f(x^*) - f(\overline{x}_t) + f(\overline{x}_t) - f(x_t)\\
&\stackrel{(a)}{\leq} \Delta(\overline{x}_t) + f_t(\overline{x}_t) + c_t\sigma_{t-1}(\overline{x}_t) - f_t(x_t) + c_t\sigma_{t-1}(x_t)\\
&\stackrel{(b)}{\leq} c_t\sigma_{t-1}(\overline{x}_t) + c_t\sigma_{t-1}(\overline{x}_t) + c_t\sigma_{t-1}(x_t) + f_t(\overline{x}_t) - f_t(x_t)\\
&=c_t(2\sigma_{t-1}(\overline{x}_t) + \sigma_{t-1}(x_t)) + \underline{f_t(\overline{x}_t) - f_t(x_t)},
\end{split}
\end{equation}
in which $(a)$ follows from the definition of $\Delta(x)$ and $|f_t(x)-f(x)|\leq c_t\sigma_{t-1}(x)$ for all $x\in \mathcal{X}$ since we assume both events $E^{f}(t)$ and $E^{f_t}(t)$ are true, and $(b)$ results from the fact that $\overline{x}_t$ is unsaturated.
Next, we analyze the expected value of the underlined term given a filtration $\mathcal{F}_{t-1}$:
\begin{equation}
\begin{split}
\mathbb{E}&\left[f_t(\overline{x}_t) - f_t(x_t) | \mathcal{F}_{t-1}\right] \\
&= \mathbb{P}\left(A_t\right)\mathbb{E}\left[f_t(\overline{x}_t) - f_t(x_t) | \mathcal{F}_{t-1}, A_t\right] + \mathbb{P}\left(B_t\right)\sum^N_{n=1}P_N[n]\mathbb{E}\left[f_t(\overline{x}_t) - f_t(x_t) | \mathcal{F}_{t-1}, B_{t,n}\right]\\
&\stackrel{(a)}{\leq} (1-p_t)\sum^N_{n=1}P_N[n]\mathbb{E}\left[f_t(\overline{x}_t) - f_t(x_t) | \mathcal{F}_{t-1}, B_{t,n}\right]\\
&\stackrel{(b)}{\leq} (1-p_t)\sum^N_{n=1}P_N[n]\mathbb{E}\left[\hat{g}_{n,t}(\overline{x}_t)+\Delta_{n,t}-\hat{g}_{n,t}(x_t) + \Delta_{n,t} | \mathcal{F}_{t-1}, B_{t,n}\right]\\
&\stackrel{(c)}{\leq} 2(1-p_t)\sum^N_{n=1}P_N[n] \Delta_{n,t} \triangleq \psi_t,
\end{split}
\end{equation}
in which $(a)$ follows since when $A_t$ is true, i.e., when $x_t=\arg\max_{x\in\mathcal{X}}f_t(x)$, $f_t(\overline{x}_t) - f_t(x_t) \leq 0$, $(b)$ makes use of Lemma~\ref{lemma:bound_gt_ft} (note that here we are also conditioning on the events $E^{f}(t)$ and $E^{f_t}(t)$ which is the same as Lemma~\ref{lemma:bound_gt_ft}, and that Lemma~\ref{lemma:bound_gt_ft} holds irrespective of the event $B_{t,n}$ since both $E^{f}_t$ and $E^{f_t}(t)$ are $\mathcal{F}_{t-1}$-measurable) and thus holds with probability of $\geq 1 - \delta/2$, and $(c)$ follows since conditioned on the event $B_{t,n}$ (i.e., $x_t=\arg\max_{x\in \mathcal{X}}\hat{g}_{n,t}(x)$), $\hat{g}_{n,t}(\overline{x}_t) - \hat{g}_{n,t}(x_t) \leq 0$.

Subsequently, we can analyze the expected instantaneous regret by separately considering the two cases in which the event $E^{f_t}(t)$ is true and false respectively:
\begin{equation}
\begin{split}
\mathbb{E}&\left[r_t | \mathcal{F}_{t-1}\right] \\
&\leq \mathbb{E}\left[c_t(2\sigma_{t-1}(\overline{x}_t) + \sigma_{t-1}(x_t)) + f_t(\overline{x}_t) - f_t(x_t) | \mathcal{F}_{t-1}\right] + 2B\mathbb{P}\left[\overline{E^{f_t}(t)} | \mathcal{F}_{t-1} \right]\\
&\leq \mathbb{E}\left[c_t(2\sigma_{t-1}(\overline{x}_t) + \sigma_{t-1}(x_t)) | \mathcal{F}_{t-1}\right] + \mathbb{E}\left[f_t(\overline{x}_t) - f_t(x_t) | \mathcal{F}_{t-1}\right] + 2B\mathbb{P}\left[\overline{E^{f_t}(t)} | \mathcal{F}_{t-1} \right]\\
& \leq \frac{2c_t}{P_t}\mathbb{E}\left[\sigma_{t-1}(x_t)| \mathcal{F}_{t-1} \right] + c_t\mathbb{E}\left[\sigma_{t-1}(x_t) | \mathcal{F}_{t-1}\right] + \psi_t + \frac{2B}{t^2}\\
&\leq c_t\left(1 + \frac{2}{P_t}\right)\mathbb{E}\left[\sigma_{t-1}(x_t)| \mathcal{F}_{t-1} \right] + \psi_t + \frac{2B}{t^2}.
\label{eq:expected_inst_regret}
\end{split}
\end{equation}
Note that since $1/(p-1/t^2) \leq 5/p$ and $p_t\geq p_1$ for all $t\geq 1$,
\begin{equation}
\frac{2}{P_t} = \frac{2}{p_t(p - \frac{1}{t^2})}\leq \frac{10}{p p_t}\leq \frac{10}{pp_1}.
\end{equation}

Therefore,~\eqref{eq:expected_inst_regret} can be further analyzed as 
\begin{equation}
\mathbb{E} \left[r_t | \mathcal{F}_{t-1}\right] \leq c_t\left(1 + \frac{10}{p p_1}\right)\mathbb{E}\left[\sigma_{t-1}(x_t)| \mathcal{F}_{t-1} \right] + \psi_t + \frac{2B}{t^2},
\end{equation}
which completes the proof.
\end{proof}

Subsequently, we make use of the concentration inequality of super-martingales to derive a bound on the cumulative regret.
\begin{definition}
Define $Y_0=0$, and for all $t=1,\ldots,T$,
\[
\overline{r}_t=r_t \mathbb{I}\{E^{f}(t)\},
\]
\[
X_t = \overline{r}_t - c_t\left(1 + \frac{10}{p p_1}\right)\sigma_{t-1}(x_t) - \psi_t - \frac{2B}{t^2},
\]
\[
Y_t=\sum^t_{s=1}X_s.
\]
\end{definition}

\begin{lemma}
Conditioned on Lemma~\ref{lemma:upper_bound_expected_regret} (i.e., with probability of $\geq 1 - \delta/2$), $(Y_t:t=0,\ldots,T)$ is a super-martingale with respect to the filtration $\mathcal{F}_t$.
\end{lemma}
\begin{proof}
\begin{equation}
\begin{split}
\mathbb{E}\left[Y_t - Y_{t-1} | \mathcal{F}_{t-1}\right] &= \mathbb{E}\left[X_t | \mathcal{F}_{t-1}\right]\\
&=\mathbb{E}\left[\overline{r}_t - c_t\left(1 + \frac{10}{p p_1}\right)\sigma_{t-1}(x_t) - \psi_t - \frac{2B}{t^2} | \mathcal{F}_{t-1}\right]\\
&=\mathbb{E}\left[\overline{r}_t | \mathcal{F}_{t-1}\right] - \left[c_t\left(1 + \frac{10}{p p_1}\right)\mathbb{E}\left[\sigma_{t-1}(x_t) | \mathcal{F}_{t-1}\right] + \psi_t + \frac{2B}{t^2}
\right]\\
&\leq 0,
\end{split}
\end{equation}
in which the last inequality follows from Lemma~\ref{lemma:upper_bound_expected_regret} when the event $E^{f}(t)$ is true; when $E^{f}(t)$ is false, $\overline{r}_t=0$ and thus the inequality holds trivially.
\end{proof}

The Azuma-Hoeffding Inequality presented below will be useful for proving the concentration of the super-martingale $(Y_t:t=0,\ldots,T)$.
\begin{lemma}[Azuma-Hoeffding Inequality]
\label{lemma:azuma_hoeffding}
Given any $\delta' \in (0, 1)$. If a super-martingale $(Z_T: t=1,\ldots,T)$, defined with respect to the filtration $\mathcal{F}_t$, satisfies $|Z_t-Z_{t-1}|\leq \alpha_t$ for some constant $\alpha_t$, then for all $t=1,\ldots,T$ and with probability of at least $1 - \delta'$,
\[
Z_T - Z_0 \leq \sqrt{2\log(1/\delta')\sum^T_{t=1}\alpha_t^2}.
\]
\end{lemma}

Finally, we can derive an upper bound on the cumulative regret through the following lemma.
\begin{lemma}
Given $\delta \in (0, 1)$, then with probability of at least $1 - \delta$,
\begin{equation*}
\begin{split}
R_T\leq &c_T\left(1 + \frac{10}{p p_1}\right) \mathcal{O}(\sqrt{T\gamma_T}) + \sum^T_{t=1}\psi_t + \frac{B\pi^2}{3} +\\
&\left[c_T\left(1 + \frac{4B}{p} + \frac{10}{p p_1}
\right) + \psi_1 + \mathcal{O}(\sqrt{\log T}) \right]\sqrt{2T\log\frac{4}{\delta}},
\end{split}
\end{equation*}
in which 
$\gamma_T$ is the maximum information gain about $f$ obtained from any set of $T$ observations.
\end{lemma}
\begin{proof}
\begin{equation}
\begin{split}
|Y_t - Y_{t-1}| &= |X_t| \leq |\overline{r}_t| + c_t\left(1 + \frac{10}{p p_1}\right)\sigma_{t-1}(x_t) + \psi_t + \frac{2B}{t^2}\\
&\stackrel{(a)}{\leq} 2B + c_t\left(1 + \frac{10}{p p_1}\right) + \psi_t + \frac{2B}{t^2}\\
&\stackrel{(b)}{\leq} \frac{2B c_t}{p} + c_t\left(1 + \frac{10}{p p_1}\right) + \psi_t + \frac{2Bc_t}{p}\\
&\leq c_t\left(1 + \frac{4B}{p} + \frac{10}{p p_1}
\right) + \psi_t,
\end{split}
\end{equation}
in which $(a)$ follows since the posterior variance is upper-bounded by $1$,
$(b)$ follows since $2B\leq 2Bc_t/p$ and $2B/t^2\leq 2Bc_t/p$.

This allows us to apply the Azuma-Hoeffding Inequality (Lemma~\ref{lemma:azuma_hoeffding}) by using an error probability of $\delta / 4$,
\begin{equation}
\begin{split}
\sum^T_{t=1}\overline{r}_t &\leq \sum^T_{t=1} c_t\left(1 + \frac{10}{p p_1}\right)\sigma_{t-1}(x_t) + \sum^T_{t=1}\psi_t + \sum^T_{t=1}\frac{2B}{t^2} +\\
&\qquad \sqrt{2\log\frac{4}{\delta}\sum^{T}_{t=1}\left[c_t\left(1 + \frac{4B}{p} + \frac{10}{p p_1}
\right) + \psi_t\right]^2}\\
&\leq c_T\left(1 + \frac{10}{p p_1}\right) \sum^T_{t=1}\sigma_{t-1}(x_t) + \sum^T_{t=1}\psi_t + \frac{B\pi^2}{3} + \\
&\qquad \left[c_T\left(1 + \frac{4B}{p} + \frac{10}{p p_1}
\right) + \psi_1 + \mathcal{O}(\sqrt{\log T})\right]\sqrt{2T\log\frac{4}{\delta}}\\
&= c_T\left(1 + \frac{10}{p p_1}\right) \mathcal{O}(\sqrt{T\gamma_T}) + \sum^T_{t=1}\psi_t + \frac{B\pi^2}{3} + \\
&\qquad \left[c_T\left(1 + \frac{4B}{p} + \frac{10}{p p_1}
\right) + \psi_1 + \mathcal{O}(\sqrt{\log T})\right]\sqrt{2T\log\frac{4}{\delta}},
\end{split}
\end{equation}
which holds with probability $\geq 1 - \delta/4$.
The last inequality follows since $c_t$ is increasing in $t$, $\sum^T_{t=1}1/t^2=\pi^2/6$, and $\psi_t\leq \psi_{1} + \mathcal{O}(\sqrt{\log t})$ for all $t\in\mathbb{Z}^+$ which is ensured by the way in which we choose the sequence $p_t$, i.e., such that $(1-p_t)c_t \leq (1-p_1)c_1$ for all $t\in\mathbb{Z}^+\setminus\{1\}$.
Lastly, note that the event $E^{f}(t)$ holds with probability $\geq 1 - \delta/4$, i.e., $\overline{r}_t=r_t$ with probability $\geq 1 - \delta/4$.
In the last equality, we made use of the fact that $\sum^T_{t=1}\sigma_{t-1}(x_t)=\mathcal{O}(\sqrt{T\gamma_T})$ 
which is proved by Lemmas 5.3 and 5.4 of~\cite{srinivas2009gaussian}.
Taking into account the error probability of Lemma~\ref{lemma:upper_bound_expected_regret} ($\delta/2$), which is required for $(Y_t:t=0,\ldots,T)$ to form a super-martingale, completes the proof.
\end{proof}

Finally, we are ready to prove Theorem~\ref{theorem:fts}. Recall that $c_t = \mathcal{O}\left(\left(B + \sqrt{\gamma_t+\log(1/\delta)}\right)\sqrt{\log t}\right)$. Therefore, 
\begin{equation}
\begin{split}
R_T &= \mathcal{O}\Bigg(\frac{1}{p_1}\left(B + \sqrt{\gamma_T + \log\frac{1}{\delta}}\right)\sqrt{\log T}\sqrt{T\gamma_T} + \sum^T_{t=1}\psi_t + \\
&\qquad \left(B + \frac{1}{p_1}\right)\left(B + \sqrt{\gamma_T+\log\frac{1}{\delta}}\right)\sqrt{\log T}\sqrt{T\log\frac{1}{\delta}}\Bigg)\\
&=\mathcal{O}\left(\left(B + \frac{1}{p_1}\right)\sqrt{T\log T\gamma_T \log\frac{1}{\delta}\left(\gamma_T + \log\frac{1}{\delta}\right)} + \sum^T_{t=1}\psi_t\right)\\
&=\tilde{\mathcal{O}}\left(\left(B+\frac{1}{p_1}\right)\gamma_T\sqrt{T} + \sum^T_{t=1}\psi_t\right).
\end{split}
\end{equation}

\section{Further Experimental Details and Results}
\label{app:experiments}
All experiments reported in this work are run on a computer with 48 cores of Xeon Silver 4116 (2.1Ghz) processors, RAM of 256GB, and 1 NVIDIA Tesla T4 GPU.
For fair comparisons, in all experiments, the same random initializations are used by all methods.

\subsection{Optimization of Synthetic Functions}
In the synthetic experiments, we sample the objective functions from a GP with a length scale of $0.03$.
The functions are defined on a $1$-dimensional discrete domain uniformly distributed in $[0, 1]$, with size $|\mathcal{X}|=1,000$.
The output values of all functions $f(x),\forall x\in\mathcal{X}$ are normalized into the range of $[0, 1]$. 
Whenever an input $x$ is queried, the corresponding noisy observation is obtained by adding a zero-mean Gaussian noise $\mathcal{N}(0, \sigma^2)$ where $\sigma^2=0.01$ to the corresponding function value $f(x)$ (Section~\ref{sec:background}, first paragraph).
For a sampled objective function for the target agent, we generate the objective functions of the other agents, as well as their observations, in the following way.
For agent $\mathcal{A}_n$, we go through every input in the entire discrete domain, and for each input, we add either $d_n$ or $-d_n$ to the corresponding output function value with probability $0.5$ each, 
after which the resulting value is used as the objective function value of the agent $\mathcal{A}_n$.
This step ensures the validity of the definition of $d_n$ as the maximum difference between the objective function of the target agent $\mathcal{A}$ and that of agent $\mathcal{A}_n$: $d_n=\max_{x\in \mathcal{X}}|f(x) - g_n(x)|$ (Section~\ref{sec:background}, last paragraph).
Next, we randomly sample $t_n$ inputs from the entire discrete domain, and for each sampled input, we obtain a noisy output observation 
by adding a zero-mean Gaussian noise: $\mathcal{N}(0, \sigma^2)$ where $\sigma^2=0.01$, to the corresponding function value.
Subsequently, following the procedures described in the first paragraph of Section~\ref{subsec:fts_algo},
every agent $\mathcal{A}_n$ applies RFF approximation using its own $t_n$ observations (input-output pairs) to derive the RFF approximation parameters $\nu_n$ and $\Sigma_n$
and hence to draw a sample of $\omega_n$, which is the parameter to be passed to and used by the target agent $\mathcal{A}$.
Finally, after receiving the parameters $\omega_n$'s from all other agents, the target agent starts to run the FTS algorithm (Algorithm~\ref{alg:F_TS}).

\subsection{Real-world Experiments}

\subsubsection{Results in the Most General Setting with Increasing $t_n$}
\label{app:results_increasing_t_n}
Here we perform additional experiments in the most general setting of our FTS algorithm (Section~\ref{subsec:fts_algo}, last paragraph): 
(a) information can be received from every agent $\mathcal{A}_n$ before every iteration instead of only before the first iteration, and 
(b) every $\mathcal{A}_n$ may also be performing black-box optimization tasks (possibly also using FTS),
such that $\mathcal{A}_n$ may collect more observations (i.e., $t_n$ may increase) between different rounds of communication.
We use the three real-world experiments (Section~\ref{subsec:exp_real_world}) to investigate the performances in this setting, 
and compare the performances with those in the simpler setting where communication is allowed only before the first iteration.
\begin{figure}[t]
	\centering
	\begin{tabular}{ccc}
		\hspace{-3mm} \includegraphics[width=0.32\linewidth]{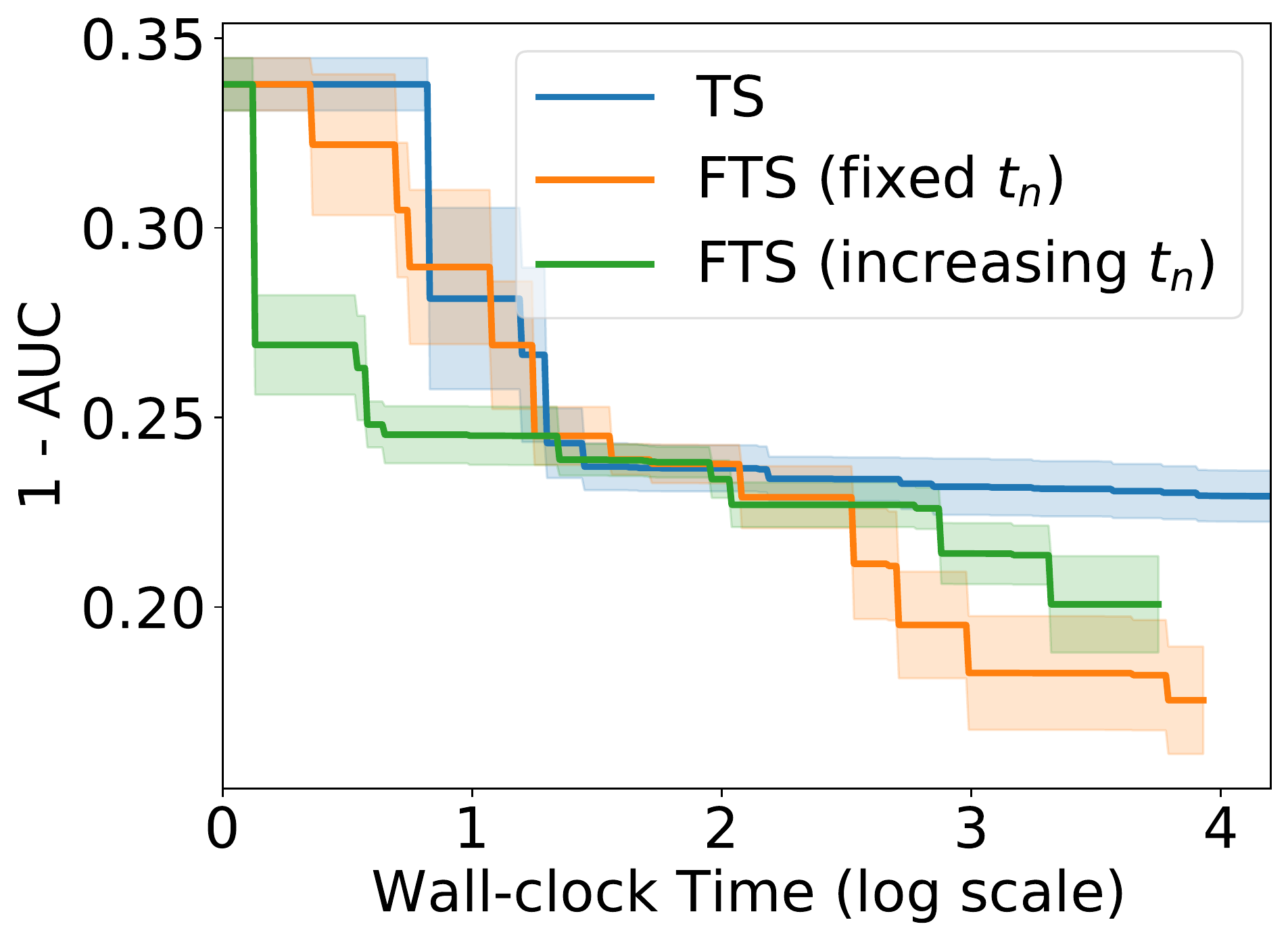} & \hspace{-4mm}
		\includegraphics[width=0.315\linewidth]{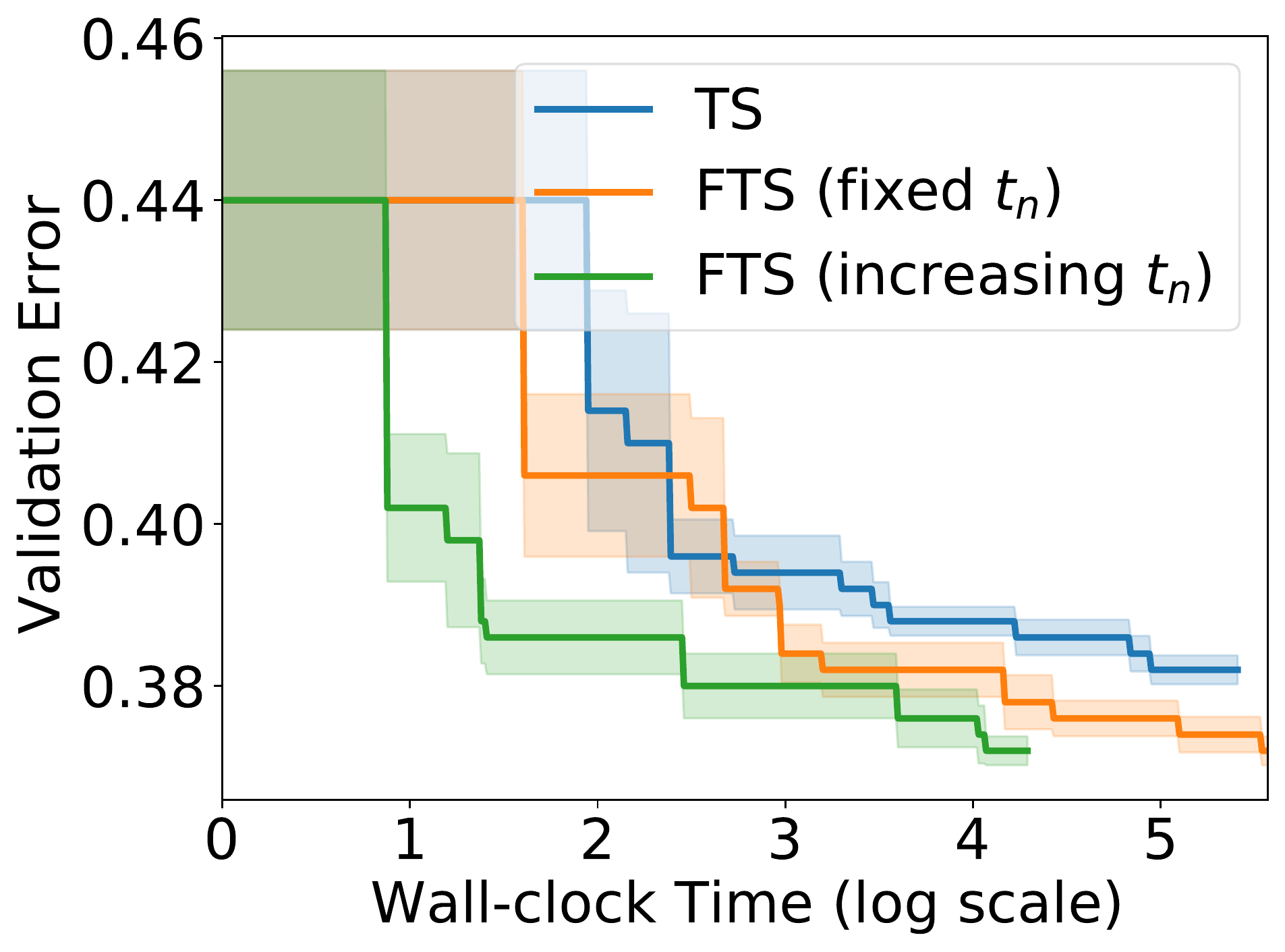} & \hspace{-4mm} 
		\includegraphics[width=0.32\linewidth]{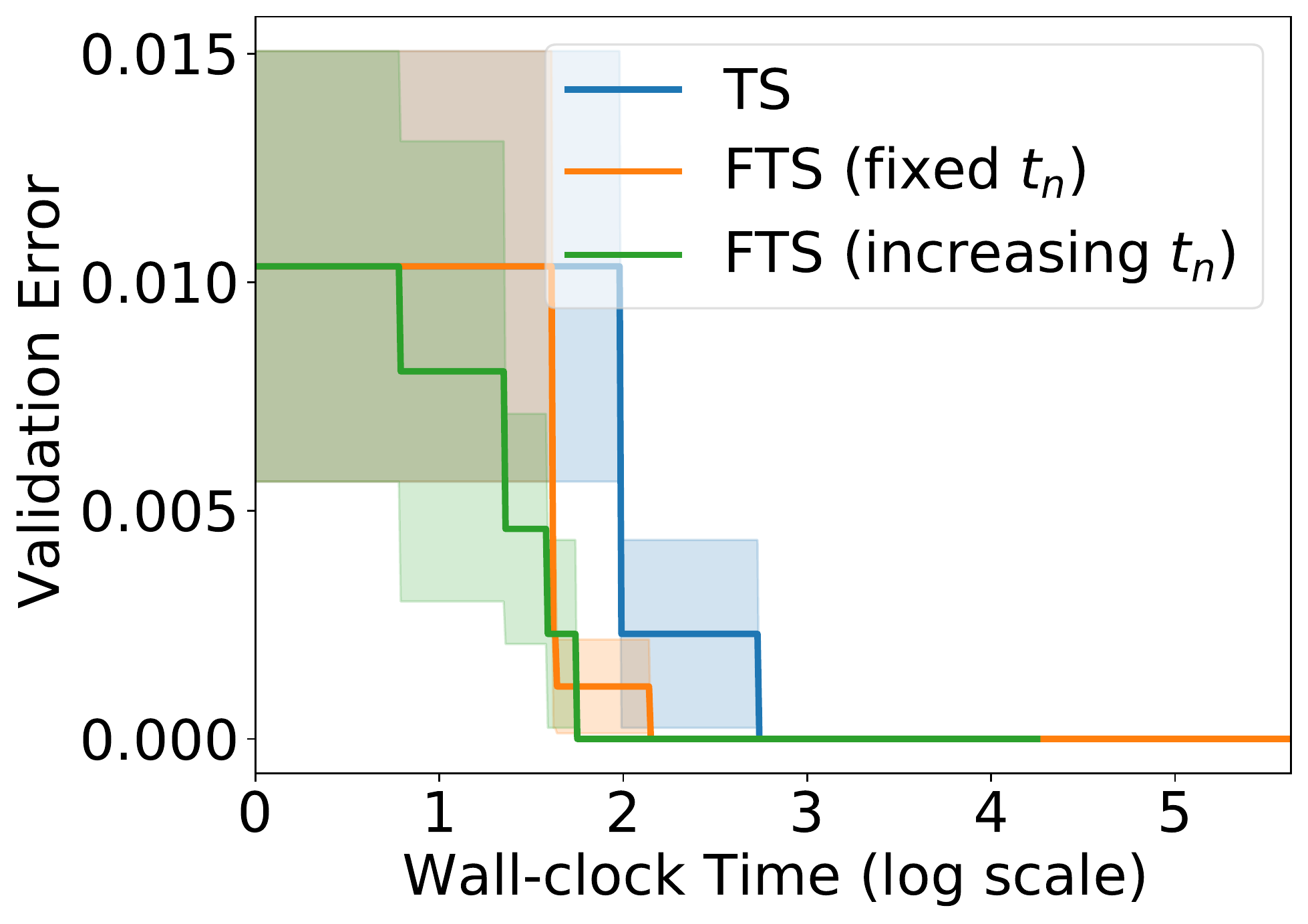}\\
		{(a)} & {(b)} & {(c)}
	\end{tabular}\vspace{-1mm}
	\caption{Performances in the most general setting in which $t_n$ is increasing (green curve) for the (a) landmine detection, (b) Google glasses and (c) mobile phone sensors experiments.
	The specific experimental setup is described in Appendix~\ref{app:results_increasing_t_n}.
	The results correspond to Fig.~\ref{fig:real_world_exp_time_plot} in the main paper.}
	\label{fig:increasing_tn}
\end{figure}

Here we describe the detailed experimental setup for the experiments in this section.
Before the first iteration of the FTS algorithm, every agent $\mathcal{A}_n$ for $n=1,\ldots,N$, who has completed $t_n=50$ iterations of its own BO task (we use standard TS here for simplicity, 
but it can be replaced by FTS in which $\mathcal{A}_n$ is the target agent), passes the first message to the target agent $\mathcal{A}$.
Next, before every iteration $t > 1$ of the FTS algorithm (Algorithm~\ref{alg:F_TS}), every agent $\mathcal{A}_n$ runs \emph{one more iteration} of its own BO task, 
calculates the updated RFF approximation parameters $\nu_{n,t}$ and $\Sigma_{n,t}$~\eqref{eq:sigma_nu}, 
samples a new $\omega_{n,t}$ from its posterior belief: $\omega_{n,t} \sim \mathcal{N}(\nu_{n,t}, \sigma^2\Sigma_{n,t}^{-1})$, 
and finally passes the sampled $\omega_{n,t}$ to the target agent $\mathcal{A}$.
Then, $\mathcal{A}$ uses the received updated information to run iteration $t$ of the FTS algorithm.
After this, the information from every agent $\mathcal{A}_n$ is updated and sent to $\mathcal{A}$ again, and the FTS algorithm proceeds to the next iteration $t+1$.
As a result, every $t_n,\forall n=1,\ldots,N$ increases by $1$ after every iteration of FTS.

The performances in all three experiments are shown in Fig.~\ref{fig:increasing_tn},
in which FTS outperforms standard TS in both settings for all experiments.
The figure also shows that in the most general setting in which $t_n$ is increasing, the performances for the two activity recognition experiments (Google glasses and mobile phone sensors experiments)
are improved, whereas the performances for the landmine detection experiment are comparable in both settings.
Note that the most general setting with increasing $t_n$ may not necessarily lead to better performances:
Although using more observations from those agents with similar objective functions to the target agent can give more useful information and hence 
potentially benefit the FTS algorithm, more observations from \emph{heterogeneous agents} may turn out to hurt the performance of FTS
since the information from these agents are actually harmful for the BO task of the target agent.

\subsubsection{More Experimental Details}
In all real-world experiments, we use $\text{length scale}=0.01$ to generate the random features (Appendix~\ref{app:construct_random_features}) 
and $\sigma^2=10^{-6}$ in the RFF approximation using equations~\eqref{eq:posterior_dist_omega} and~\eqref{eq:sigma_nu}.

\textbf{Landmine Detection.}
This dataset, downloadable from \url{http://www.ee.duke.edu/~lcarin/LandmineData.zip}, consists of the data from $29$ landmine fields, with each field associated with a dataset for landmine detection.
The dataset of each landmine field is made up of a number of input-output pairs, each corresponding to a location;
for every location, the input includes $9$ features extracted from radar images and the output is a binary label indicating whether the location contains landmines.
The number of data points (input-output pairs) of every field ranges from $445$ to $690$, with a mean of $511$;
for every field, we use $50\%$ of the data points as the training set, and the other $50\%$ as the validation set.
We use support vector machines (SVM) as the predictive model, and tune two SVM hyperparameters: RBF kernel parameter in the range of $[0.01,10]$, and L$2$ regularization parameter in $[10^{-4},10]$.
For every queried hyperparamter setting, the SVM model is trained on the training set using this particular set of hyperparameters, 
and evaluated using the validation set to produce the reported performances.
As mentioned in the main text, the dataset of the landmine fields are significantly imbalanced, i.e., there are considerably more locations without than with landmines.
Specifically, the percentage of positive samples (i.e., locations with landmines) in different landmine fields ranges from $2.9\%$ to $9.4\%$, with a mean of $6.2\%$.
Therefore, for this dataset, validation error is inappropriate since an all-zero prediction would result in very low classification error.
Hence, we use the Area Under the Receiver Operating Characteristic Curve (AUC) which is a more appropriate metric when evaluating the performance of ML models 
on imbalanced datasets.

\textbf{Activity Recognition Using Google Glasses.}
This dataset consists of two-hour long sensor data collected using Google glasses from $38$ participants, while the participants are performing different activities such as eating.
The dataset can be downloaded from \url{http://www.skleinberg.org/data/GLEAM.tar.gz}.
For every participant, we group the sensor data into different time windows; for each time window, we calculate the statistics (i.e., mean, variance and kurtosis) of different sensor measurements
within this time window, and use them as the features ($57$ features in total are extracted from each time window); 
the label for each time window is a binary value indicating whether the participant is eating or conducting other activities during this time window.
As a result, for every participant, each time window produces a data point, i.e., an input-output pair.
The number of data points for every participant ranges from $	242$ to $3416$ with an average of $1930$. 
For every participant, we randomly select $100$ data points as the validation set, and use the remaining data points as the training set.
We use logistic regression (LR) as the activity prediction model for every participant, and tune $3$ hyperparameters of LR: the batch size in the range of $[20, 60]$,
the L$2$ regularization parameter in $[10^{-6}, 1]$, and the learning rate in $[0.01, 0.1]$.
Following the common practice for using LR and neural network models, the inputs are pre-processed by removing the mean and dividing by the standard deviation.

\textbf{Activity Recognition Using Mobile Phone Sensors.}
This dataset, which can be downloaded from \url{https://archive.ics.uci.edu/ml/datasets/Human+Activity+Recognition+Using+Smartphones}, 
contains measurements from mobile phone sensors (accelerometer and gyroscope) involving $30$ subjects.
$561$ features were provided together with the dataset, with each set of features associated with a corresponding label indicating which one of the six activities the subject is performing.
Therefore, the activity recognition problem for every subject corresponds to a $6$-class classification problem.
The number of data points (input-output pairs) possessed by the subjects ranges from $281$ to $409$ with a mean of $343$.
For every subject, we use $50\%$ of the data points as the training set, and the remaining $50\%$ as the validation set.
We again use LR as the activity recognition model, and the tuned hyperparameters, as well as their ranges, are the same as those in the activity recognition experiment using Google glasses.

\subsubsection{Additional Results for More Agents}
\label{app:more_experimental_results}
In this section, we present additional experimental results for the three real-world experiments (Section~\ref{subsec:exp_real_world}).
Note that as mentioned in Section~\ref{subsec:exp_real_world} (last paragraph), the results presented in Fig.~\ref{fig:real_world_exp_time_plot} in the main text correspond to using the first agent 
(of the $6$ agents used to produce the results in Fig.~\ref{fig:real_world_exp} in the main text) as the target agent for every experiment.
Meanwhile, the additional results shown in this section (Figs.~\ref{fig:add_results_app_landmine},~\ref{fig:add_results_app_google} and~\ref{fig:add_results_app_har})
correspond to using each of the remaining $5$ agents (agents $2$ to $6$) as the target agent.
Note that since all three real-world datasets contain heterogeneous agents (Section~\ref{subsec:exp_real_world}, first paragraph), 
it is unreasonable to expect FTS to always outperform standard TS for all agents.
Instead, as shown in the figures, FTS performs better than TS for some agents, and comparably with TS for other agents.
\begin{figure}
    \centering
    \begin{subfigure}[t]{0.32\linewidth}
        \includegraphics[width=\linewidth]{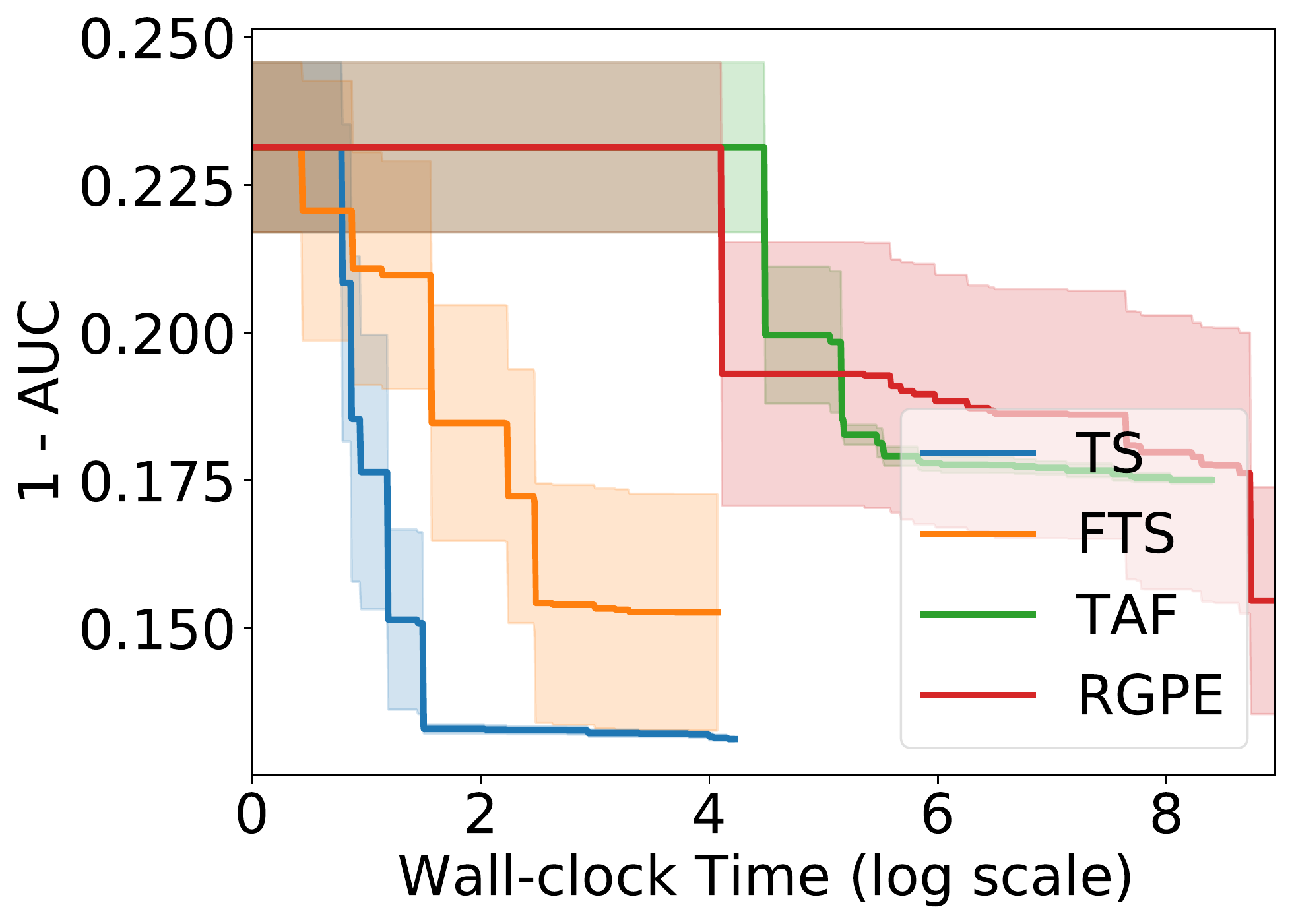}
        \caption{}
    \end{subfigure}
    \begin{subfigure}[t]{0.32\linewidth}
        \includegraphics[width=\linewidth]{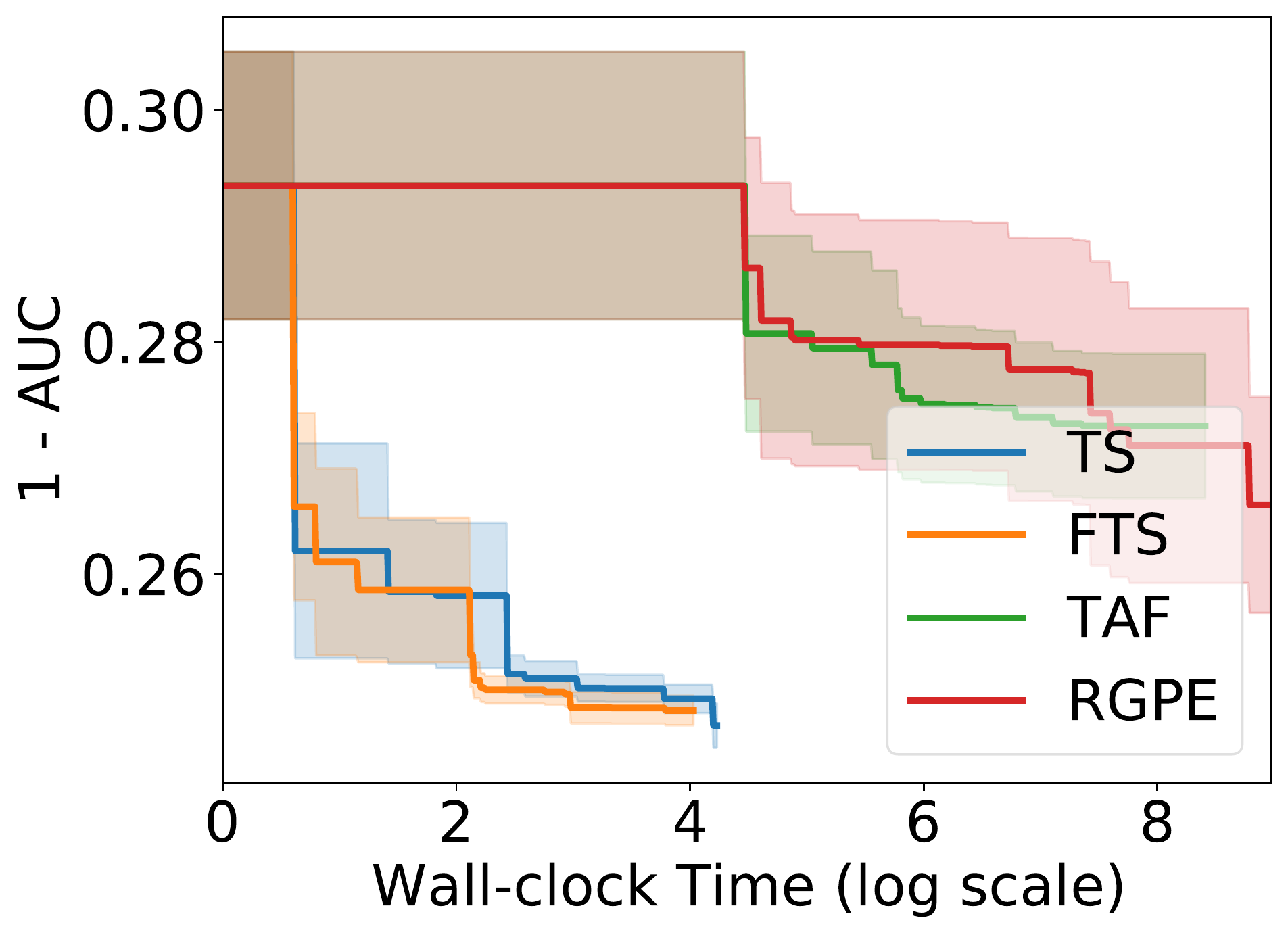}
        \caption{}
    \end{subfigure}
    \begin{subfigure}[t]{0.32\linewidth}
        \includegraphics[width=\linewidth]{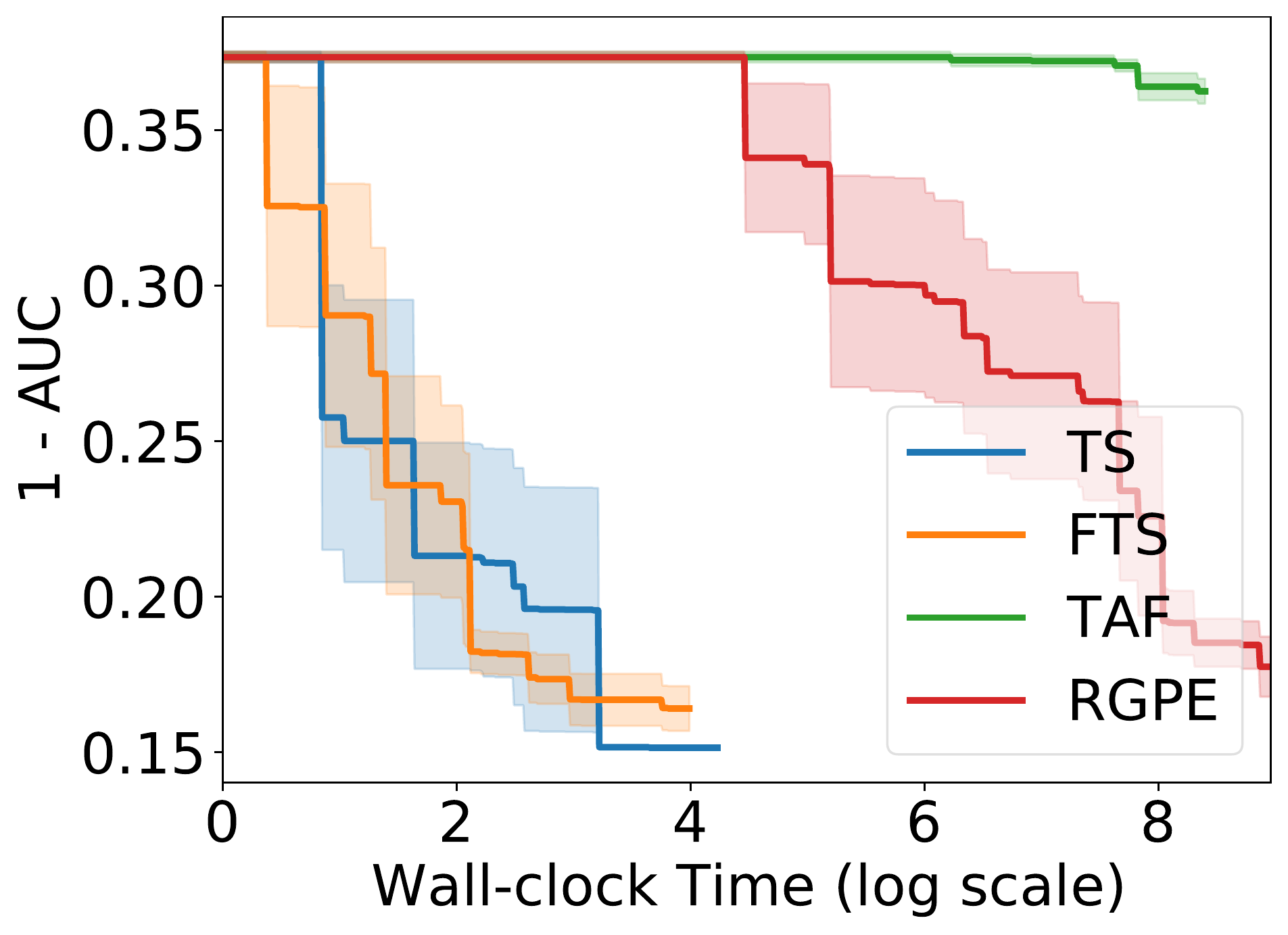}
        \caption{}
    \end{subfigure}
    \begin{subfigure}[t]{0.32\linewidth}
        \includegraphics[width=\linewidth]{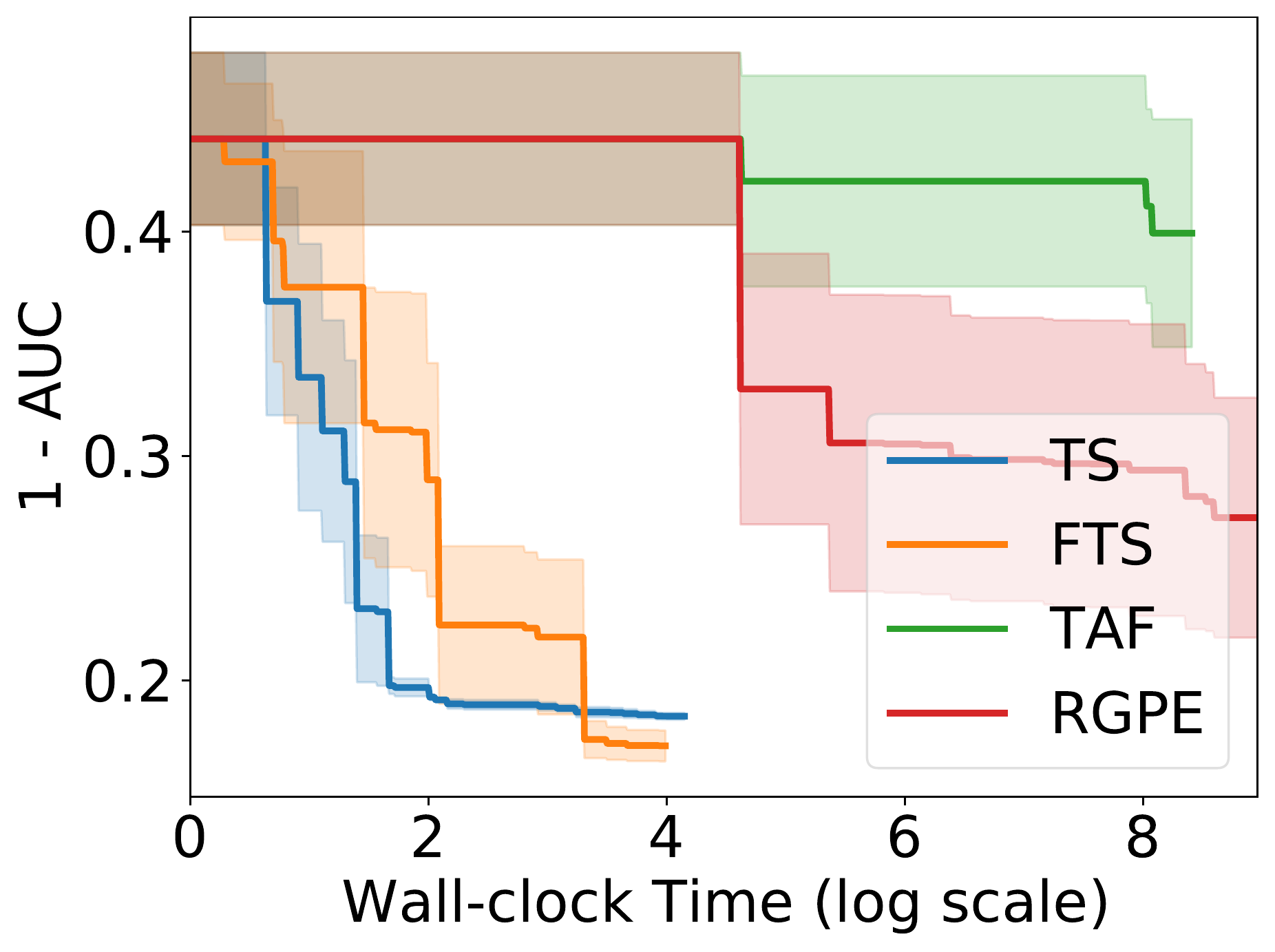}
        \caption{}
    \end{subfigure}
    \begin{subfigure}[t]{0.32\linewidth}
        \includegraphics[width=\linewidth]{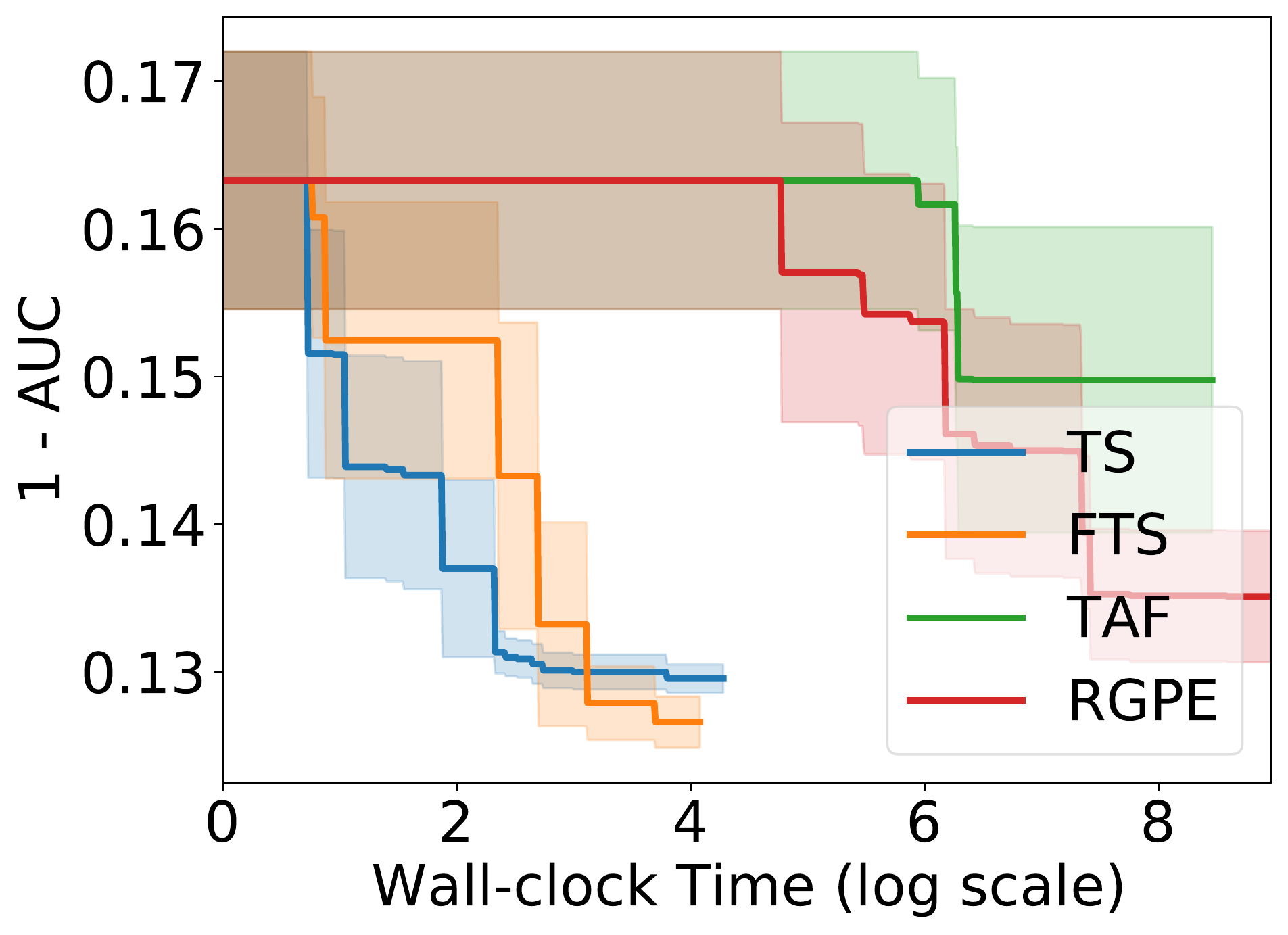}
        \caption{}
    \end{subfigure}
    \caption{Additional results for the other $5$ target agents for the landmine detection experiment ($M=100$).}
    \label{fig:add_results_app_landmine}
\end{figure}

\begin{figure}
    \centering
    \begin{subfigure}[t]{0.32\linewidth}
        \includegraphics[width=\linewidth]{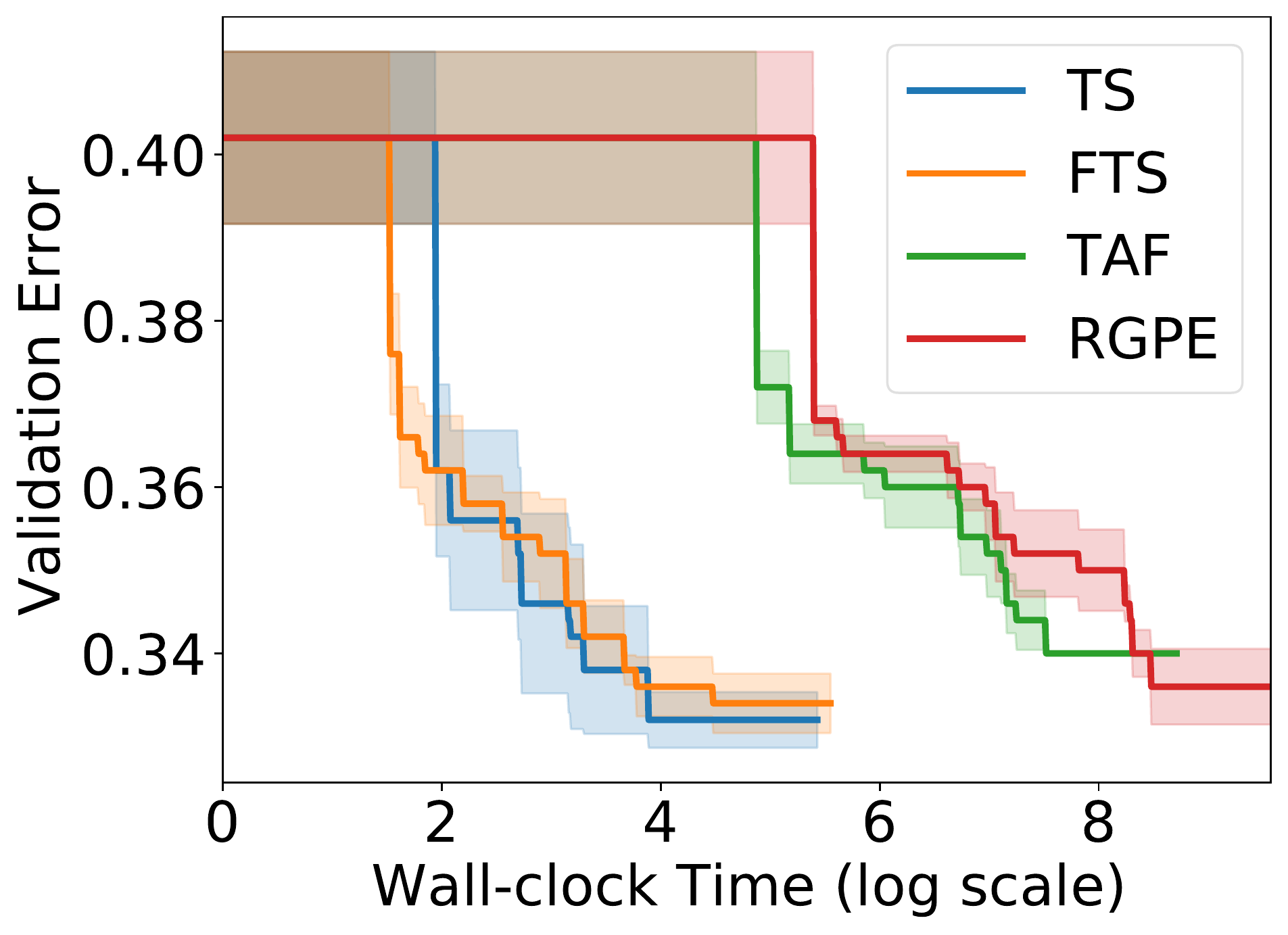}
        \caption{}
    \end{subfigure}
    \begin{subfigure}[t]{0.32\linewidth}
        \includegraphics[width=\linewidth]{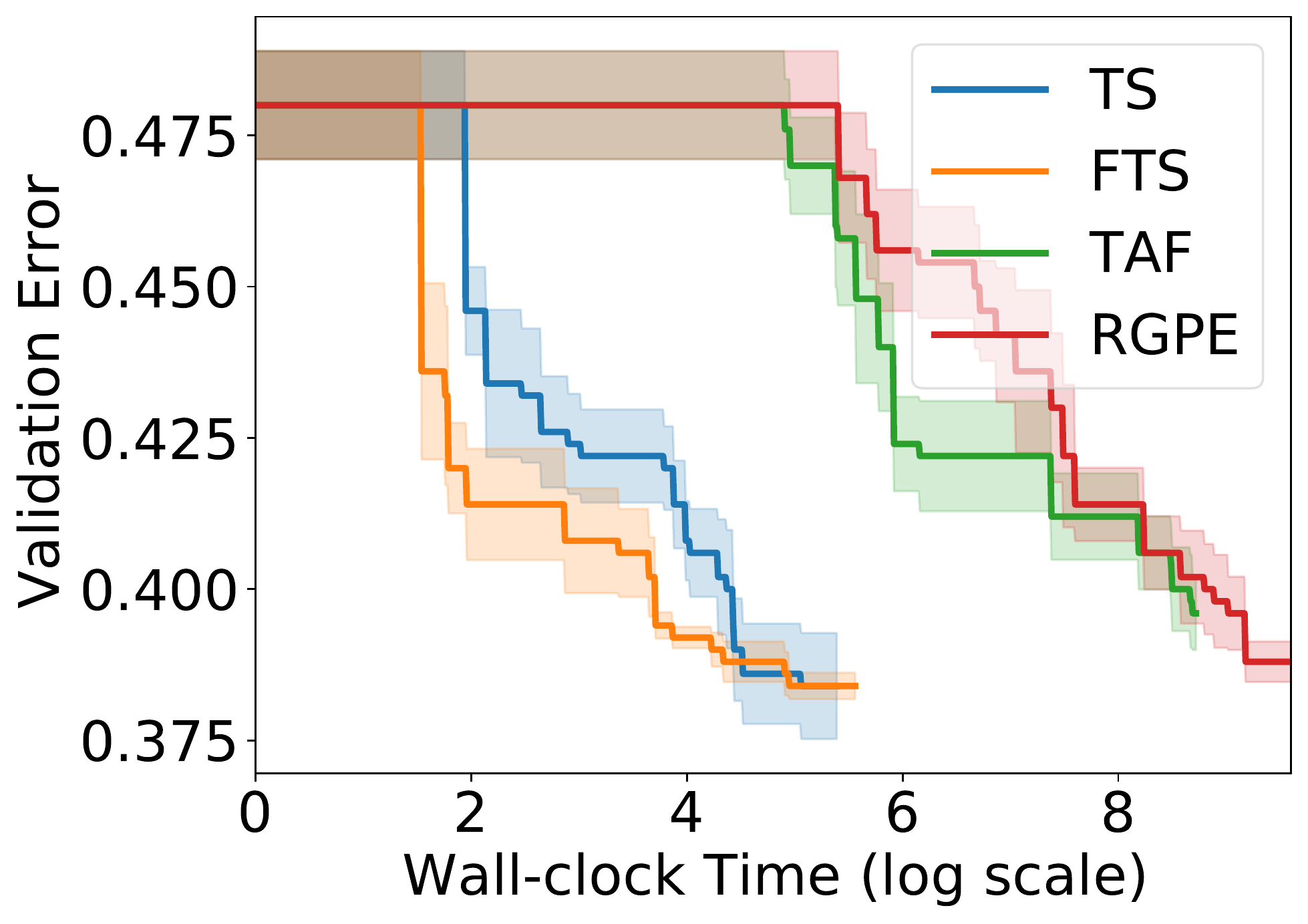}
        \caption{}
    \end{subfigure}
    \begin{subfigure}[t]{0.32\linewidth}
        \includegraphics[width=\linewidth]{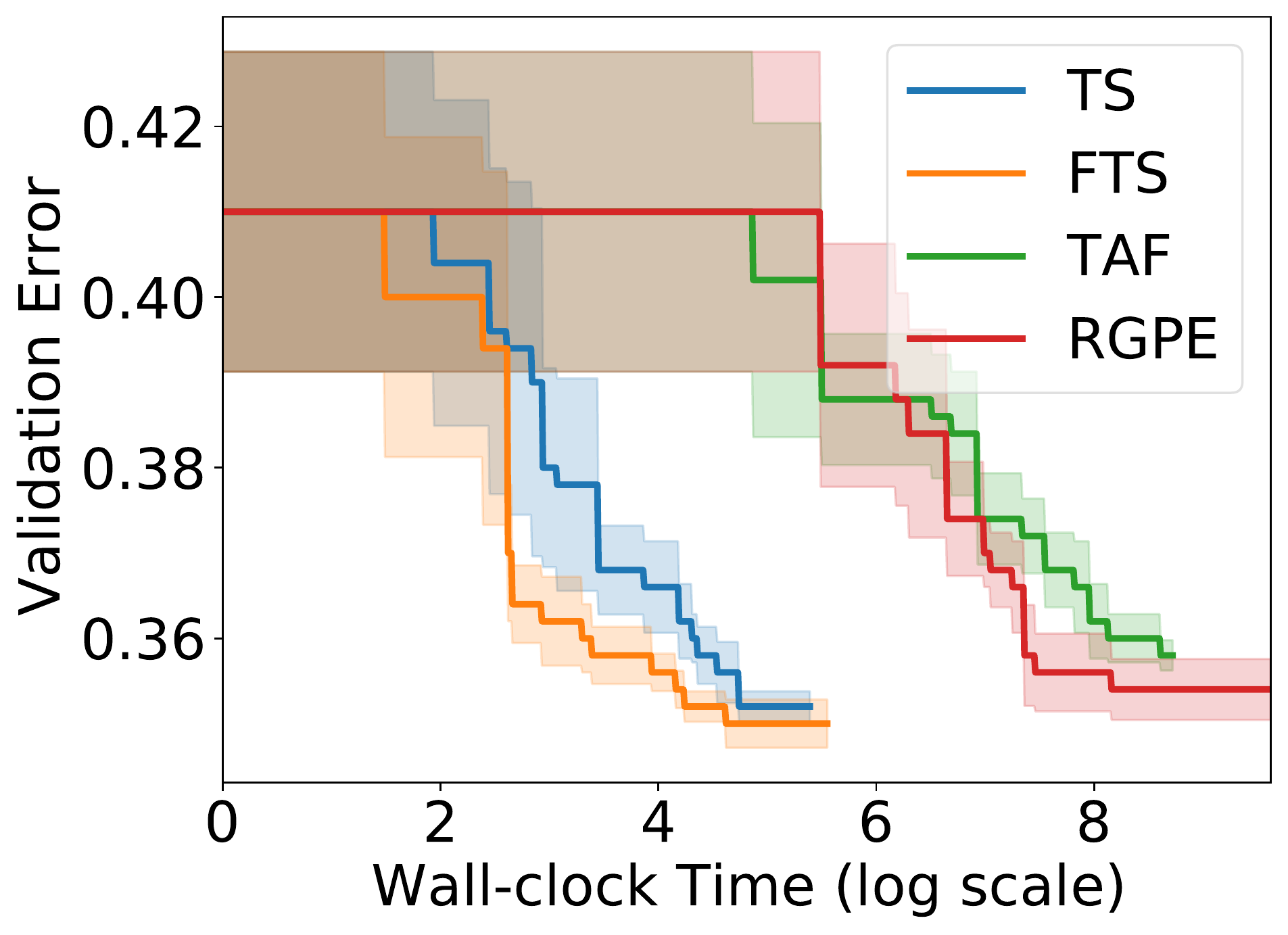}
        \caption{}
    \end{subfigure}
    \begin{subfigure}[t]{0.32\linewidth}
        \includegraphics[width=\linewidth]{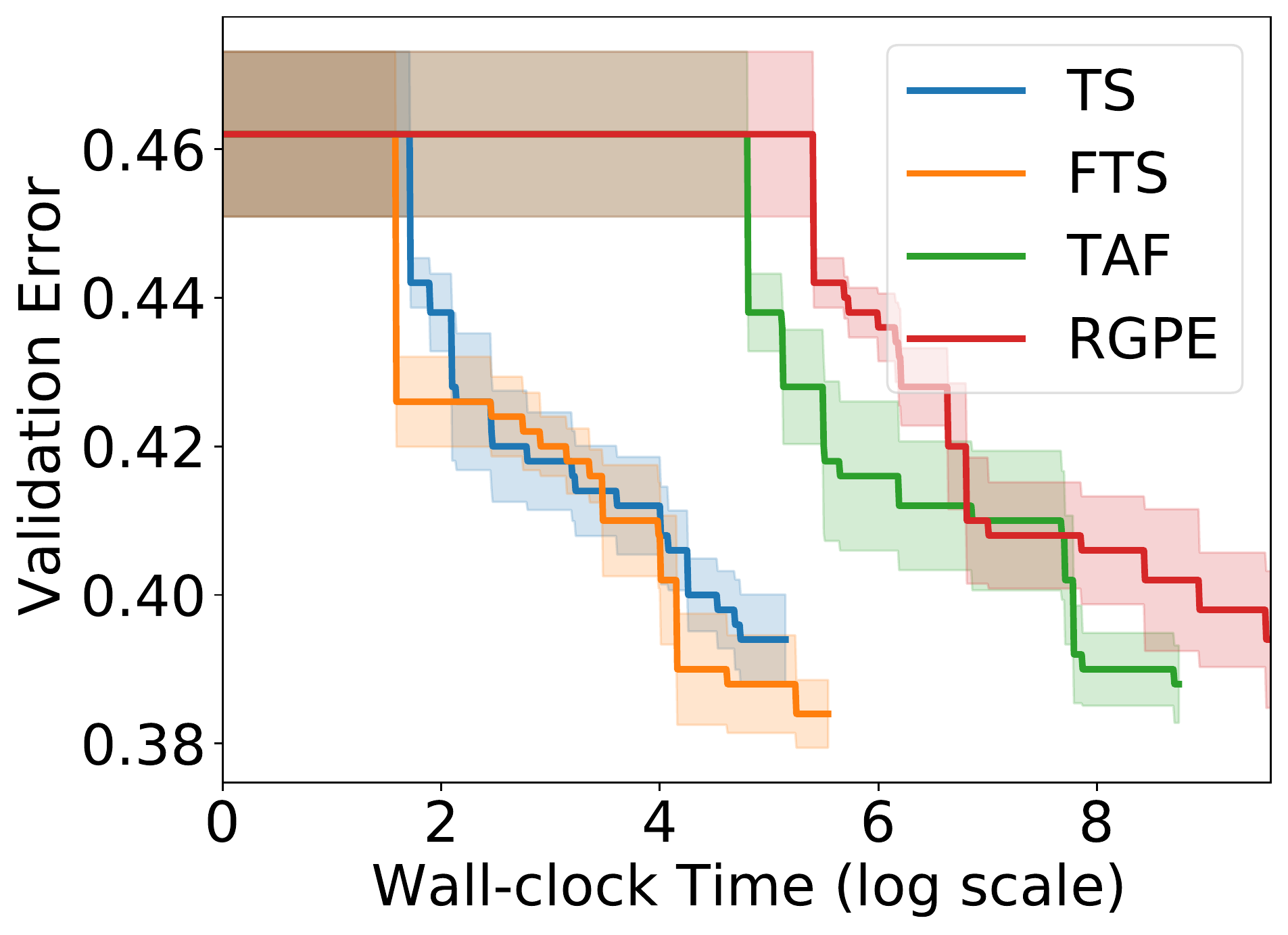}
        \caption{}
    \end{subfigure}
    \begin{subfigure}[t]{0.32\linewidth}
        \includegraphics[width=\linewidth]{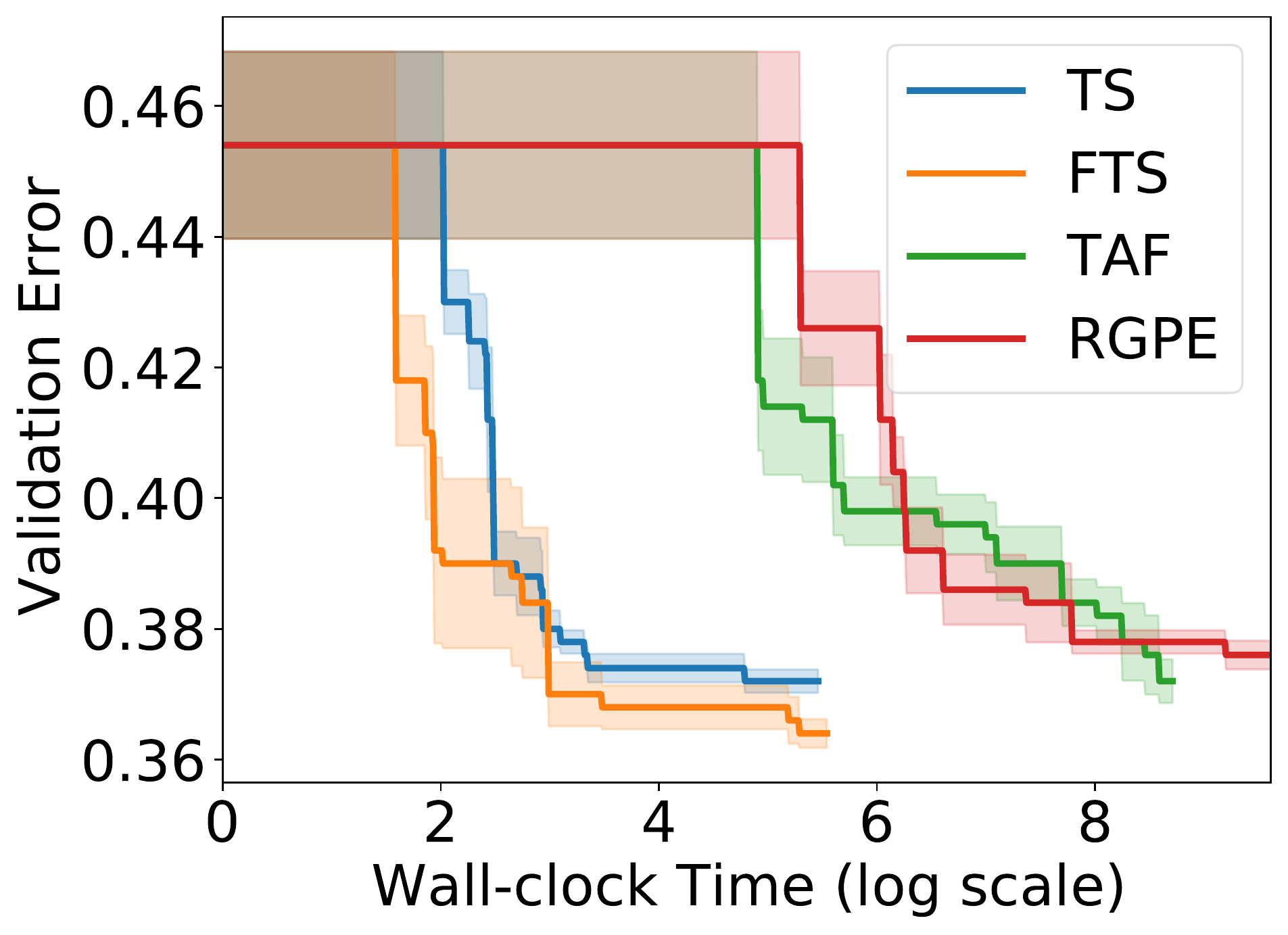}
        \caption{}
    \end{subfigure}
    \caption{Additional results for the other $5$ target agents for the activity recognition experiment using Google glasses ($M=100$).}
    \label{fig:add_results_app_google}
\end{figure}

\begin{figure}
    \centering
    \begin{subfigure}[t]{0.32\linewidth}
        \includegraphics[width=\linewidth]{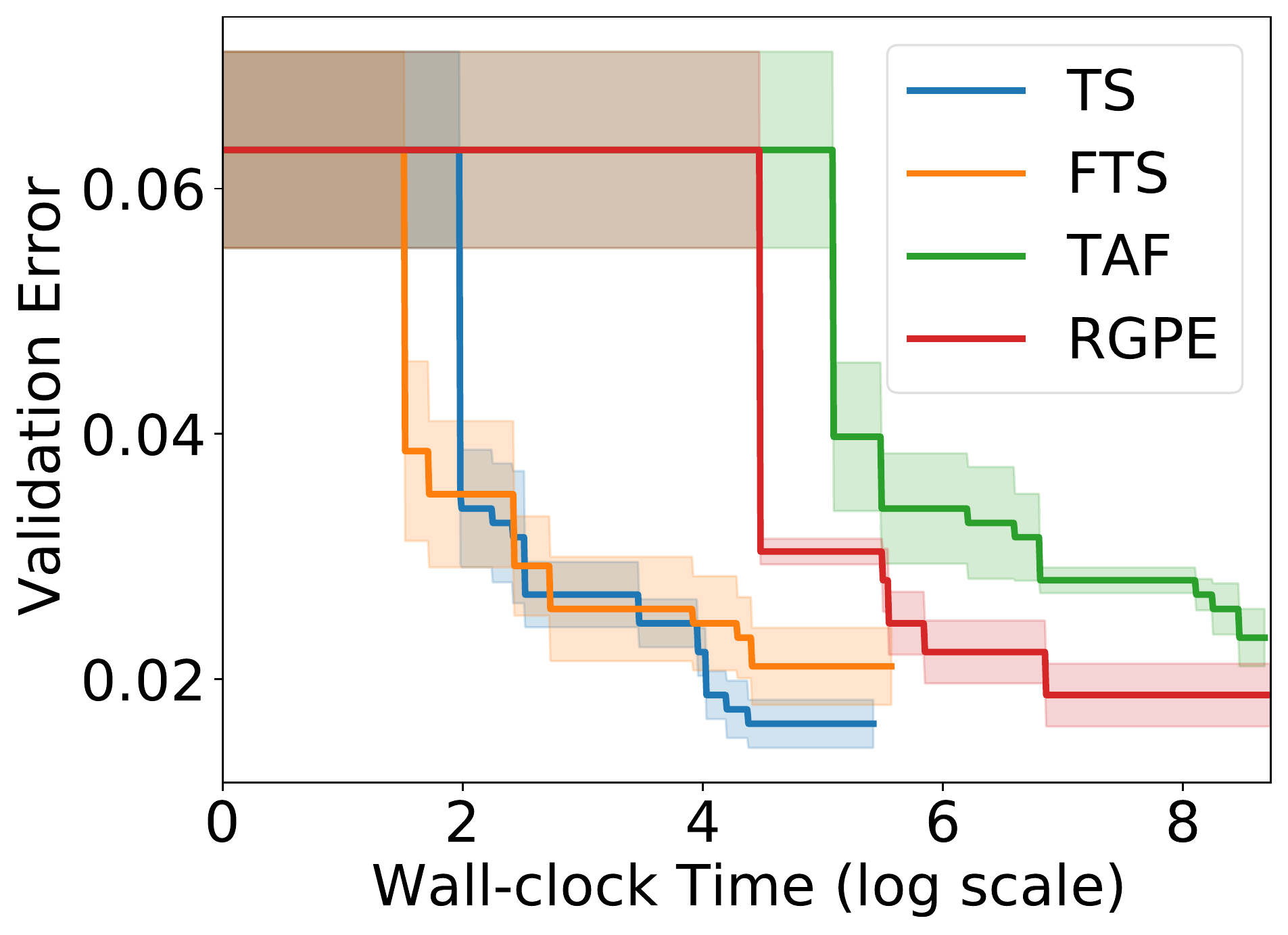}
        \caption{}
    \end{subfigure}
    \begin{subfigure}[t]{0.32\linewidth}
        \includegraphics[width=\linewidth]{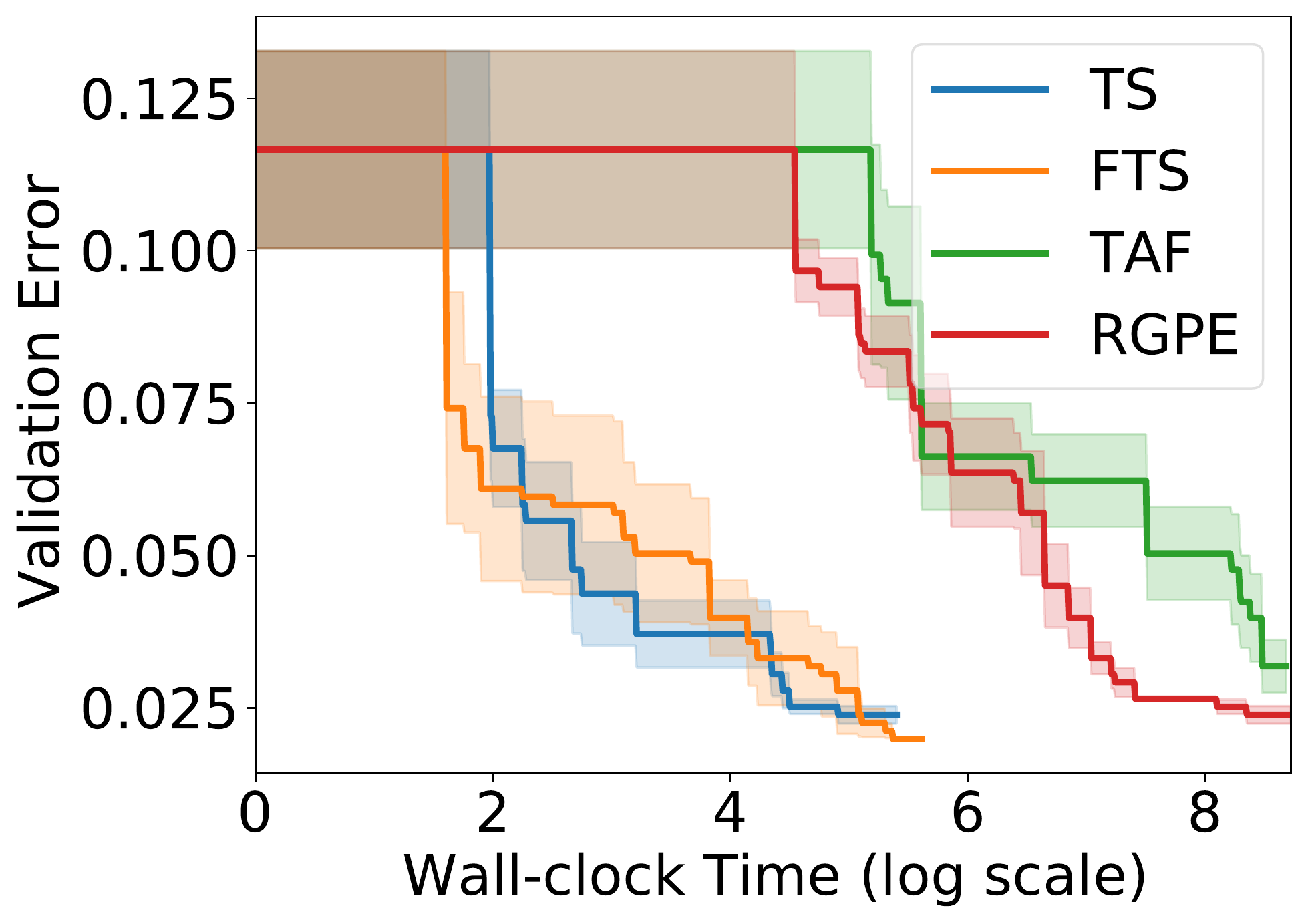}
        \caption{}
    \end{subfigure}
    \begin{subfigure}[t]{0.32\linewidth}
        \includegraphics[width=\linewidth]{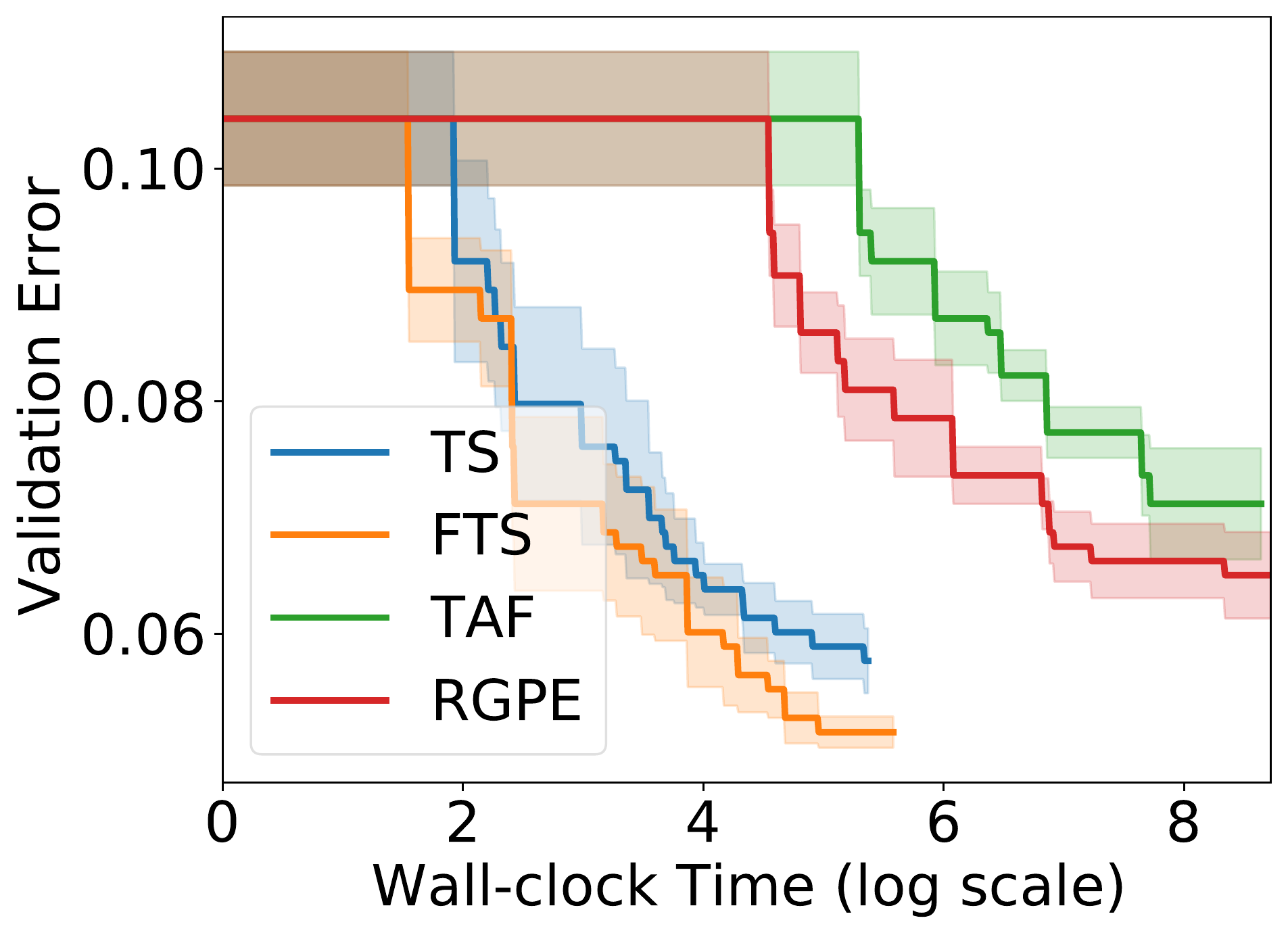}
        \caption{}
    \end{subfigure}
    \begin{subfigure}[t]{0.32\linewidth}
        \includegraphics[width=\linewidth]{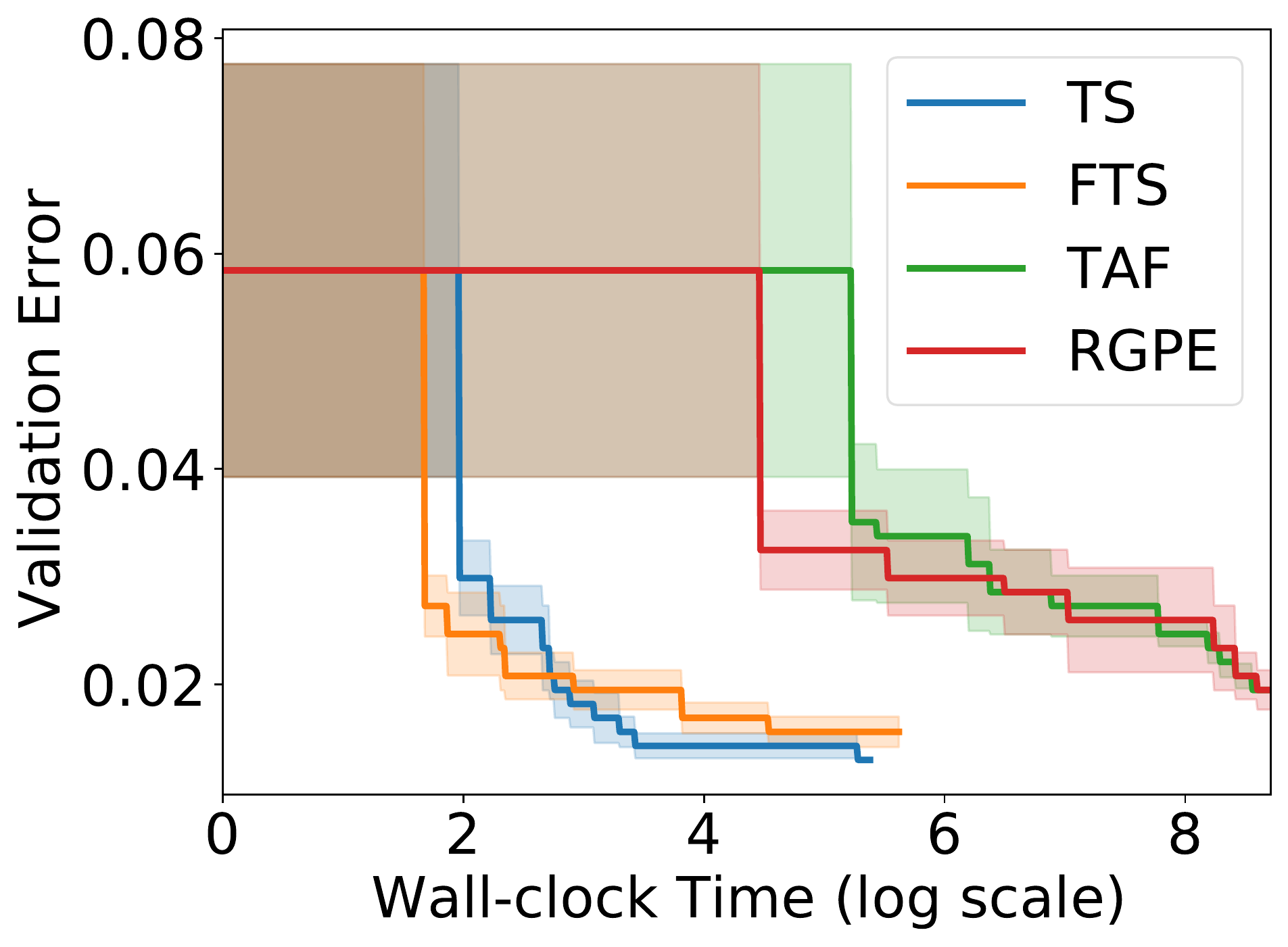}
        \caption{}
    \end{subfigure}
    \begin{subfigure}[t]{0.32\linewidth}
        \includegraphics[width=\linewidth]{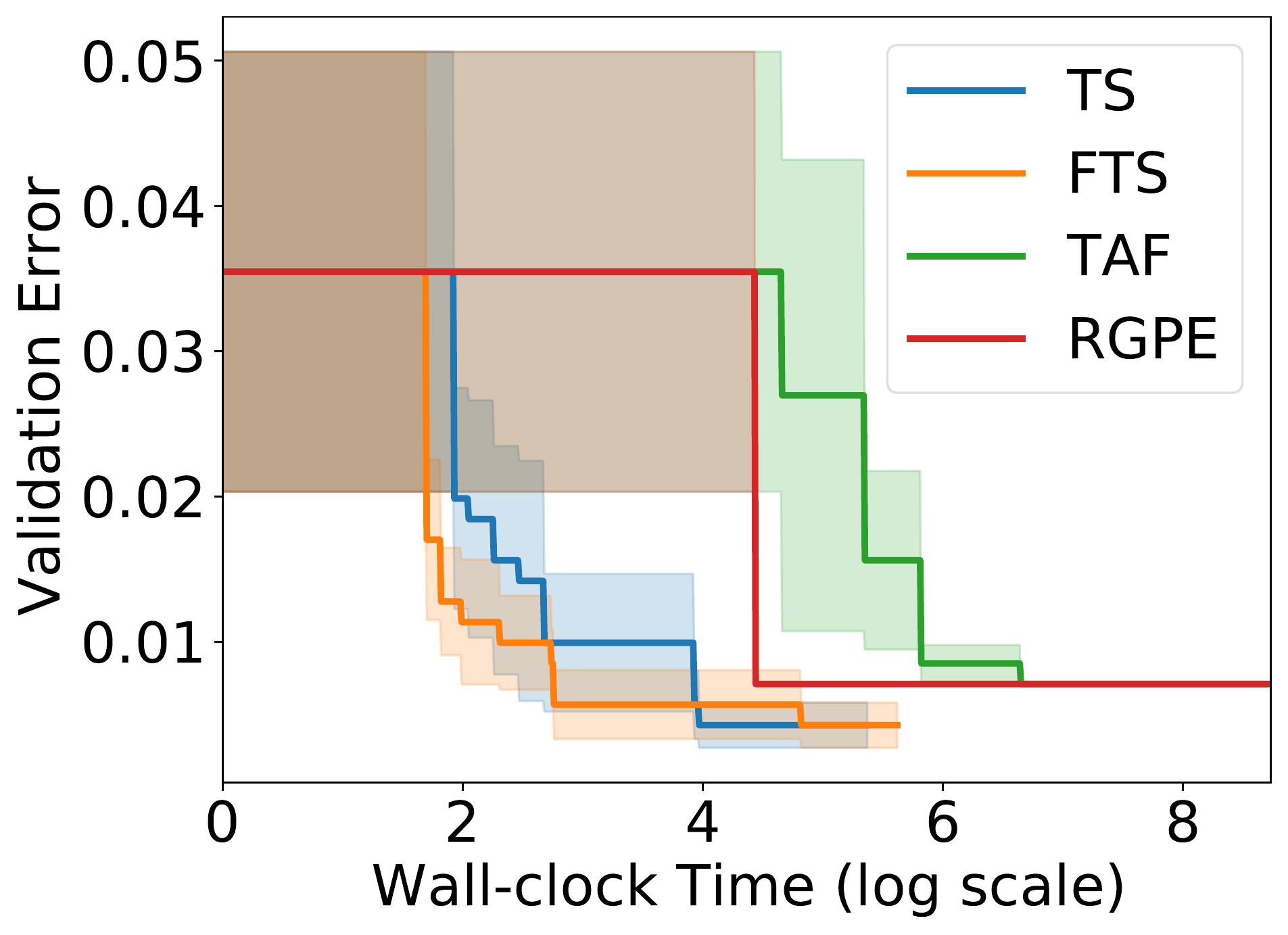}
        \caption{}
    \end{subfigure}
    \caption{Additional results for the other $5$ target agents for the activity recognition experiment using mobile phone sensors ($M=100$).}
    \label{fig:add_results_app_har}
\end{figure}

\end{document}